\documentclass{article}

\PassOptionsToPackage{numbers, compress}{natbib}

\usepackage[final]{neurips_2025}

\usepackage[utf8]{inputenc} 
\usepackage[T1]{fontenc}    
\usepackage{hyperref}       
\usepackage{url}            
\usepackage{booktabs}       
\usepackage{amsfonts}       
\usepackage{nicefrac}       
\usepackage{microtype}      
\usepackage{xcolor}         
\usepackage{graphicx}
\usepackage{subfigure}
\usepackage{amsmath}
\usepackage{amssymb}
\usepackage{mathtools}
\usepackage{amsthm}
\usepackage[capitalize,noabbrev]{cleveref}
\usepackage{titletoc}
\usepackage{multirow}
\usepackage[font=small,labelfont=bf]{caption}

\theoremstyle{plain}
\newtheorem{theorem}{Theorem}[section]
\newtheorem{proposition}[theorem]{Proposition}
\newtheorem{lemma}[theorem]{Lemma}

\theoremstyle{definition}
\newtheorem{definition}[theorem]{Definition}
\newtheorem{assumption}[theorem]{Assumption}
\theoremstyle{remark}

\renewcommand{\paragraph}[1]{\textbf{#1}~~}
\newcommand{\eg}{\textit{e.g.}}
\newcommand{\ie}{\textit{i.e.}}
\newcommand{\wrt}{\textit{w.r.t.\ }}
\newcommand{\npcname}{\textrm{NPC}}
\newcommand{\rnpcname}{\textrm{RNPC}}

\title{Understanding and Improving Adversarial Robustness of Neural Probabilistic Circuits}

\author{%
  Weixin Chen \\
  University of Illinois Urbana-Champaign \\
  \texttt{weixinc2@illinois.edu} \\
  \And
  Han Zhao \\
  University of Illinois Urbana-Champaign \\
  \texttt{hanzhao@illinois.edu} \\
}

\begin{document}

\maketitle

\begin{abstract}
Neural Probabilistic Circuits (\npcname s), a new class of concept bottleneck models, comprise an attribute recognition model and a probabilistic circuit for reasoning.
By integrating the outputs from these two modules, \npcname s produce compositional and interpretable predictions.
While offering enhanced interpretability and high performance on downstream tasks, the neural-network-based attribute recognition model remains a black box. 
This vulnerability allows adversarial attacks to manipulate attribute predictions by introducing carefully crafted, subtle perturbations to input images, potentially compromising the final predictions.
In this paper, we theoretically analyze the adversarial robustness of \npcname~and demonstrate that it only depends on the robustness of the attribute recognition model and is independent of the robustness of the probabilistic circuit.
Moreover, we propose \rnpcname, the first robust neural probabilistic circuit against adversarial attacks on the recognition module. \rnpcname~introduces a novel class-wise integration for inference, ensuring a robust combination of outputs from the two modules.
Our theoretical analysis demonstrates that \rnpcname~exhibits provably improved adversarial robustness compared to \npcname. Empirical results on image classification tasks show that \rnpcname~achieves superior adversarial robustness compared to existing concept bottleneck models while maintaining high accuracy on benign inputs.
The code is available at \href{https://github.com/uiuctml/RNPC}{https://github.com/uiuctml/RNPC}.
\end{abstract}

\section{Introduction} \label{sec:intro}
Deep Neural Networks (DNNs) exhibit superior performance across a range of downstream tasks.
However, DNNs are often criticized for their lack of interpretability, making it hard to understand the decision-making process, especially when they are deployed in high-stakes domains, such as legal justice and healthcare~\citep{stop}. 
Concept Bottleneck Models (CBMs)~\citep{cbm, dcr, probcbm, posthoc_cbm, labelfree_cbm} are a class of models that attempt to improve model interpretability by incorporating human-understandable binary concepts (\eg, white color) as an intermediate layer, followed by simple predictors such as linear models. 
This bottleneck enables model predictions to be interpreted using the predicted concepts due to the simplicity of the linear predictors on top of the concepts.
While demonstrating improved~interpretability, CBMs usually suffer from a performance drop compared to DNNs. 
Recently, a new class of concept bottleneck models, Neural Probabilistic Circuits (\npcname s)~\citep{npc}, has been introduced, which offers a promising balance between model interpretability and task performance.
\npcname~consists of an attribute recognition model and a probabilistic circuit~\citep{probabilistic_circuit}.
The attribute recognition model predicts various interpretable categorical attributes (\eg, color) from an input image. The probabilistic circuit supports tractable joint, marginal, and conditional inference over these attributes and the class variable.
By integrating the probability of each instantiation of attributes and the conditional probability of a specific class given that instantiation, \npcname~generates the prediction score for the class.

Despite enhanced transparency in the model architecture, the attribute recognition model within \npcname, implemented using a neural network, remains a black box.
This raises the threat of malicious attacks targeting the attribute recognition model.
Adversarial attacks~\citep{fgsd_attack, pgd_attack, cw_attack}, a typical type of attack against neural networks, attempt to manipulate model predictions by applying carefully crafted, imperceptible perturbations to the input images. 
Such attacks against the attribute recognition model can mislead attribute predictions, potentially compromising \npcname's performance on downstream tasks.

In this paper, we theoretically analyze \npcname's robustness against these attacks, understanding how the robustness of individual modules affects that of the overall model.
Surprisingly, we show that the robustness of the overall model only depends on the robustness of the attribute recognition model, and including a probabilistic circuit does not impact the robustness of the overall model. This is in sharp contrast to the compositional nature of \npcname's estimation error, as demonstrated in~\citet[Theorem 2]{npc}. 
This means that adversarial robustness can be achieved for free by using probabilistic circuits on top of intermediate concepts, rather than the linear predictors used in conventional CBMs.

To further improve the adversarial robustness of \npcname, we propose the Robust Neural Probabilistic Circuit (\rnpcname), which adopts the same model architecture as \npcname~while introducing a novel class-wise integration approach for inference.
Specifically, we first partition the attribute space by class, where each class corresponds to a set of high-probability attribute instantiations, and then define the neighborhood for each class to allow perturbations. Rather than focusing on individual attribute instantiations, \rnpcname~integrates the probability over the neighborhood of each class and the conditional probability of a target (class) given the high probability region of that class. 

Theoretically, we show that such class-wise integration enables \rnpcname~to achieve improved adversarial robustness compared to \npcname. We also perform an analysis of \rnpcname's performance on benign inputs. Similar to \npcname, the estimation error of \rnpcname~is compositional and bounded by a linear combination of errors from its individual modules. Moreover, we provide an explicit characterization to quantify the trade-off between \rnpcname's adversarial robustness and benign performance.

Empirical results on diverse image classification datasets demonstrate that \rnpcname~outperforms existing concept bottleneck models in robustness against three types of adversarial attacks across various attack budgets while maintaining high accuracy on benign inputs.
Additionally, we conduct extensive ablation studies, including analyzing the impact of the number of attacked attributes and examining the effect of spurious correlations among various attributes.

Our main contributions are threefold. 
\textbf{1)} We propose the first robust neural probabilistic circuit, named \rnpcname, against adversarial attacks on the attribute recognition model. In particular, \rnpcname~introduces a novel class-wise integration approach for inference, ensuring a robust combination of outputs from different modules.
\textbf{2)} Theoretically, we demonstrate that:
\underline{a)} The robustness of \npcname~and \rnpcname~depends only on the robustness of the attribute recognition model, and introducing a probabilistic circuit on top of the attribute recognition model is free for robustness.
\underline{b)} \rnpcname~is guaranteed to achieve higher robustness than \npcname~under certain conditions.
\underline{c)} \rnpcname~maintains a compositional estimation error on benign inputs.
\underline{d)} There exists a trade-off between \rnpcname's adversarial robustness and benign performance.
\textbf{3)} Empirically, we show that \rnpcname~achieves superior robustness against diverse adversarial attacks compared to various concept bottleneck models, while maintaining high accuracy on benign inputs.

\section{Preliminaries} \label{sec:prelim}
\subsection{Neural probabilistic circuits}
A Neural Probabilistic Circuit (\npcname)~\citep{npc} consists of an attribute recognition model and a probabilistic circuit.
Let $X\in\mathcal{X}$, $Y\in\mathcal{Y}$, $A_k\in\mathcal{A}_k$ denote the input variable, the class variable, and the $k$-th attribute variable, respectively, with their lowercase letters representing the corresponding instantiations.
Consider $K$ attributes, $A_1, \ldots, A_K$, or $A_{1:K}$ in short.
The neural-network-based attribute recognition model takes an image $x$ as input and outputs probability vectors for various attributes. The $k$-th probability vector is denoted as $\left( \mathbb{P}_{\theta_k}(A_k=a_k \mid X=x) \right)_{a_k\in\mathcal{A}_k}$, where $\theta_k$ represents the model parameters related to the $k$-th attribute.
The probabilistic circuit~\citep{probabilistic_circuit} learns the joint distribution of $Y$ and $A_{1:K}$, while also supporting tractable conditional inference such as $\mathbb{P}_w(Y \mid A_{1:K})$, where $w$ represents the circuit's parameters.
Combining the outputs from both the attribute recognition model and the probabilistic circuit, \npcname's prediction score for class $y$ is interpretable, which is the sum of the probability of each instantiation of attributes, weighted by the conditional probability of $y$ given this instantiation, \ie, 
{
\small
\begin{align}
    &\mathbb{P}_{\theta, w}\left(Y=y \mid X=x\right) = \sum_{a_{1:K}} \prod_{k=1}^K \mathbb{P}_{\theta_k}\left(A_k=a_k \mid X=x\right)
    \cdot \mathbb{P}_{w}\left(Y=y \mid A_{1:K} = a_{1:K}\right), \label{eq:npc}
\end{align}
}
where $\theta$ denotes all parameters of the attribute recognition model.
An illustration of \npcname~is in Fig \ref{fig:framework}.

\begin{figure*}[tb]
    \centering
    \includegraphics[width=0.90\linewidth]{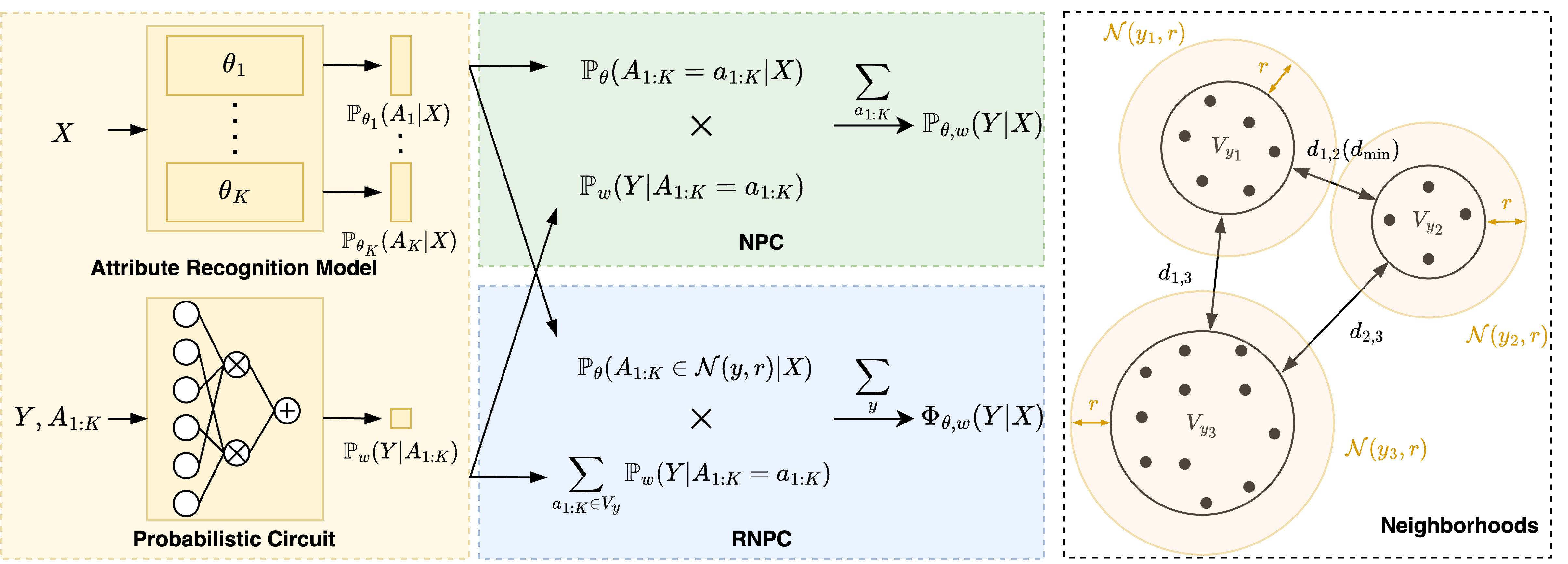}
    \caption{
    \small
    \textbf{Left:} Model architectures and inference procedures of \npcname~and \rnpcname.
    \npcname~and \rnpcname~share an attribute recognition model and a probabilistic circuit.
    During inference, \npcname~employs a node-wise integration to integrate the outputs of the two modules, producing predictions for downstream tasks. In contrast, \rnpcname~adopts a class-wise integration, which leads to more robust task predictions.
    \textbf{Right:} Illustration of three classes in the attribute space. $V_y$ represents the set of attribute nodes with high probabilities. $d_{\min}$ is the minimum inter-class distance and the radius $r$ is defined as $\lfloor\frac{d_{\min}-1}{2}\rfloor$.
    $\mathcal{N}(y, r)$ denotes the neighborhood of class $y$ with radius $r$.
    }
    \label{fig:framework}
\end{figure*}

\subsection{Threat model}
%
Consider a white-box, norm-bounded, untargeted adversarial attack against the attribute recognition model. Given an input $(x, a_{1:K})$, the attacker seeks to find a perturbed input $\tilde{x}$ within an $\ell_p$-norm ball of radius $\ell$ centered at $x$, such that the \textit{attribute recognition model}'s predictions for \textit{one or more} attributes become incorrect. Assume $m$ attacked attributes $A_{i_1:i_m}$, the attack objective is $\max_{\tilde{x} \in \mathbb{B}_p(x, \ell)} \frac{1}{m} \sum_{k=1}^m \mathcal{L}\left( \mathbb{P}_{\theta_{i_k}}(A_{i_k} \mid X=\tilde{x}), a_{i_k} \right)$ where $\mathbb{B}_p(x, \ell):=\{\tilde{x} \in \mathcal{X}: \Vert \tilde{x}-x \Vert_p \leqslant \ell\}$ and $\mathcal{L}$ denotes a per-attribute loss function, which may vary depending on the chosen attack method.
\section{Understanding the adversarial robustness of neural probabilistic circuits}
\citet{npc} have shown that under Assumption \ref{assump:sufficient} and \ref{assump:complete}, the \textit{estimation error} of \npcname~is compositional, \ie, it can be upper bounded by a linear combination of the error of the attribute recognition model and the error of the probabilistic circuit.
In this section, we delve into the \textit{adversarial robustness} of \npcname, exploring how the adversarial robustness of the attribute recognition model as well as the probabilistic circuit affects that of \npcname.

\begin{assumption}[Sufficient attributes~\citep{npc}]
    The class label $Y$ and the input $X$ are conditionally independent given the attributes $A_{1:K}$, \ie, $Y \perp X \mid A_{1:K}$. 
    \label{assump:sufficient}
\end{assumption}

\begin{assumption}[Complete information~\citep{npc}]
    Given any input, all attributes are conditionally mutually independent, \ie, $A_1\perp A_2\perp\cdots\perp A_K \mid X$.
    \label{assump:complete}
\end{assumption}

\begin{definition}
    The \textit{prediction perturbation} of \npcname~against an adversarial attack on the attribute recognition model is defined as the worst-case total variance (TV) distance between the class distributions conditioned on the vanilla and perturbed inputs, \ie,
    {
    \small
    \begin{align*}
        &\Delta_{\theta, w}^{\npcname}\footnotemark := \mathbb{E}_X\left[\max_{\tilde{X}\in \mathbb{B}_p(X, \ell)}~d_{\mathrm{TV}}\left(\mathbb{P}_{\theta, w}(Y \mid X), \mathbb{P}_{\theta, w}(Y \mid \tilde{X})\right)\right].
    \end{align*}
    }
    \label{def:npc_robustness}
\end{definition}
\footnotetext{For uncluttered notation, we omit the dependency on \(\ell\) if it is clear from the context.}

The metric quantifies the adversarial robustness of \npcname, with a lower value signifying stronger robustness.
Based on this definition, the following theorem decomposes the adversarial robustness of \npcname~under the assumption of complete information.

\begin{theorem}[Adversarial robustness of \npcname s]
Under Assumption \ref{assump:complete}, the prediction perturbation of \npcname~is bounded by the worst-case TV distance between the overall attribute distributions conditioned on the vanilla and perturbed inputs, which is further bounded by the sum of the worst-case TV distances for each attribute, \ie,

\resizebox{\textwidth}{!}{%
    \begin{minipage}{\textwidth}
    \begin{align*}
        \Delta_{\theta, w}^{\npcname}
        &\leqslant
        \mathbb{E}_X\left[\max_{\tilde{X}\in \mathbb{B}_p(X, \ell)}~d_{\mathrm{TV}}\left(\mathbb{P}_{\theta}\left(A_{1:K} \mid X\right), \mathbb{P}_{\theta}\left(A_{1:K} \mid \tilde{X}\right)\right)\right] 
        \leqslant \sum_{k=1}^{K} \mathbb{E}_X\left[\max_{\tilde{X}\in \mathbb{B}_p(X, \ell)}~d_{\mathrm{TV}}\left(\mathbb{P}_{\theta_k}\left(A_k \mid X\right), \mathbb{P}_{\theta_k}\left(A_k \mid \tilde{X}\right)\right)\right].
    \end{align*}
    \end{minipage}
}
\label{thm:npc_perturbation}
\end{theorem}

We denote the first bound as $\Lambda_{\npcname}$.
Theorem \ref{thm:npc_perturbation} demonstrates that the prediction perturbation of \npcname~is upper bounded by that of the attribute recognition model. Different from typical DNNs whose robustness is influenced by the robustness of each layer~\citep{xiaoyi2020layer, rahnama2020robust}, the robustness of \npcname~depends solely on that of the attribute recognition model. 
Adding a probabilistic circuit on top does not affect the robustness of \npcname. 
Note that this is in sharp contrast to typical CBMs, where the linear-layer-based predictors adversely affect the robustness of the overall model~\citep{robustness_cbm2}.

\section{Improving the adversarial robustness of neural probabilistic circuits} \label{sec:improve}
In this section, we propose Robust Neural Probabilistic Circuits (\rnpcname s). \rnpcname s introduce a novel approach for integrating the outputs from the attribute recognition model and the probabilistic circuit, leading to class predictions that are provably more robust than those of \npcname s.

\subsection{Notation and definitions} \label{sec:notation}
Let $D:=\left\{\left(x, a_{1: K}, y\right)\right\}$ denote a dataset, and let $V:= \left\{ a_{1:K}: \mathbb{P}_D(A_{1:K}=a_{1:K}) \geqslant \gamma \right\}$ denote the corresponding set in the attribute space that has a high probability mass, specifically larger than a constant $\gamma \geq 0$.
We partition $V$ according to the most probable class $a_{1:K}$ is in, \ie, if $y^*=\arg\max_{y\in\mathcal{Y}} \mathbb{P}_D(Y=y\mid A_{1:K}=a_{1:K})$, then $a_{1:K}\in V_{y^*}$.
Overall, $V = \bigcup_{y\in\mathcal{Y}}V_y$ and $V_i \cap V_j = \emptyset,~\forall i \neq j$.
Let $\Omega$ denote the whole attribute space, and let $V^c:=\Omega \backslash V$ denote the complement of $V$, which is the set of attribute instantiations with a probability mass at most $\gamma$. In the following, we mainly focus on the attribute set $V$ and name each $a_{1:K}$ as an (attribute) node.
The Hamming distance between two nodes, say $a_{1:K}$ and $a_{1:K}^{\prime}$, is defined as the number of attributes in which the two nodes differ, \ie, $\operatorname{Ham}(a_{1:K}, a'_{1:K}):= \sum_{k=1}^K \mathbb{I}(a_k \neq a_k^{\prime})$.

\begin{definition}
The \textit{inter-class distance} between class $i$ and class $j$ is defined as the minimum Hamming distance between nodes of $V_i$ and nodes of $V_j$, \ie, $d_{i, j} := \min _{v_i \in V_i, v_j \in V_j}\left\{\operatorname{Ham}\left(v_i, v_j\right)\right\}$.
\end{definition}

\begin{definition}
The \textit{minimum inter-class distance} of an attribute set $V$ is $d_{\min } := \min _{i, j \in\mathcal{Y}, i \neq j}\left\{d_{i, j}\right\}$. The \textit{radius} of an attribute set $V$ is $r := \lfloor\frac{d_{\min}-1}{2}\rfloor$.
\end{definition}


\begin{definition}
The \textit{neighborhood} of class $y$ with radius $r$ is defined as the union of $V_y$ and the nodes from $V^c$ whose distance from $V_y$ is not larger than $r$, \ie, $\mathcal{N}(y,r) := V_y~\bigcup~\left\{ a_{1:K}^c\in V^c: \min_{a_{1:K}\in V_y} \operatorname{Ham}\left( a_{1:K}^c, a_{1:K} \right) \leqslant r \right\}$.
\end{definition}

Figure \ref{fig:framework} illustrates the high-probability attribute nodes and neighborhoods of three classes.

\subsection{Robust neural probabilistic circuits} \label{sec:rnpc}
\rnpcname s employ the same model architecture as \npcname s, but use a novel inference procedure that potentially leads to more robust predictions for downstream tasks.

\paragraph{Model architecture and training.}
The architecture of \rnpcname~is the same as that of \npcname, consisting of an attribute recognition model and a probabilistic circuit.
These two modules are trained independently. 
Specifically, the attribute recognition model is trained by minimizing the sum of cross-entropy losses over all attributes, \ie, \( \min_\theta -\frac{1}{|D|} \sum_{(x, a_{1:K}) \in D} \sum_{k=1}^K \log \mathbb{P}_{\theta_k}(A_k = a_k \mid X=x) \). 
Following~\citet{npc}, the structure of the probabilistic circuit is learned using the LearnSPN algorithm~\citep{learnspn}, and its parameters are optimized with the CCCP algorithm~\citep{cccp}.
Note that \npcname~and \rnpcname~share the same trained attribute recognition model and the same learned probabilistic circuit; the only difference between them lies in the inference procedure.

\paragraph{Intuition.} 
Interpreting $\prod_{k=1}^K \mathbb{P}_{\theta_k}\left(A_k=a_k \mid X\right)$ as the weight of an attribute node $a_{1:K}$ and $\mathbb{P}_{w}\left(Y \mid A_{1:K} = a_{1:K}\right)$ as the contribution of this node to $Y$, \npcname's prediction $\mathbb{P}_{\theta, w}\left(Y \mid X\right)$ becomes a weighted sum of all nodes' contributions to $Y$. While such predictions are generally accurate, they remain vulnerable to adversarial attacks. 
For example, consider an input image $x$ depicting a no-entry sign, with the attribute label $a_{1:K}^* = (\text{red}, \text{circle}, \text{slash})$. For a well-trained model, both the weight of $a_{1:K}^*$ and its contribution to $y_{\text{no-entry}}$ are typically high, leading $y_{\text{no-entry}}$ to be the most probable class. 
However, if an attacker attacks any $m$ attributes, say $A_{1:m}$, \ie, perturbing the predicted probabilities for these attributes, the weight of $a_{1:K}^*$ will decrease due to the reduction in $\mathbb{P}_{\theta_k}\left(A_k=a_k^* \mid X\right)$ for $k \in [m]$. Consequently, the weight of the set $\{a_{1:K}: a_{1:m} \neq a_{1:m}^*, a_{m+1:K} = a_{m+1:K}^*\}$ increases, which lies within $\mathcal{N}(y_{\text{no-entry}}, r)$ when $m \leqslant r$. 
However, the contributions of these nodes (\eg, $(\text{blue}, \text{circle}, \text{slash})$) to $y_{\text{no-entry}}$ could be very small, possibly causing $y_{\text{no-entry}}$ no longer the most probable class. 
Nevertheless, if those ``shifted'' weights can be aggregated and aligned with the high contribution of $a_{1:K}^*$ to $y_{\text{no-entry}}$, the adverse impact of the attacks can be alleviated.

\paragraph{Inference.}
Inspired by the example above, we propose a novel inference procedure that robustly integrates the output of the attribute recognition model and the output of the probabilistic circuit to produce the final predictions.
In particular, instead of adopting \npcname's node-wise integration, we introduce the following class-wise integration,
{
\small
\begin{align}
    &\Phi_{\theta, w}(Y\mid X) = \sum_{\tilde{y} \in \mathcal{Y}} \left( \mathbb{P}_{\theta}\left(A_{1:K} \in \mathcal{N}(\tilde{y}, r) \mid X\right) \cdot \sum_{a_{1:K} \in V_{\tilde{y}}} \mathbb{P}_{w}\left(Y \mid A_{1:K}=a_{1:K}\right) \right). \label{eq:rnpc_unml}
\end{align}
}

Equation (\ref{eq:rnpc_unml}) characterizes an inference procedure that integrates the weight of (neighborhood of) each class with the contribution of (high-probability nodes of) this class to $Y$.
Thus, \rnpcname's prediction $\Phi_{\theta, w}(Y\mid X)$ becomes a weighted sum of all classes' contributions to $Y$. In particular, $\Phi_{\theta, w}$ represents an unnormalized probability. The corresponding partition function and normalized probability are denoted as $Z_{\theta}(X) $ $ = \sum_{y\in \mathcal{Y}}\Phi_{\theta, w}(Y=y\mid X) $ $= \sum_{\tilde{y} \in \mathcal{Y}} \left( \mathbb{P}_{\theta}\left(A_{1:K} \in \mathcal{N}(\tilde{y}, r) \mid X\right) \cdot \left|V_{\tilde{y}}\right| \right)$ and $\hat{\Phi}_{\theta, w}(Y\mid X) = \Phi_{\theta, w}(Y\mid X)/Z_{\theta}(X)$, respectively. 

In an adversarial attack, input perturbations result in perturbations in predicted attribute probabilities. If such an attack shifts the probabilities originally assigned to nodes of $V_y$ to any other nodes within $\mathcal{N}(y, r)$, the class-wise integration ensures that the weight of class $y$ is barely affected.
Meanwhile, the conditional probabilities generated by the probabilistic circuit remain benign. Consequently, the predictions produced by \rnpcname~are robust.

\paragraph{Complexity.}
Let $|f_k|$ denote the size of the $k$-th neural network in the attribute recognition model, and let $|S|$ be the size of the probabilistic circuit (\ie, the number of edges).
We have the following proposition that compares the computational complexities of inference in \npcname~and \rnpcname.
\begin{proposition}
\label{prop:complexity}
    The computational complexities of inference in \npcname~and \rnpcname~are respectively $O\left(\sum_{k=1}^K |f_k| + |S|\cdot \prod_{k=1}^K |\mathcal{A}_k|\right)$ and $O\left(\sum_{k=1}^K |f_k| + |S|\cdot|V|\right)$, with $|V|\leqslant \prod_{k=1}^K |\mathcal{A}_k|$.
\end{proposition}
The detailed proof is deferred to Appendix \ref{app:complexity}. 
\Cref{prop:complexity} shows that \rnpcname~is more efficient than \npcname~in terms of the inference complexity.


\subsection{Theoretical analysis}
In this section, we provide a theoretical analysis of the adversarial robustness and benign performance of \rnpcname. Similar to Definition \ref{def:npc_robustness}, we first define a metric to quantify its adversarial robustness.
\begin{definition}
    The \textit{prediction perturbation} of \rnpcname~against an adversarial attack on the attribute recognition model is defined as the worst-case TV distance between the class distributions conditioned on the vanilla and perturbed inputs, \ie,
    {
    \small
    \begin{align*}
        \Delta_{\theta, w}^{\rnpcname} := \mathbb{E}_X\left[\max_{\tilde{X}\in \mathbb{B}_p(X, \ell)}~d_{\mathrm{TV}}\left(\hat{\Phi}_{\theta, w}(Y \mid X), \hat{\Phi}_{\theta, w}(Y \mid \tilde{X})\right)\right].
    \end{align*}
    }
\end{definition}

\subsubsection{Adversarial robustness of \rnpcname s}

\begin{lemma}[Adversarial robustness of \rnpcname s]
    The prediction perturbation of \rnpcname~is bounded by the worst-case change in probabilities within a neighborhood caused by the attack, \ie,
    {
    \small
    \begin{align*}
    \Delta_{\theta, w}^{\rnpcname}
    \leqslant 
    \mathbb{E}_X\left[\max_{\tilde{X}\in \mathbb{B}_p(X, \ell)}~ \left\{ \max_{\tilde{y}\in\mathcal{Y}} \left| 1-\frac{\mathbb{P}_{\theta}(A_{1:K}\in \mathcal{N}(\tilde{y}, r)\mid \tilde{X})}{\mathbb{P}_{\theta}(A_{1:K}\in \mathcal{N}(\tilde{y}, r)\mid X)} \right| \right\} \right] .
    \end{align*}
    }
    \label{thm:rnpc_perturbation}
\end{lemma}
We denote this bound as $\Lambda_{\rnpcname}$. Similar to \npcname, the adversarial robustness of \rnpcname~depends solely on the attribute recognition model and is not influenced by the probabilistic circuit. In particular, one can see from the above bound that, as long as the perturbation does not change the probabilities of attributes encoded in each ball $\mathcal{N}(\tilde{y}, r)$, the upper bound is 0, \ie, the prediction is robust.

\subsubsection{Comparison in adversarial robustness}
\npcname~and \rnpcname~adopt node-wise and class-wise integration approaches during inference, respectively, leading to different upper bounds on the prediction perturbation. 
Here, we investigate the relationship between these bounds and compare the adversarial robustness of \npcname~and \rnpcname.



\begin{theorem}[Comparison in adversarial robustness]
    Consider a $p$-norm-bounded adversarial attack with a budget of $\ell$.
    Assume the attribute recognition model $f_\theta$ is randomized and satisfies $\epsilon$-Differential Privacy (DP) with respect to the $p$-norm.
    Let the probability of an attribute taking a specific value correspond to the expected model output, \ie, $\mathbb{P}_{\theta_k}(A_k=a_k \mid X) = \mathbb{E}[f_{\theta_k}(X)_{a_k}]$, where the expectation is taken over the randomness within the model.
    Under Assumption \ref{assump:complete}, the following holds:
    $\Lambda_{\npcname} \leqslant \frac{|\mathcal{A}_1|\ldots|\mathcal{A}_K|}{2}\alpha_\epsilon$ and $\Lambda_{\rnpcname} \leqslant \alpha_\epsilon$, where $\alpha_\epsilon := \max\{1-e^{-K\epsilon}, e^{K\epsilon}-1\}$.
    Moreover, there exist instances where both inequalities simultaneously hold as equalities.
    \label{thm:comparison}
\end{theorem}

The above theorem establishes the relationship between $\Lambda_{\npcname}$ and $\Lambda_{\rnpcname}$ under the condition of DP.
Specifically, compared to $\Lambda_{\rnpcname}$, $\Lambda_{\npcname}$ is bounded by an exponentially larger value that scales exponentially with the number of attributes.
This larger bound potentially leads to significantly weaker adversarial robustness for \npcname, highlighting the robustness improvement achieved by \rnpcname. 

\subsubsection{Benign task performance of \rnpcname s}

In this section, we focus on \rnpcname's benign task performance and answer the following questions:
\emph{Does \rnpcname~exhibit a compositional estimation error similar to \npcname?}
Furthermore, \textit{while Theorem \ref{thm:comparison} indicates the robustness improvement of \rnpcname, is there a trade-off in its prediction accuracy?}



\begin{proposition}[Optimal \rnpcname s]
    The optimal \rnpcname~\wrt the expected TV distance between the predicted distribution $\hat{\Phi}_{\theta,w}(Y\mid X)$ and the ground-truth distribution $\mathbb{P}^*(Y\mid X)$ is $\hat{\Phi}^*(Y\mid X) := \Phi^*(Y\mid X)/Z^*(X)$, where
    {
    \small
    \begin{align*}
    \Phi^*(Y\mid X) &:= \sum_{\tilde{y}} \left( \mathbb{P}^*(A_{1:K}\in\mathcal{N}(\tilde{y},r)\mid X)\cdot \sum_{\tilde{a}_{1:K}\in V_{\tilde{y}}}\mathbb{P}^*(Y\mid A_{1:K}=\tilde{a}_{1:K}) \right),
    \end{align*}
    }
    and $Z^*(X)$ is the partition function. Here, $\mathbb{P}^*$ denotes the respective ground-truth distributions.
    \label{prop:optimal}
\end{proposition}

\begin{definition}
    The \textit{estimation error} of \rnpcname~is defined as the expected TV distance between the predicted distribution and the optimal distribution, \ie,
    {
    \small
    \begin{align*}
        \hat{\varepsilon}_{\theta, w}^{\rnpcname} := \mathbb{E}_X\left[ d_{\mathrm{TV}}\left( \hat{\Phi}_{\theta,w}(Y\mid X), \hat{\Phi}^*(Y\mid X) \right) \right].
    \end{align*}
    }
\end{definition}

\begin{theorem}[Compositional estimation error]
    The estimation error of \rnpcname~is bounded by a linear combination of errors from the attribute recognition model and the probabilistic circuit, \ie,
    {
    \small
    \begin{align*}
        \hat{\varepsilon}_{\theta, w}^{\rnpcname} \leqslant 
        &~\mathbb{E}_X\left[ \max_{\tilde{y}} \left| 1-\frac{\mathbb{P}_\theta(A_{1:K}\in \mathcal{N}(\tilde{y}, r) \mid X)}{\mathbb{P}^*(A_{1:K}\in \mathcal{N}(\tilde{y}, r) \mid X)} \right| \right] + \frac{2}{\gamma} d_{\mathrm{TV}}\left(\mathbb{P}_w(Y, A_{1:K}), \mathbb{P}^*(Y, A_{1:K})\right),
    \end{align*}
    }
    where $\mathbb{P}^*$ denotes respective ground-truth distributions.
    \label{thm:compositional_error}
\end{theorem}

Theorem \ref{thm:compositional_error} demonstrates that \rnpcname~exhibits a compositional estimation error; improving either module \wrt its own error can enhance the benign performance of \rnpcname.

On the other hand, under~\Cref{assump:sufficient} and~\ref{assump:complete}, it is easy to show that the optimal \npcname~corresponds to the ground-truth distribution, which does not hold for \rnpcname s.
Thus, we define the distance between the optimal \rnpcname~and the ground-truth distribution as \rnpcname s' trade-off in benign performance.

\begin{theorem}[Trade-off of \rnpcname s]
    The trade-off of \rnpcname s in benign performance, defined as the expected TV distance between the optimal \rnpcname~\(\hat{\Phi}^*(Y \mid X)\) and the ground-truth distribution \(\mathbb{P}^*(Y \mid X)\), is bounded as follows,    
    {
    \small
    \begin{align*}
        &\mathbb{E}_X \left[ d_{\mathrm{TV}}\left( \hat{\Phi}^*(Y\mid X), \mathbb{P}^*(Y\mid X) \right) \right] 
        \leqslant \mathbb{E}_X\left[ \max_{\tilde{y}}~d_{\mathrm{TV}}\left( \bar{\mathbb{P}}^*(Y\mid A_{1:K}\in V_{\tilde{y}}), \mathbb{P}^*(Y\mid X) \right) \right],
    \end{align*}
    }
    where $\bar{\mathbb{P}}^*(Y\mid A_{1:K}\in V_{\tilde{y}}) := \frac{1}{|V_{\tilde{y}}|}\sum_{a_{1:K}\in V_{\tilde{y}}}\mathbb{P}^*(Y\mid A_{1:K}=a_{1:K})$ represents the average ground-truth conditional distribution of $Y$ given $A_{1:K}\in V_{\tilde{y}}$.
    \label{thm:tradeoff}
\end{theorem}

Theorem \ref{thm:tradeoff} characterizes an upper bound on the trade-off, which is determined by
the underlying data distributions and the partitioning of the attribute space, specifically $\{V_y\}$.
Changes in these factors affect the optimal \rnpcname's distance from the ground-truth distribution, which can also be interpreted as the price of robustness paid by \rnpcname s.

\section{Experiments} \label{sec:exp}

\subsection{Experimental settings} \label{sec:exp_setting}
\paragraph{Datasets.}
We create four image classification datasets.
\textbf{1) MNIST-Add3:} 
This dataset is constructed from the MNIST dataset~\citep{mnist}, following the standard processing procedures outlined in~\citet{deepproblog, rsbench}.
Each image consists of a concatenation of three digit images, with each digit serving as one attribute. The task on this dataset is to predict the sum of these three digits.
\textbf{2) MNIST-Add5:} 
This dataset is constructed similarly to MNIST-Add3, except that each image concatenates five digit images. 
\textbf{3) CelebA-Syn:} This dataset is synthesized based on the CelebA dataset~\citep{celeba} using StarGAN~\citep{stargan}. Each synthesized image demonstrates eight facial attributes, such as hair color. 
Each unique combination of attribute values is assigned a group number. The task on this dataset is to identify each image's group number.
\textbf{4) GTSRB-Sub:} This dataset is a subset of the GTSRB dataset~\citep{gtsrb}, where each image represents a traffic sign and is annotated with four attributes like color and shape. The task is to classify images into their corresponding sign types.
To ensure certain minimum inter-class distances over the above datasets, we constrain the possible instantiations of attributes, \ie, only images aligned with these instantiations are generated or sampled.
For instance, in a three-dimensional attribute space, if the instantiations are limited to $\{(0, 0, 0), (2,2,2)\}$, the resulting minimum inter-class distance is 3.
Using this approach, the minimum inter-class distances for MNIST-Add3, MNIST-Add5, CelebA-Syn, and GTSRB-Sub are set to 3, 5, 4, and 3.

\paragraph{Baseline models and architectures.}
We select three representative concept bottleneck models as baselines, including \npcname~\citep{npc}, vanilla CBM~\citep{cbm}, and DCR~\citep{dcr}.
We adopt independent two-layer multilayer perceptrons (MLPs) to learn different attributes in both \npcname~and \rnpcname.
To ensure a fair comparison, CBM\footnote{Unless otherwise specified, CBM refers to vanilla CBM.} employs a two-layer MLP as its recognition module. DCR uses a similar architecture with the second layer replaced with its embedding layer.

\paragraph{Attack configurations.} To validate the robustness of \rnpcname, we adopt attacks that can significantly compromise the model’s predictions on the attacked attribute(s). Specifically, we use the $\infty$-norm-bounded PGD attack~\citep{pgd_attack}, the 2-norm-bounded PGD attack~\citep{pgd_attack}, and the 2-norm CW attack~\citep{cw_attack}, all configured in the untargeted setting. Empirical results demonstrate that under our threat model, these attacks are sufficient to substantially reduce attribute prediction accuracy—often down to 0\% under large norm bounds.
\textbf{Evaluation metrics.}~~
We adopt classification accuracy as the evaluation metric.
Specifically, when testing on benign (adversarial) inputs, the accuracy is referred to as \textit{benign (adversarial) accuracy}. 
Further details on the experimental settings can be found in Appendix~\ref{app:exp_setting}.

\subsection{Main results} \label{sec:exp_main}
\paragraph{Benign accuracy.}
Table \ref{tab:benign_acc} shows that both \rnpcname~and the baseline models perform exceptionally well across the four datasets, exhibiting benign accuracy approximating 100\%.
Notably, on these selected datasets, \rnpcname~is comparable with \npcname~and even attains slightly higher benign accuracy on the MNIST-Add3 and MNIST-Add5 datasets.
These empirical results indicate that \rnpcname's trade-off in benign accuracy could be negligible on these datasets.

\begin{table}[hb]
\centering
\caption{
\small
Benign accuracy (\%) of CBM, DCR, \npcname, and \rnpcname~on four image classification datasets. 
}
\label{tab:benign_acc}
\scalebox{0.85}{
\begin{tabular}{lcccc}
\toprule
\textbf{Dataset} & \textbf{CBM} & \textbf{DCR} & \textbf{NPC} & \textbf{RNPC} \\
\midrule
MNIST-Add3  & 99.02 & 98.54 & 99.32 & \textbf{99.37} \\
MNIST-Add5  & 99.37 & 99.21 & 99.40 & \textbf{99.51} \\
CelebA-Syn  & 99.83 & 99.45 & \textbf{99.95} & \textbf{99.95} \\
GTSRB-Sub   & 99.42 & 99.42 & \textbf{99.57} & 99.49 \\
\bottomrule
\end{tabular}
}
\end{table}

\paragraph{Adversarial accuracy.}
Figure \ref{fig:pgd_attack} illustrates the adversarial accuracy of \rnpcname~and the baseline models under the $\infty$-norm-bounded PGD attack with varying norm bounds.
In this setup, the attacker attacks a single attribute at a time, generating an adversarial perturbation to distort the prediction result for that attribute. The solid lines and the surrounding shaded regions represent the mean adversarial accuracy and the standard deviation, respectively, computed across all attacked attributes.
Additionally, Figure \ref{fig:attr_pred} in Appendix \ref{app:attr_pred} presents the accuracy of the attribute recognition model in predicting the attacked attribute, which drops to nearly 0\% under large norm bounds.
In Figure \ref{fig:pgd_attack}, we observe that across all datasets, the adversarial accuracy of all models decreases as the norm bound increases, demonstrating that stronger attacks cause greater harm to the models.

\begin{figure*}[tb]
    \centering
    \begin{minipage}{0.30\textwidth}
        \centering
        \includegraphics[width=\textwidth]{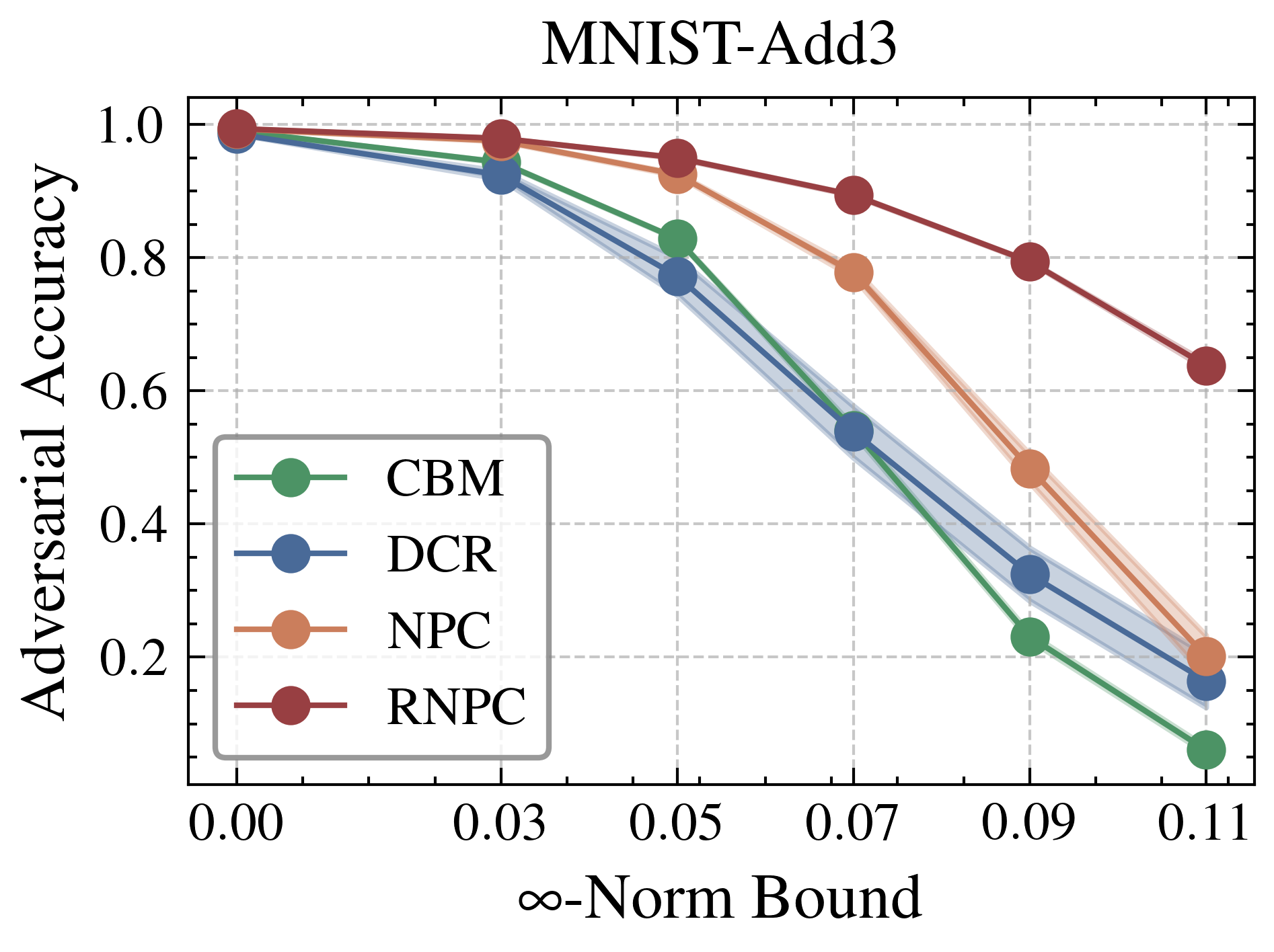}
    \end{minipage}
    \begin{minipage}{0.30\textwidth}
        \centering
        \includegraphics[width=\textwidth]{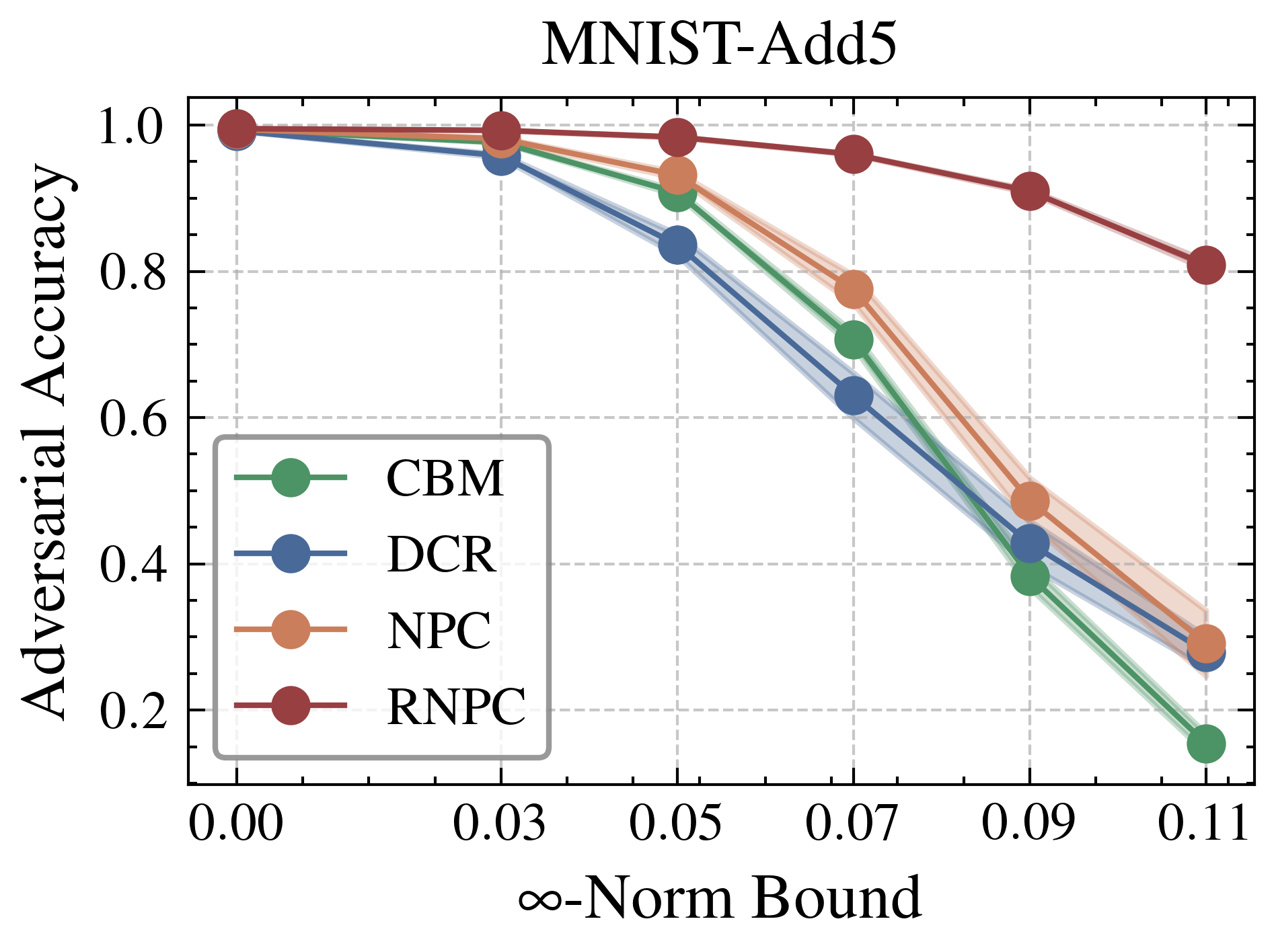}
    \end{minipage}
    \begin{minipage}{0.30\textwidth}
        \centering
        \includegraphics[width=\textwidth]{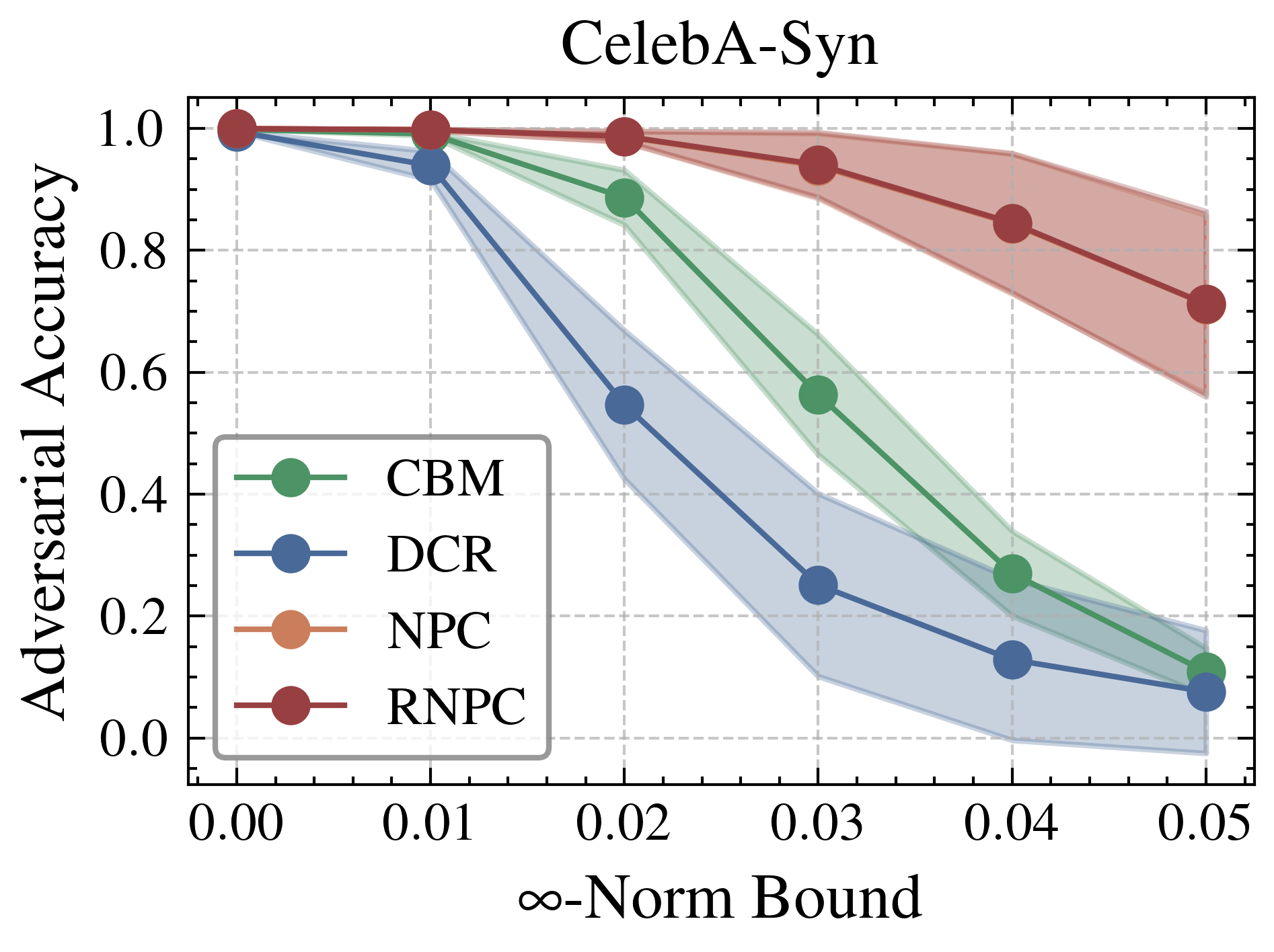}
    \end{minipage}
    \caption{
    \small Adversarial accuracy of CBM, DCR, \npcname, and \rnpcname~under the $\infty$-norm-bounded PGD attack with varying norm bounds on the MNIST-Add3, MNIST-Add5, and CelebA-Syn datasets. 
    The attacker attacks a single attribute at a time.
    The solid lines and the surrounding shaded regions represent the mean adversarial accuracy and the standard deviation, respectively, computed across all attacked attributes.}
    \label{fig:pgd_attack}
\end{figure*}

Importantly, \npcname~and \rnpcname~consistently exhibit higher robustness compared to CBM and DCR, as their adversarial accuracy remains higher under attacks with any $\infty$-norm bound.
This finding indicates that incorporating the probabilistic circuit into a model's architecture can strengthen its robustness, while the task predictors used in CBM and DCR might adversely impact model robustness.

Moreover, on the MNIST-Add3 and MNIST-Add5 datasets, \rnpcname~significantly outperforms \npcname, especially under attacks with larger norm bounds. For instance, on MNIST-Add5, when the $\infty$-norm bound reaches 0.11, \npcname's adversarial accuracy drops below 40\% whereas \rnpcname~maintains an adversarial accuracy above 80\%.
These results demonstrate that \rnpcname~provides superior robustness compared to \npcname~on these datasets, highlighting the effectiveness of the proposed class-wise integration approach.
On the CelebA-Syn dataset, \rnpcname~performs almost the same as \npcname, with both showing high robustness even under attacks with large norm bounds.
The performance against the 2-norm-bounded PGD and CW attacks is deferred to Appendix \ref{app:more_attack}.

\subsection{Ablation studies} \label{sec:exp_ablation}

\paragraph{Impact of the number of attacked attributes.}
In Section \ref{sec:exp_main}, we analyze the models' performance under attacks targeting a single attribute. 
Here, we investigate the impact of the number of attacked attributes.
To this end, we vary the number of attacked attributes for the $\infty$-norm-bounded PGD attack with a norm bound of 0.11.
Specifically, on the MNIST-Add3 dataset, we attack the attributes ``D1'' (``D'' stands for ``Digit''), ``D1, D2'', and ``D1, D2, D3'', respectively.
On MNIST-Add5, we additionally attack the attributes ``D1, D2, D3, D4'' and ``D1, D2, D3, D4, D5''.
The adversarial accuracy under attacks with varying numbers of attacked attributes is shown in Figure \ref{fig:combined} (a-b).


We observe that the adversarial accuracy of all models exhibits a downward trend as the number of attacked attributes increases, despite the perturbations remaining within the same norm bound.
These results demonstrate that, compared to heavily perturbing the predicted probabilities of a single attribute, perturbing the predicted probabilities of multiple attributes—even if not as heavily—can have a more significant negative impact.
Additionally, we find that \rnpcname~consistently outperforms the baseline models by a margin under attacks across varying numbers of attacked attributes. 
These results underscore the high robustness of \rnpcname, even with a large number of attacked attributes.

Furthermore, we discover that,
when the number of attacked attributes does not exceed the radius of a dataset, \rnpcname~has a more distinct advantage over the baseline models.
Specifically, \rnpcname~achieves up to 45\% higher adversarial accuracy than the best baseline model on MNIST-Add3 ($r=1$) when \textit{one} attribute is attacked, and 37\% higher on MNIST-Add5 ($r=2$) when \textit{two} attributes are attacked.
In contrast, when this number exceeds the radius, the advantage of \rnpcname~tends to decrease.


\begin{figure}[t]
    \centering
    \begin{minipage}{0.30\textwidth}
        \centering
        \includegraphics[width=\textwidth]{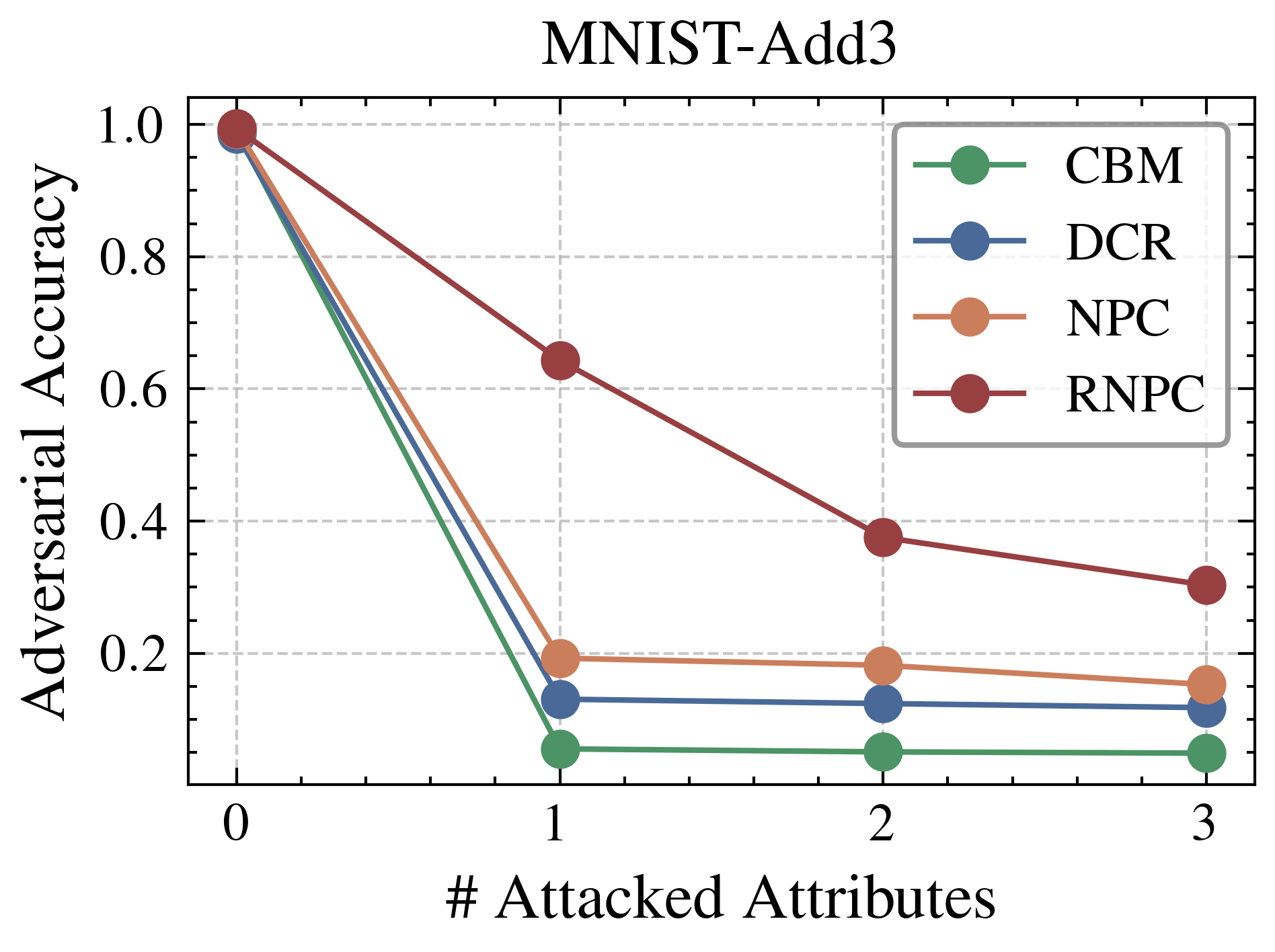}
    \end{minipage}
    \begin{minipage}{0.30\textwidth}
        \centering
        \includegraphics[width=\textwidth]{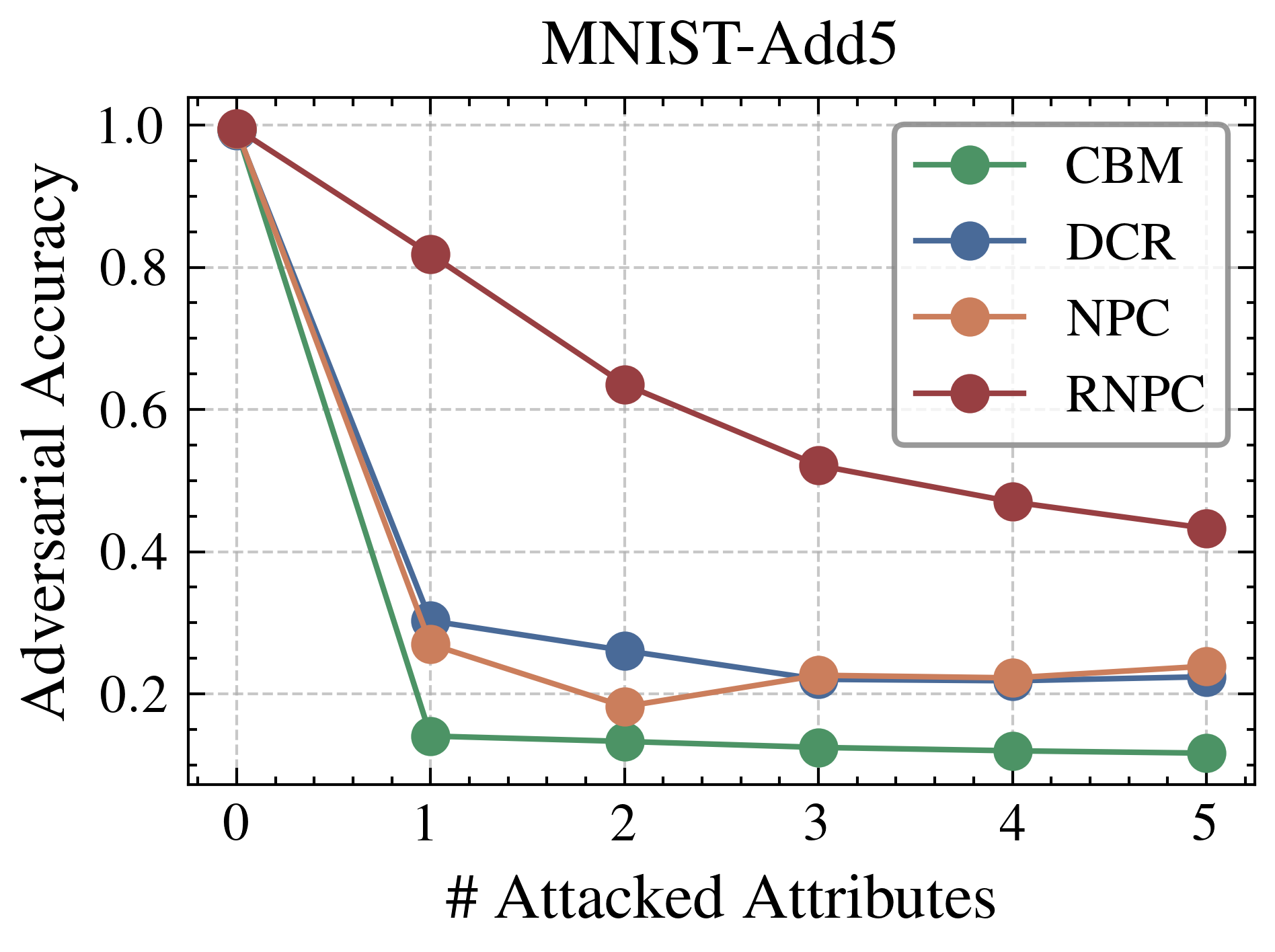}
    \end{minipage}
    \begin{minipage}{0.30\textwidth}
        \centering
        \includegraphics[width=\textwidth]{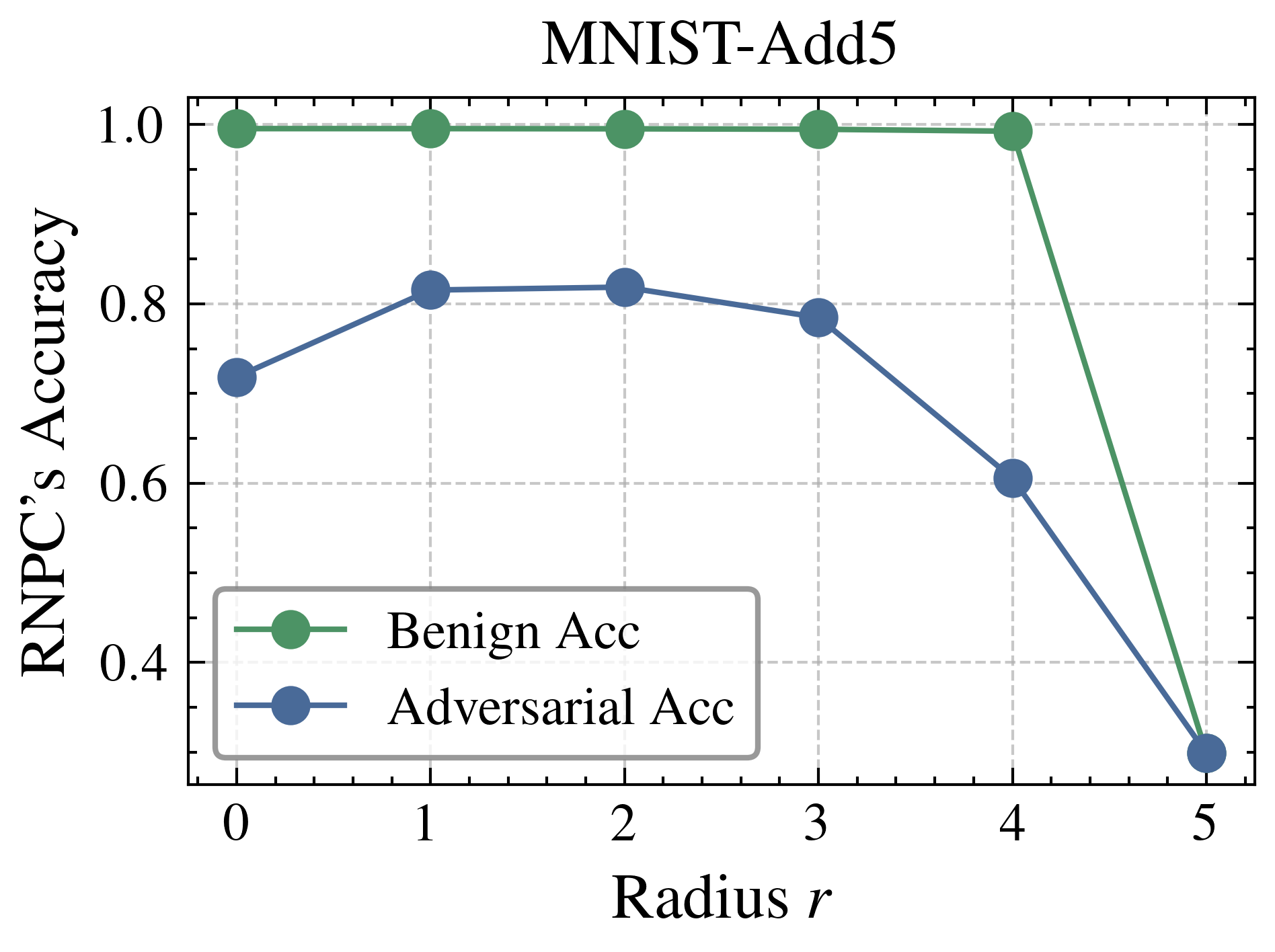}
    \end{minipage}
    \captionsetup{skip=1pt}
    \caption{
    \small
    \textbf{(a-b):}
    Adversarial accuracy under the $\infty$-norm-bounded PGD attack (norm bound = 0.11) with varying numbers of attacked attributes on MNIST-Add3 and MNIST-Add5. 
    \textbf{(c):}
    Performance of \rnpcname~with different values of $r$ on MNIST-Add5. This dataset has 5 attributes, and its attribute set has a radius of $r^*=2$.
    }
    \label{fig:combined}
\vspace{-0.1em}
\end{figure}

\paragraph{Impact of the radius.}
In Section~\ref{sec:notation}, we define $r$ as the radius of an attribute set, which is determined by the intrinsic structure of this set and fixed once the set is given. To avoid confusion, we denote this intrinsic radius of an attribute set as $r^*$. 
In Equation~(\ref{eq:rnpc_unml}), $r$ appears as a hyperparameter in the formulation of \rnpcname. Throughout the paper, we use $r = r^*$ by default. However, it can be adjusted, and in this section, we explore how varying $r$ affects the performance of \rnpcname.
We conduct this analysis on the MNIST-Add5 dataset, which has $K=5$ attributes and an intrinsic radius $r^* = 2$. We vary RNPC’s radius hyperparameter $r$ in the range $[0, 5]$ and evaluate both benign and adversarial accuracy under an $\infty$-bounded PGD attack with a norm bound of $0.11$.

We begin with benign accuracy. As $r$ increases from $0$ to $4$, we observe a decreasing logit gap between the top-1 and top-2 predicted classes. Nevertheless, \rnpcname~maintains near-perfect accuracy across this range (see Figure~\ref{fig:combined} (c)). However, when $r$ reaches 5, accuracy drops sharply to 29.9\%. This is expected because, at $r = 5$, the neighborhood $\mathcal{N}(\tilde{y}, r)$ spans the entire attribute space, \ie, $\mathbb{P}_\theta(A_{1:K} \in \mathcal{N}(\tilde{y},r) | X) = \mathbb{P}_\theta(A_{1:K} \in \Omega | X) = 1$. As a result, Equation~(\ref{eq:rnpc_unml}) no longer incorporates any meaningful information from the input $X$.

Next, we examine adversarial performance. When $r \leqslant r^*$, decreasing $r$ reduces adversarial accuracy. Conversely, increasing $r$ beyond $r^*$ also causes accuracy degradation. This indicates that the intrinsic radius $r^*$ is optimal in the sense that deviating from $r^*$—in either direction—hurts robustness. Nevertheless, \rnpcname~remains substantially more robust than other baseline models in these settings, as the baselines achieve less than 40\% adversarial accuracy (see Figure~\ref{fig:pgd_attack} (middle)). Notably, when $r = 5$, adversarial accuracy falls to 29.9\%, matching the benign accuracy and leading to a prediction perturbation of $\Delta_{\theta, w}^{\rnpcname} = 0$. This is consistent with our theoretical findings, as Lemma~\ref{thm:rnpc_perturbation} indicates that the upper bound on perturbation becomes zero in this case.
Overall, when $r \neq K$, \rnpcname~demonstrates strong resilience to changes in $r$ under benign settings; while its adversarial accuracy declines when $r$ deviates from $r^*$, \rnpcname~remains superior to other baseline models.

\paragraph{Impact of spurious correlations.}
Figure \ref{fig:gtsrbsub} (a) illustrates the adversarial accuracy of CBM, DCR, \npcname, and \rnpcname~on the GTSRB-Sub dataset, under the $\infty$-norm-bounded PGD attacks with varying norm bounds.
We observe that \rnpcname~and \npcname~exhibit similar performance, with both achieving higher adversarial accuracy compared to CBM and DCR under attacks with small norm bounds. However, as the norm bound increases, the advantage of \rnpcname~and \npcname~gradually diminishes. 
We hypothesize that, when the attribute recognition model is trained to recognize a specific attribute, the model might capture an unintended relationship between this attribute and the co-occurring features of other attributes.
For instance, the shape ``diamond'' always co-occurs with the color ``white'' on the GTSRB-Sub dataset.
Consequently, the model might rely on the features of the ``diamond'' shape to determine whether the color of an input image is ``white''.
Such unintended relationships, known as \textit{spurious correlations}, are a common phenomenon in neural networks~\citep{shortcut, ye2024spurious}.
Due to the potential spurious correlations, attacking one attribute leads to attacking multiple attributes.
We name such a phenomenon as \textit{attack propagation}.

The results in Figure \ref{fig:gtsrbsub} (b) validate our hypothesis.
Specifically, attacking any single attribute on the GTSRB-Sub dataset leads to a significant drop in the accuracy of recognizing other attributes.
That is, although the attack targets only one attribute, the attack propagation induces more attacked attributes. As discussed in the study on the impact of the number of attacked attributes, when this number exceeds the radius of a dataset ($r=1$ for GTSRB-Sub), the performance of \rnpcname~could be compromised.
Potential solutions for mitigating the attack propagation are discussed in Section~\ref{sec:discussion}.
More ablation studies are deferred to Appendix~\ref{app:ablation}.


\begin{figure}[t]
    \centering
    \begin{minipage}{0.30\textwidth}
        \centering
        \includegraphics[width=\textwidth]{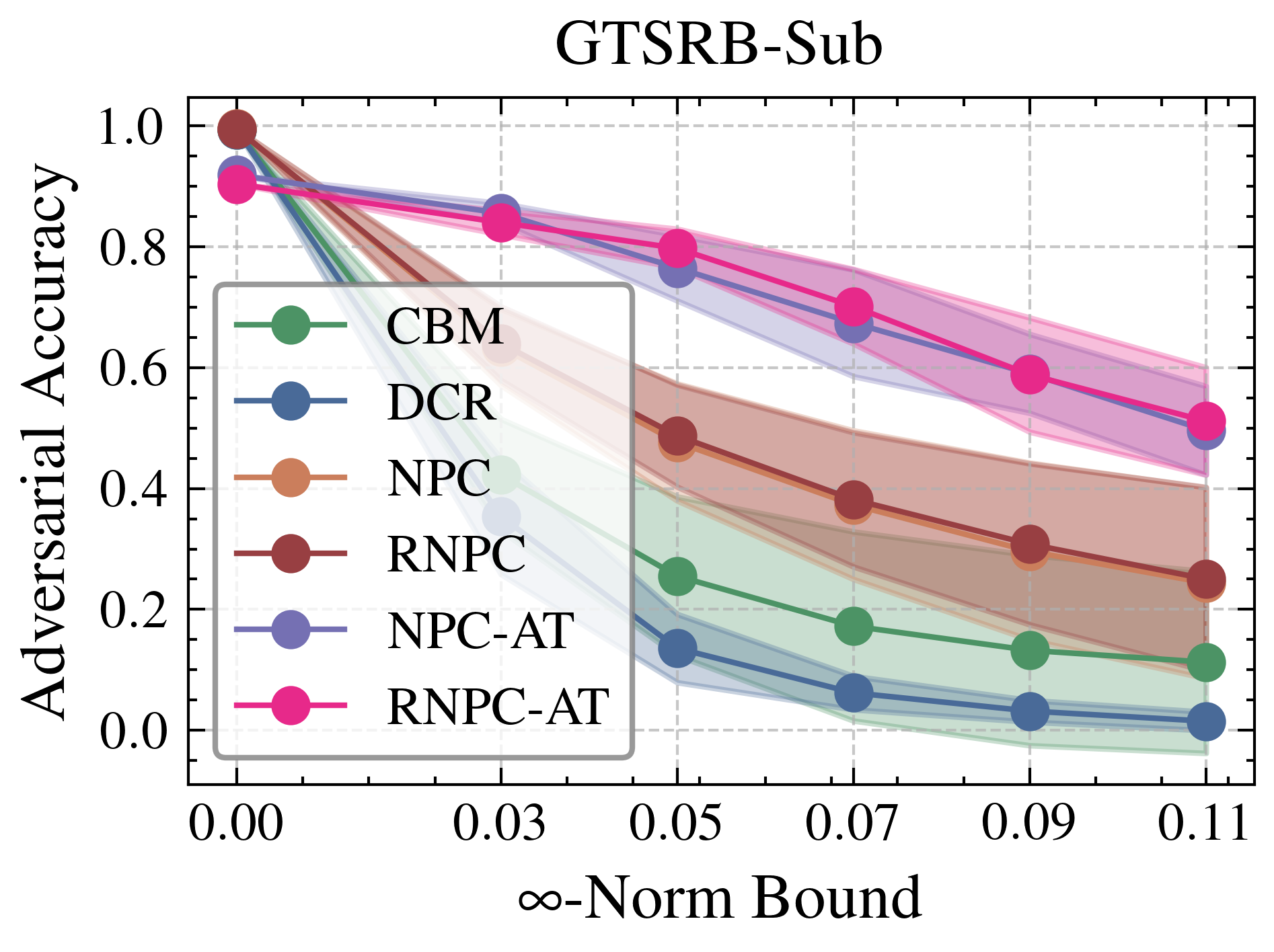}
    \end{minipage}
    \begin{minipage}{0.30\textwidth}
        \centering
        \includegraphics[width=\textwidth]{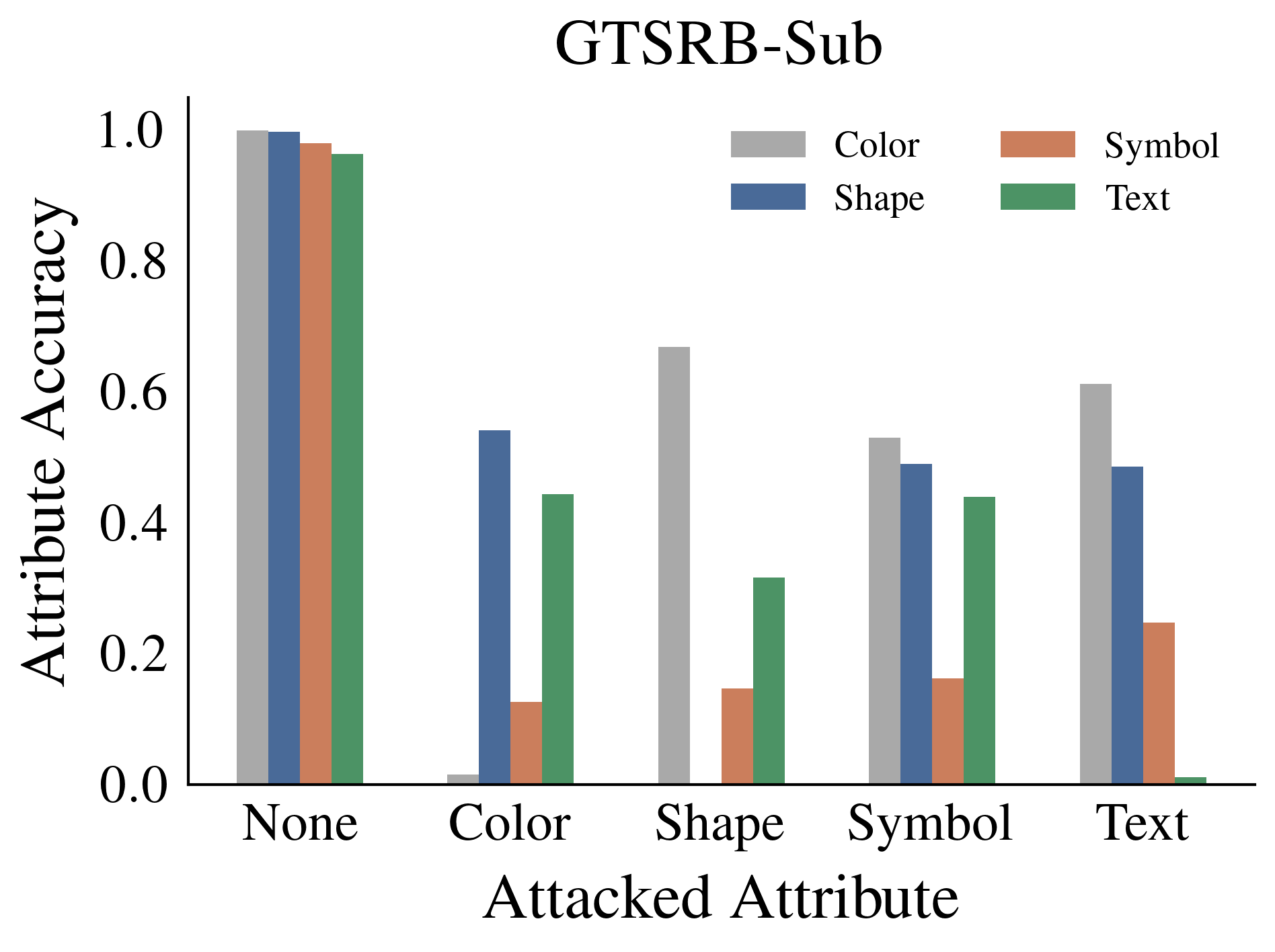}
    \end{minipage}
    \begin{minipage}{0.30\textwidth}
        \centering
        \includegraphics[width=\textwidth]{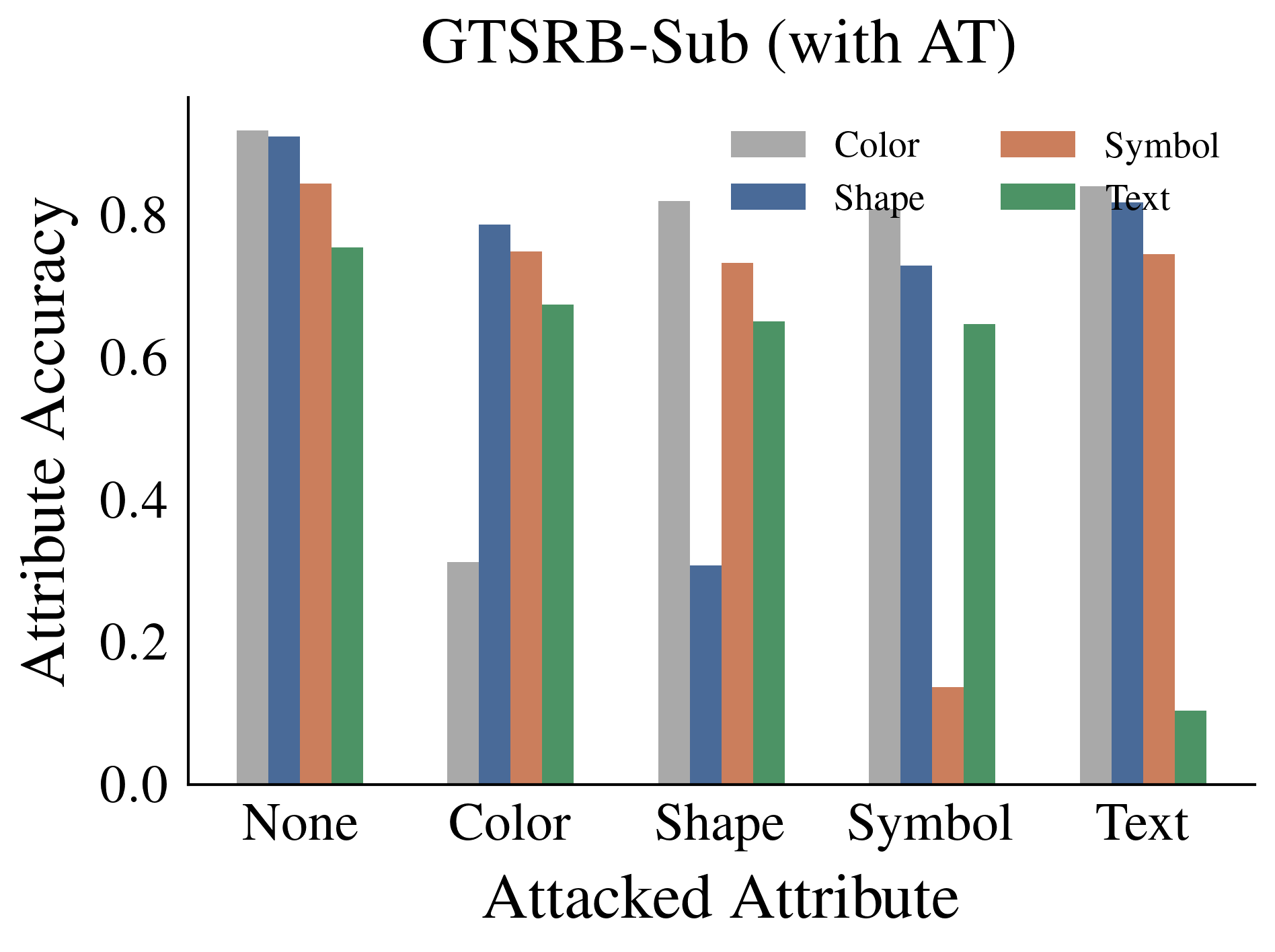}
    \end{minipage}
    \captionsetup{skip=1pt}
    \caption{
    \small
    \textbf{(a):}
    Adversarial accuracy under the $\infty$-norm-bounded PGD attack with varying norm bounds on GTSRB-Sub. The attacker attacks a single attribute at a time.
    Solid lines and shaded regions represent the mean adversarial accuracy and the standard deviation, computed across all attacked attributes.
    \textbf{(b):}
    Accuracy of NPC and RNPC's attribute recognition model in recognizing various attributes under the $\infty$-norm-bounded PGD attack (norm bound = 0.11) targeting a different single attribute.
    \textbf{(c):}
    Same setting as (b), but using \textit{adversarially trained} attribute recognition models for NPC and RNPC.
    }
    \captionsetup{skip=2pt}
    \label{fig:gtsrbsub}
\vspace{-0.1em}
\end{figure}
\section{Related work} \label{sec:related}
\paragraph{Robustness of concept bottleneck models.}
Exploring the robustness of CBMs against adversarial attacks is crucial for understanding their reliability in practical applications.
Specifically, \citet{robustness_cbm2} demonstrate that when the predicted concept probabilities are perturbed by attacks, CBMs often produce incorrect predictions. 
They further investigate \textit{how to ensure that the predicted concept probabilities remain unchanged under adversarial attacks, thereby improving the robustness of CBMs}, and propose a training algorithm for CBMs based on adversarial training.
Differently, our work admits the changes in the predicted concept probabilities and explores the question:
\textit{Can we make CBMs robust even when the predicted concept probabilities are perturbed by adversarial attacks?}
A comprehensive literature review on CBMs, adversarial attacks, the robustness of CBMs, and the robustness of DNNs using probabilistic circuits is provided in Appendix \ref{app:related}.
\section{Discussion} \label{sec:discussion}
In this section and Appendix \ref{app:discussion}, we discuss the limitations of \rnpcname~and explore potential solutions. 

\paragraph{Mitigating the attack propagation effect.}
Section \ref{sec:exp_ablation} shows that \rnpcname~suffers from the attack propagation effect, which arises from spurious correlations learned by the attribute recognition model. Conventional solutions for mitigating these correlations include data augmentation~\citep{data_augmentation_survey, data_augmentation}, counterfactual data generation~\citep{counterfactual_generation1, counterfactual_generation2}, etc.
Here, we propose an adversarial-training-based approach that leverages adversarial examples to disentangle spurious correlations between attributes.
Suppose the model relies on the feature of \textit{white} to identify the shape as \textit{diamond}. An adversarial example that perturbs the color attribute can shift its feature from \textit{white} toward some another color, thereby weakening its association with \textit{diamond}.

To achieve this, we generate adversarial examples that target a randomly selected attribute during training and train the model on both the adversarial and benign samples. 
As shown in Figure \ref{fig:gtsrbsub} (c), compared to models trained without adversarial training (see Figure \ref{fig:gtsrbsub} (b)), adversarially trained models exhibit significantly smaller drops in accuracy of other attributes when a specific attribute is attacked.
This demonstrates a reduction in the attack propagation effect.
Consequently, Figure \ref{fig:gtsrbsub} (a) shows that the robustness of both \npcname~and \rnpcname~is enhanced, outperforming other baseline models across different attack norm bounds.

\section{Conclusions}
In this paper, we delve into the adversarial robustness of Neural Probabilistic Circuits (\npcname s), showing that incorporating a probabilistic circuit into a model's architecture does not affect the robustness of the overall model.
Moreover, we improve the robustness of \npcname s by introducing a class-wise integration inference approach that produces robust predictions.
Both theoretical and empirical results across various datasets and attacks demonstrate that the resulting model, named \rnpcname, achieves higher robustness.
Due to the space limit, a more detailed conclusion is presented in Appendix \ref{app:conclusion}.

\section*{Acknowledgment}
This work is supported by NSF IIS grant No.\ 2416897, NSF III grant No.\ 2504555 and an NSF CAREER Award No.\ 2442290. HZ would like to thank the support of a Google Research Scholar Award and Nvidia Academic Grant Award. The views and conclusions expressed in this paper are solely those of the authors and do not necessarily reflect the official policies or positions of the supporting companies and government agencies.

\bibliographystyle{unsrtnat}
\bibliography{neurips_2025}

\begin{thebibliography}{56}
\providecommand{\natexlab}[1]{#1}
\providecommand{\url}[1]{\texttt{#1}}
\expandafter\ifx\csname urlstyle\endcsname\relax
  \providecommand{\doi}[1]{doi: #1}\else
  \providecommand{\doi}{doi: \begingroup \urlstyle{rm}\Url}\fi

\bibitem[Rudin(2019)]{stop}
Cynthia Rudin.
\newblock Stop explaining black box machine learning models for high stakes decisions and use interpretable models instead.
\newblock \emph{Nat. Mach. Intell.}, 1\penalty0 (5):\penalty0 206--215, 2019.

\bibitem[Koh et~al.(2020)Koh, Nguyen, Tang, Mussmann, Pierson, Kim, and Liang]{cbm}
Pang~Wei Koh, Thao Nguyen, Yew~Siang Tang, Stephen Mussmann, Emma Pierson, Been Kim, and Percy Liang.
\newblock Concept bottleneck models.
\newblock In \emph{ICML}, volume 119 of \emph{Proceedings of Machine Learning Research}, pages 5338--5348, 2020.

\bibitem[Barbiero et~al.(2023)Barbiero, Ciravegna, Giannini, Zarlenga, Magister, Tonda, Lio, Precioso, Jamnik, and Marra]{dcr}
Pietro Barbiero, Gabriele Ciravegna, Francesco Giannini, Mateo~Espinosa Zarlenga, Lucie~Charlotte Magister, Alberto Tonda, Pietro Lio, Fr{\'{e}}d{\'{e}}ric Precioso, Mateja Jamnik, and Giuseppe Marra.
\newblock Interpretable neural-symbolic concept reasoning.
\newblock In \emph{ICML}, volume 202 of \emph{Proceedings of Machine Learning Research}, pages 1801--1825, 2023.

\bibitem[Kim et~al.(2023)Kim, Jung, Park, Kim, and Yoon]{probcbm}
Eunji Kim, Dahuin Jung, Sangha Park, Siwon Kim, and Sungroh Yoon.
\newblock Probabilistic concept bottleneck models.
\newblock In \emph{ICML}, volume 202 of \emph{Proceedings of Machine Learning Research}, pages 16521--16540, 2023.

\bibitem[Y{\"{u}}ksekg{\"{o}}n{\"{u}}l et~al.(2023)Y{\"{u}}ksekg{\"{o}}n{\"{u}}l, Wang, and Zou]{posthoc_cbm}
Mert Y{\"{u}}ksekg{\"{o}}n{\"{u}}l, Maggie Wang, and James Zou.
\newblock Post-hoc concept bottleneck models.
\newblock In \emph{ICLR}, 2023.

\bibitem[Oikarinen et~al.(2023)Oikarinen, Das, Nguyen, and Weng]{labelfree_cbm}
Tuomas~P. Oikarinen, Subhro Das, Lam~M. Nguyen, and Tsui{-}Wei Weng.
\newblock Label-free concept bottleneck models.
\newblock In \emph{ICLR}, 2023.

\bibitem[Chen et~al.(2024)Chen, Yu, Shao, Sha, and Zhao]{npc}
Weixin Chen, Simon Yu, Huajie Shao, Lui Sha, and Han Zhao.
\newblock Neural probabilistic circuits: Enabling compositional and interpretable predictions through logical reasoning.
\newblock \emph{arXiv preprint arXiv:2501.07021}, 2024.

\bibitem[Poon and Domingos(2011)]{probabilistic_circuit}
Hoifung Poon and Pedro~M. Domingos.
\newblock Sum-product networks: {A} new deep architecture.
\newblock In \emph{UAI}, pages 337--346, 2011.

\bibitem[Goodfellow et~al.(2015)Goodfellow, Shlens, and Szegedy]{fgsd_attack}
Ian~J. Goodfellow, Jonathon Shlens, and Christian Szegedy.
\newblock Explaining and harnessing adversarial examples.
\newblock In \emph{ICLR}, 2015.

\bibitem[Madry et~al.(2018)Madry, Makelov, Schmidt, Tsipras, and Vladu]{pgd_attack}
Aleksander Madry, Aleksandar Makelov, Ludwig Schmidt, Dimitris Tsipras, and Adrian Vladu.
\newblock Towards deep learning models resistant to adversarial attacks.
\newblock In \emph{ICLR}, 2018.

\bibitem[Carlini and Wagner(2017)]{cw_attack}
Nicholas Carlini and David~A. Wagner.
\newblock Towards evaluating the robustness of neural networks.
\newblock In \emph{IEEE SP}, pages 39--57, 2017.

\bibitem[Chen and Zhang(2020)]{xiaoyi2020layer}
Xiaoyi Chen and Ni~Zhang.
\newblock Layer-wise adversarial training approach to improve adversarial robustness.
\newblock In \emph{IJCNN}, pages 1--8, 2020.

\bibitem[Rahnama et~al.(2020)Rahnama, Nguyen, and Raff]{rahnama2020robust}
Arash Rahnama, Andre~T Nguyen, and Edward Raff.
\newblock Robust design of deep neural networks against adversarial attacks based on lyapunov theory.
\newblock In \emph{CVPR}, pages 8178--8187, 2020.

\bibitem[Sinha et~al.(2023)Sinha, Huai, Sun, and Zhang]{robustness_cbm2}
Sanchit Sinha, Mengdi Huai, Jianhui Sun, and Aidong Zhang.
\newblock Understanding and enhancing robustness of concept-based models.
\newblock In \emph{AAAI}, pages 15127--15135, 2023.

\bibitem[Gens and Domingos(2013)]{learnspn}
Robert Gens and Pedro~M. Domingos.
\newblock Learning the structure of sum-product networks.
\newblock In \emph{ICML}, pages 873--880, 2013.

\bibitem[Zhao et~al.(2016)Zhao, Poupart, and Gordon]{cccp}
Han Zhao, Pascal Poupart, and Geoffrey~J. Gordon.
\newblock A unified approach for learning the parameters of sum-product networks.
\newblock In \emph{NeurIPS}, pages 433--441, 2016.

\bibitem[LeCun et~al.(1998)LeCun, Bottou, Bengio, and Haffner]{mnist}
Yann LeCun, L{\'{e}}on Bottou, Yoshua Bengio, and Patrick Haffner.
\newblock Gradient-based learning applied to document recognition.
\newblock \emph{Proc. {IEEE}}, 86\penalty0 (11):\penalty0 2278--2324, 1998.

\bibitem[Manhaeve et~al.(2018)Manhaeve, Dumancic, Kimmig, Demeester, and Raedt]{deepproblog}
Robin Manhaeve, Sebastijan Dumancic, Angelika Kimmig, Thomas Demeester, and Luc~De Raedt.
\newblock Deepproblog: Neural probabilistic logic programming.
\newblock In \emph{NeurIPS}, pages 3753--3763, 2018.

\bibitem[Bortolotti et~al.(2024)Bortolotti, Marconato, Carraro, Morettin, van Krieken, Vergari, Teso, and Passerini]{rsbench}
Samuele Bortolotti, Emanuele Marconato, Tommaso Carraro, Paolo Morettin, Emile van Krieken, Antonio Vergari, Stefano Teso, and Andrea Passerini.
\newblock A neuro-symbolic benchmark suite for concept quality and reasoning shortcuts.
\newblock In \emph{NeurIPS (Track on Datasets and Benchmarks)}, 2024.

\bibitem[Liu et~al.(2015)Liu, Luo, Wang, and Tang]{celeba}
Ziwei Liu, Ping Luo, Xiaogang Wang, and Xiaoou Tang.
\newblock Deep learning face attributes in the wild.
\newblock In \emph{ICCV}, pages 3730--3738, 2015.

\bibitem[Choi et~al.(2018)Choi, Choi, Kim, Ha, Kim, and Choo]{stargan}
Yunjey Choi, Min{-}Je Choi, Munyoung Kim, Jung{-}Woo Ha, Sunghun Kim, and Jaegul Choo.
\newblock Stargan: Unified generative adversarial networks for multi-domain image-to-image translation.
\newblock In \emph{CVPR}, pages 8789--8797, 2018.

\bibitem[Stallkamp et~al.(2011)Stallkamp, Schlipsing, Salmen, and Igel]{gtsrb}
Johannes Stallkamp, Marc Schlipsing, Jan Salmen, and Christian Igel.
\newblock The german traffic sign recognition benchmark: {A} multi-class classification competition.
\newblock In \emph{IJCNN}, pages 1453--1460, 2011.

\bibitem[Geirhos et~al.(2020)Geirhos, Jacobsen, Michaelis, Zemel, Brendel, Bethge, and Wichmann]{shortcut}
Robert Geirhos, J{\"{o}}rn{-}Henrik Jacobsen, Claudio Michaelis, Richard~S. Zemel, Wieland Brendel, Matthias Bethge, and Felix~A. Wichmann.
\newblock Shortcut learning in deep neural networks.
\newblock \emph{Nat. Mach. Intell.}, 2\penalty0 (11):\penalty0 665--673, 2020.

\bibitem[Ye et~al.(2024)Ye, Zheng, Cao, Ma, and Zhang]{ye2024spurious}
Wenqian Ye, Guangtao Zheng, Xu~Cao, Yunsheng Ma, and Aidong Zhang.
\newblock Spurious correlations in machine learning: A survey.
\newblock \emph{arXiv preprint arXiv:2402.12715}, 2024.

\bibitem[Shorten and Khoshgoftaar(2019)]{data_augmentation_survey}
Connor Shorten and Taghi~M. Khoshgoftaar.
\newblock A survey on image data augmentation for deep learning.
\newblock \emph{J. Big Data}, 6:\penalty0 60, 2019.

\bibitem[Ilse et~al.(2021)Ilse, Tomczak, and Forr{\'{e}}]{data_augmentation}
Maximilian Ilse, Jakub~M. Tomczak, and Patrick Forr{\'{e}}.
\newblock Selecting data augmentation for simulating interventions.
\newblock In \emph{ICML}, volume 139 of \emph{Proceedings of Machine Learning Research}, pages 4555--4562, 2021.

\bibitem[Sauer and Geiger(2021)]{counterfactual_generation1}
Axel Sauer and Andreas Geiger.
\newblock Counterfactual generative networks.
\newblock In \emph{ICLR}, 2021.

\bibitem[Kaushik et~al.(2020)Kaushik, Hovy, and Lipton]{counterfactual_generation2}
Divyansh Kaushik, Eduard~H. Hovy, and Zachary~Chase Lipton.
\newblock Learning the difference that makes {A} difference with counterfactually-augmented data.
\newblock In \emph{ICLR}, 2020.

\bibitem[Zarlenga et~al.(2022)Zarlenga, Barbiero, Ciravegna, Marra, Giannini, Diligenti, Shams, Precioso, Melacci, Weller, Li{\'{o}}, and Jamnik]{cem}
Mateo~Espinosa Zarlenga, Pietro Barbiero, Gabriele Ciravegna, Giuseppe Marra, Francesco Giannini, Michelangelo Diligenti, Zohreh Shams, Fr{\'{e}}d{\'{e}}ric Precioso, Stefano Melacci, Adrian Weller, Pietro Li{\'{o}}, and Mateja Jamnik.
\newblock Concept embedding models: Beyond the accuracy-explainability trade-off.
\newblock In \emph{NeurIPS}, 2022.

\bibitem[Yeh et~al.(2020)Yeh, Kim, Arik, Li, Pfister, and Ravikumar]{concept_emb1}
Chih{-}Kuan Yeh, Been Kim, Sercan~{\"{O}}mer Arik, Chun{-}Liang Li, Tomas Pfister, and Pradeep Ravikumar.
\newblock On completeness-aware concept-based explanations in deep neural networks.
\newblock In \emph{NeurIPS}, 2020.

\bibitem[Kazhdan et~al.(2020)Kazhdan, Dimanov, Jamnik, Li{\`{o}}, and Weller]{concept_emb2}
Dmitry Kazhdan, Botty Dimanov, Mateja Jamnik, Pietro Li{\`{o}}, and Adrian Weller.
\newblock Now you see me {(CME):} concept-based model extraction.
\newblock In \emph{CIKM}, volume 2699 of \emph{{CEUR} Workshop Proceedings}, 2020.

\bibitem[Mahinpei et~al.(2021)Mahinpei, Clark, Lage, Doshi-Velez, and Pan]{hybrid_cbm}
Anita Mahinpei, Justin Clark, Isaac Lage, Finale Doshi-Velez, and Weiwei Pan.
\newblock Promises and pitfalls of black-box concept learning models.
\newblock \emph{arXiv preprint arXiv:2106.13314}, 2021.

\bibitem[Rodr{\'{\i}}guez et~al.(2024)Rodr{\'{\i}}guez, Cu{\'{e}}llar, and Morales]{soft_tree_cbm}
David~M. Rodr{\'{\i}}guez, Manuel~P. Cu{\'{e}}llar, and Diego~Pedro Morales.
\newblock Concept logic trees: enabling user interaction for transparent image classification and human-in-the-loop learning.
\newblock \emph{Appl. Intell.}, 54\penalty0 (5):\penalty0 3667--3679, 2024.

\bibitem[Akhtar and Mian(2018)]{adversarial_survey}
Naveed Akhtar and Ajmal~S. Mian.
\newblock Threat of adversarial attacks on deep learning in computer vision: {A} survey.
\newblock \emph{{IEEE} Access}, 6:\penalty0 14410--14430, 2018.

\bibitem[Kurakin et~al.(2017)Kurakin, Goodfellow, and Bengio]{bim_attack}
Alexey Kurakin, Ian~J. Goodfellow, and Samy Bengio.
\newblock Adversarial examples in the physical world.
\newblock In \emph{ICLR Workshop}, 2017.

\bibitem[Papernot et~al.(2016)Papernot, McDaniel, Jha, Fredrikson, Celik, and Swami]{jsma_attack}
Nicolas Papernot, Patrick~D. McDaniel, Somesh Jha, Matt Fredrikson, Z.~Berkay Celik, and Ananthram Swami.
\newblock The limitations of deep learning in adversarial settings.
\newblock In \emph{EuroS{\&}P}, pages 372--387, 2016.

\bibitem[Moosavi{-}Dezfooli et~al.(2016)Moosavi{-}Dezfooli, Fawzi, and Frossard]{deepfool_attack}
Seyed{-}Mohsen Moosavi{-}Dezfooli, Alhussein Fawzi, and Pascal Frossard.
\newblock Deepfool: {A} simple and accurate method to fool deep neural networks.
\newblock In \emph{CVPR}, pages 2574--2582, 2016.

\bibitem[Rasheed et~al.(2024)Rasheed, Abdelhamid, Khan, Menezes, and Khattak]{robustness_cbm1}
Bader Rasheed, Mohamed Abdelhamid, Adil Khan, Igor Menezes, and Asad~Masood Khattak.
\newblock Exploring the impact of conceptual bottlenecks on adversarial robustness of deep neural networks.
\newblock \emph{{IEEE} Access}, 12:\penalty0 131323--131335, 2024.

\bibitem[G{\"{u}}rel et~al.(2021)G{\"{u}}rel, Qi, Rimanic, Zhang, and Li]{kemlp}
Nezihe~Merve G{\"{u}}rel, Xiangyu Qi, Luka Rimanic, Ce~Zhang, and Bo~Li.
\newblock Knowledge enhanced machine learning pipeline against diverse adversarial attacks.
\newblock In \emph{ICML}, volume 139 of \emph{Proceedings of Machine Learning Research}, pages 3976--3987, 2021.

\bibitem[Yang et~al.(2022)Yang, Zhao, Wang, Zhang, Li, Pei, Karlas, Liu, Guo, Zhang, and Li]{yang_reason}
Zhuolin Yang, Zhikuan Zhao, Boxin Wang, Jiawei Zhang, Linyi Li, Hengzhi Pei, Bojan Karlas, Ji~Liu, Heng Guo, Ce~Zhang, and Bo~Li.
\newblock Improving certified robustness via statistical learning with logical reasoning.
\newblock In \emph{NeurIPS}, 2022.

\bibitem[Zhang et~al.(2023)Zhang, Li, Zhang, and Li]{care}
Jiawei Zhang, Linyi Li, Ce~Zhang, and Bo~Li.
\newblock {CARE:} certifiably robust learning with reasoning via variational inference.
\newblock In \emph{SaTML}, pages 554--574, 2023.

\bibitem[Kang et~al.(2024)Kang, G{\"{u}}rel, Li, and Li]{colep}
Mintong Kang, Nezihe~Merve G{\"{u}}rel, Linyi Li, and Bo~Li.
\newblock {COLEP:} certifiably robust learning-reasoning conformal prediction via probabilistic circuits.
\newblock In \emph{ICLR}, 2024.

\bibitem[Richardson and Domingos(2006)]{mln}
Matthew Richardson and Pedro~M. Domingos.
\newblock Markov logic networks.
\newblock \emph{Mach. Learn.}, 62\penalty0 (1-2):\penalty0 107--136, 2006.

\bibitem[Krizhevsky(2009)]{cifar10}
Alex Krizhevsky.
\newblock Learning multiple layers of features from tiny images.
\newblock Technical report, University of Toronto, 2009.

\bibitem[Deng et~al.(2009)Deng, Dong, Socher, Li, Li, and Fei{-}Fei]{imagenet}
Jia Deng, Wei Dong, Richard Socher, Li{-}Jia Li, Kai Li, and Li~Fei{-}Fei.
\newblock Imagenet: {A} large-scale hierarchical image database.
\newblock In \emph{CVPR}, pages 248--255, 2009.

\bibitem[Radford et~al.(2021)Radford, Kim, Hallacy, Ramesh, Goh, Agarwal, Sastry, Askell, Mishkin, Clark, Krueger, and Sutskever]{clip}
Alec Radford, Jong~Wook Kim, Chris Hallacy, Aditya Ramesh, Gabriel Goh, Sandhini Agarwal, Girish Sastry, Amanda Askell, Pamela Mishkin, Jack Clark, Gretchen Krueger, and Ilya Sutskever.
\newblock Learning transferable visual models from natural language supervision.
\newblock In \emph{ICML}, volume 139 of \emph{Proceedings of Machine Learning Research}, pages 8748--8763, 2021.

\bibitem[Xu et~al.(2018)Xu, Zhang, Friedman, Liang, and den Broeck]{jingyi_semantic}
Jingyi Xu, Zilu Zhang, Tal Friedman, Yitao Liang, and Guy~Van den Broeck.
\newblock A semantic loss function for deep learning with symbolic knowledge.
\newblock In \emph{ICML}, volume~80 of \emph{Proceedings of Machine Learning Research}, pages 5498--5507, 2018.

\bibitem[Winters et~al.(2022)Winters, Marra, Manhaeve, and Raedt]{DeepStochLog}
Thomas Winters, Giuseppe Marra, Robin Manhaeve, and Luc~De Raedt.
\newblock Deepstochlog: Neural stochastic logic programming.
\newblock In \emph{AAAI}, pages 10090--10100, 2022.

\bibitem[van Krieken et~al.(2023)van Krieken, Thanapalasingam, Tomczak, van Harmelen, and ten Teije]{NeSI}
Emile van Krieken, Thiviyan Thanapalasingam, Jakub~M. Tomczak, Frank van Harmelen, and Annette ten Teije.
\newblock A-nesi: {A} scalable approximate method for probabilistic neurosymbolic inference.
\newblock In \emph{NeurIPS}, 2023.

\bibitem[Huang et~al.(2021)Huang, Li, Chen, Samel, Naik, Song, and Si]{Scallop}
Jiani Huang, Ziyang Li, Binghong Chen, Karan Samel, Mayur Naik, Le~Song, and Xujie Si.
\newblock Scallop: From probabilistic deductive databases to scalable differentiable reasoning.
\newblock In \emph{NeurIPS}, pages 25134--25145, 2021.

\bibitem[Ahmed et~al.(2022)Ahmed, Teso, Chang, den Broeck, and Vergari]{spl}
Kareem Ahmed, Stefano Teso, Kai{-}Wei Chang, Guy~Van den Broeck, and Antonio Vergari.
\newblock Semantic probabilistic layers for neuro-symbolic learning.
\newblock In \emph{NeurIPS}, 2022.

\bibitem[Ahmed et~al.(2023)Ahmed, Chang, and den Broeck]{pseudo_semantic}
Kareem Ahmed, Kai{-}Wei Chang, and Guy~Van den Broeck.
\newblock A pseudo-semantic loss for autoregressive models with logical constraints.
\newblock In \emph{NeurIPS}, 2023.

\bibitem[Pryor et~al.(2023)Pryor, Dickens, Augustine, Albalak, Wang, and Getoor]{NeuPSL}
Connor Pryor, Charles Dickens, Eriq Augustine, Alon Albalak, William~Yang Wang, and Lise Getoor.
\newblock Neupsl: Neural probabilistic soft logic.
\newblock In \emph{IJCAI}, pages 4145--4153, 2023.

\bibitem[Yang et~al.(2020)Yang, Ishay, and Lee]{NeurASP}
Zhun Yang, Adam Ishay, and Joohyung Lee.
\newblock Neurasp: Embracing neural networks into answer set programming.
\newblock In \emph{IJCAI}, pages 1755--1762, 2020.

\bibitem[Kim(2020)]{kim2020torchattacks}
Hoki Kim.
\newblock Torchattacks: A pytorch repository for adversarial attacks.
\newblock \emph{arXiv preprint arXiv:2010.01950}, 2020.

\bibitem[L{\'{e}}cuyer et~al.(2019)L{\'{e}}cuyer, Atlidakis, Geambasu, Hsu, and Jana]{dp_robustness}
Mathias L{\'{e}}cuyer, Vaggelis Atlidakis, Roxana Geambasu, Daniel Hsu, and Suman Jana.
\newblock Certified robustness to adversarial examples with differential privacy.
\newblock In \emph{IEEE SP}, pages 656--672, 2019.

\end{thebibliography}

\newpage
\appendix
\startcontents
\printcontents{}{1}{\section*{Appendix contents}}

\section{More related work} \label{app:related}
\paragraph{Concept bottleneck models.}
Concept bottleneck models (CBMs) were first introduced in~\citet{cbm}. These models are constructed by combining a concept recognition model with a task predictor. The concept recognition model, typically based on neural networks, takes an image as input and outputs probabilities for various high-level concepts, such as ``red\_color'' and ``white\_color''. The task predictor, often a linear model, determines the final class based on these predicted concept probabilities. Thanks to the simplicity of the linear predictor, the class predictions can be interpreted in terms of the predicted concept probabilities.

Follow-up works primarily focus on improving two aspects: the model's performance on downstream tasks and its interpretability.
\textbf{1) Improving task performance:}
Rather than using a layer of concept probabilities as the bottleneck, \citet{cem, concept_emb1, concept_emb2, probcbm, hybrid_cbm} propose the use of concept embeddings. Specifically, the concept recognition model generates high-dimensional embeddings for various concepts, and the task predictor determines the final class based on these embeddings. Since an embedding typically encodes more information than a single probability, these models often achieve better performance on downstream tasks compared to vanilla CBMs. 
However, the interpretability of these models is significantly compromised because the dimensions of the embeddings lack clear semantic meaning, and class predictions cannot be interpreted using the semantics of these embeddings.
\textbf{2) Enhancing model interpretability:} 
To further enhance the interpretability of CBMs, several works introduce novel interpretable architectures for the task predictor. For instance, \citet{dcr} propose a Deep Concept Reasoner (DCR) that allows class predictions to be interpreted through learned logical rules over the predicted concepts. Similarly, \citet{soft_tree_cbm} employ a soft decision tree, where class predictions are generated by following specific branches within the tree. More recently, \citet{npc} explore the use of probabilistic circuits as task predictors, introducing a new model called Neural Probabilistic Circuits (\npcname s).
Specifically, \npcname~treats each group of concepts (\eg, ``red\_color'', ``white\_color'') as one attribute (\eg, ``color'') and consists of two modules: an attribute recognition model and a probabilistic circuit. The neural-network-based attribute recognition model takes an image as input and outputs probability vectors for various human-understandable attributes. The probabilistic circuit learns the joint distribution over the class variable and the attribute variables, while also supporting tractable marginal and conditional inference. 
Combining the outputs from the attribute recognition model and the probabilistic circuit, \npcname~produces class predictions that can be interpreted using the predicted probabilities of various attributes and the conditional dependencies between attributes and classes.
Furthermore, \citet{npc} demonstrate that \npcname~exhibits a compositional estimation error, which is upper bounded by a linear combination of errors from its individual modules. Thanks to the integration of the probabilistic circuit, \npcname~achieves performance competitive with end-to-end DNNs while offering enhanced interpretability.

\paragraph{Adversarial attacks.}
Adversarial attacks~\citep{adversarial_survey} refer to the process of deliberately crafting small, often imperceptible, perturbations to input images with the aim of misleading neural networks into producing incorrect predictions. 
To ensure the imperceptibility of the crafted perturbations, adversarial attacks are typically norm-bounded, meaning the magnitude of the perturbation is constrained under a specified norm, such as \(L_1\), \(L_2\), or \(L_\infty\).
Classical adversarial attacks include FGSM~\citep{fgsd_attack}, BIM~\citep{bim_attack}, PGD~\citep{pgd_attack}, CW~\citep{cw_attack}, JSMA~\citep{jsma_attack}, and DeepFool~\citep{deepfool_attack}.
The success of adversarial attacks underscores the vulnerabilities of neural networks, raising critical concerns about their robustness and reliability in practical applications.

According to the attacker's knowledge of the model, adversarial attacks can be categorized into \textit{white-box} attacks and \textit{black-box} attacks. In white-box attacks, the attacker has full access to the model's architecture and parameters, whereas black-box attacks only have access to the model's outputs, relying on query-based or transfer-based strategies to generate adversarial perturbations.
According to the attacker's goal, adversarial attacks can be categorized into \textit{targeted} attacks and \textit{untargeted} attacks. In targeted attacks, the attacker aims to mislead the model into predicting a specific, incorrect label, while untargeted attacks focus on causing the model to produce any incorrect output. 
In this paper, we focus on white-box, untargeted adversarial attacks. In particular, we select the $\infty$-norm-bounded PGD attack~\citep{pgd_attack}, the 2-norm-bounded PGD attack~\citep{pgd_attack}, and the 2-norm-bounded CW attack~\citep{cw_attack}.

\paragraph{Robustness of concept bottleneck models.}
By incorporating high-level, human-understandable concepts as an intermediate layer, Concept Bottleneck Models (CBMs) provide interpretable predictions that can be explained through the predicted concepts, thereby enhancing their reliability in practical applications.
However, as the architectures of CBMs typically rely on neural networks, they can be vulnerable to adversarial attacks~\citep{fgsd_attack}.
Exploring such vulnerabilities is crucial for understanding the potential threats underlying CBMs.
In particular, \citet{robustness_cbm1, robustness_cbm2}, as well as our work, investigate the robustness of CBMs against adversarial attacks.
Despite this shared focus, these studies address distinctly different problems.

\citet{robustness_cbm1} investigate the question:
\textit{Compared to DNNs, how robust are CBMs against adversarial attacks designed to mislead class predictions?}
Their findings reveal that CBMs inherently exhibit higher robustness than their standard DNN counterparts. 
In contrast, \citet{robustness_cbm2} and our work focus on adversarial attacks that target concept predictions rather than class predictions.

\citet{robustness_cbm2} demonstrate that when the predicted concept probabilities are perturbed by adversarial attacks, CBMs often produce incorrect predictions. 
Given this vulnerability, they investigate \textit{how to ensure that the predicted concept probabilities remain unchanged under adversarial attacks, thereby improving the robustness of CBMs}.
To achieve this, they propose a training algorithm for CBMs based on adversarial training.

In contrast, our work admits the changes in the predicted concept probabilities and explores a different question:
\textit{Can we make a CBM, \npcname~in particular, robust even when the predicted concept probabilities are perturbed by adversarial attacks?}
We demonstrate that by employing a class-wise integration approach, the final predictions of \npcname~are provably more robust.
We also theoretically show that, unlike the linear model, the probabilistic circuit on top of the recognition module is free for robustness.

\paragraph{Improving robustness of end-to-end DNNs using probabilistic circuits.}
A line of research~\citep{kemlp, yang_reason, care, colep} explores leveraging probabilistic models—such as Markov logic networks~\citep{mln} and probabilistic circuits~\citep{probabilistic_circuit}—to enhance the adversarial robustness of end-to-end DNNs. 
These approaches typically rely on a high-performance but vulnerable DNN to predict class labels for inputs that may contain adversarial perturbations.
A probabilistic-model-based reasoning module is then used to correct potentially erroneous predictions made by the DNN.
This predict-then-correct paradigm contrasts with our approach, which aims to build a robust and interpretable model from scratch, without relying on a high-performance DNN as the prime predictor.

This fundamental difference in objective also leads to differences in the problem setting. Following the framework of concept bottleneck models~\citep{cbm, cem}, we assume access to a set of interpretable concepts/attributes that are sufficient to distinguish images from different classes. In contrast, the above research treats concept-based knowledge as auxiliary information used solely for correcting DNN predictions, and may consider only a limited set of attributes (\eg, shape alone).

Among the research, \citet{colep} also employ probabilistic circuits. However, there are key differences in how probabilistic circuits are utilized. Specifically, our approach fully exploits the expressive power of probabilistic circuits: we learn smooth and decomposable circuits that represent the joint distribution over attributes and classes. This structure enables efficient inference—joint, marginal, and conditional distributions can all be computed with at most two forward passes through the circuit, highlighting the advantage of tractable probabilistic reasoning.

In contrast, \citet{colep} use probabilistic circuits less efficiently. Rather than modeling the joint distribution explicitly through the circuit's structure and edge weights, they define a factor function that computes the factor of each instantiation of attributes and class labels, which is essentially the exponential of the corresponding joint probability. These factors are treated as leaf nodes in the circuit. 
When a particular instantiation is provided as input, a product node connecting the corresponding factor leaf is activated, causing the circuit to output the associated joint probability. As a result, their circuit behaves more like an arithmetic circuit composed of sum and product nodes that performs arithmetic using factors, rather than a typical probabilistic circuit with embedded probabilistic semantics and tractable inference capabilities.

\section{More discussions} \label{app:discussion}
In this section, we discuss the limitations of \rnpcname~and explore potential solutions and promising future directions.

\paragraph{Scaling to concept-annotation-free datasets.} Consistent with standard concept bottleneck models~\citep{cbm, dcr, cem}, our work assumes access to concept annotations within the dataset. However, many image classification datasets such as CIFAR-10~\citep{cifar10} and ImageNet~\citep{imagenet} lack such annotations, potentially limiting the applicability of our approach. Inspired by recent progress in label-free concept bottleneck models~\citep{posthoc_cbm, labelfree_cbm}, we can extend \rnpcname~to annotation-free settings by using multimodal models~\citep{clip} to transfer concepts from other datasets or from natural language descriptions of concepts.

\paragraph{Scaling to more attributes.}
The computational complexity of \rnpcname~is $O\left(\sum_{k=1}^K |f_k| + |S|\cdot|V|\right)$, where $|V| = \sum_{y\in\mathcal{Y}} |V_y| \in \left[|\mathcal{Y}|, \prod_{k=1}^K |\mathcal{A}_k| \right]$. 
This means the complexity can be as low as $|\mathcal{Y}|$ when each class only has one high-probability attribute assignment, but can grow exponentially with $K$ in the worst case.
Scalability remains a fundamental challenge for neuro-symbolic models, particularly those grounded in graphical model structures. Recent efforts have begun exploring approximation-based approaches to address this issue.
That said, as demonstrated in Proposition \ref{prop:complexity}, \rnpcname~achieves a substantial reduction in computational complexity compared to \npcname.

\paragraph{Augmenting the radius of a dataset.}
The inference procedure of \rnpcname~relies on neighborhoods of various classes defined by a specific radius. 
Experimental results in Section \ref{sec:exp_ablation} demonstrate that \rnpcname's performance under attacks is highly correlated with this radius. 
Specifically, when the number of attacked attributes does not exceed the radius, \rnpcname~typically achieves high adversarial accuracy. Moreover, a larger radius generally leads to higher robustness against attacks.
The radius is determined by the intrinsic properties of the dataset. For real-world datasets, the radius can be very small. For example, the original GTSRB dataset~\citep{gtsrb} has a minimum inter-class distance of 1, resulting in a radius of 0. Such a small radius can limit \rnpcname's robustness against attacks. Conversely, augmenting this radius can enhance \rnpcname's performance.

A naive approach to augmenting the radius is to repeat certain attributes in the attribute annotations.
For instance, the set $\{(0,0), (0,1)\}$ has a minimum inter-class distance of 1. 
If the second attribute is repeated, resulting in the set $\{(0,0,0),(0,1,1)\}$, the minimum inter-class distance increases by 1. 
More generally, if an attribute is repeated $m$ times, the minimum inter-class distance $d_{\min}$ can increase by up to $m$.
Another solution is to introduce new attributes that provide additional semantics for the downstream recognition tasks. These new attributes can potentially increase the radius. However, introducing new attributes requires additional annotation efforts, which may be challenging for large-scale datasets.

\paragraph{Reducing the complete information assumption.} Assumption~\ref{assump:complete} assumes that each input image contains complete information about all attributes, suggesting that the attributes are conditionally mutually independent given the input. This assumption, commonly referred to as the conditional independence assumption, is widely adopted in neuro-symbolic learning frameworks~\citep{deepproblog, jingyi_semantic, DeepStochLog, NeSI, Scallop}. From a computational perspective, this assumption enables a factorization of the joint distribution over attributes into independent marginal distributions, \ie, $\mathbb{P}_\theta (A_{1:K} \mid X) = \prod_{k=1}^K \mathbb{P}_{\theta_k} (A_k \mid X)$, which significantly reduces parameter complexity. Although this assumption is relatively mild, we consider the possibility of relaxing it in future work. When the complete information assumption does not hold, one could instead model the full joint distribution $\mathbb{P}_\theta (A_{1:K} \mid X)$ using a single expressive model~\citep{spl, pseudo_semantic, NeuPSL, NeurASP}. While this approach introduces higher parameter complexity, it allows the model to capture interdependencies among attributes and accurately represent more complex joint patterns.

\section{Detailed conclusions} \label{app:conclusion}
In this paper, we provide an understanding of the adversarial robustness of the Neural Probabilistic Circuit (\npcname).
Moreover, we improve the robustness of \npcname~by introducing a class-wise integration inference approach that produces robust predictions, and name the resulting model as \rnpcname.

Theorem \ref{thm:npc_perturbation} and Theorem \ref{thm:rnpc_perturbation} demonstrate the adversarial robustness of \npcname~and \rnpcname, respectively, showing that their robustness is upper bounded by the robustness of their attribute recognition models.
These results also suggest that using a probabilistic circuit as the task predictor does not impact the robustness of the overall model.
Theorem \ref{thm:comparison} compares these two bounds, showing that \rnpcname~enhances the robustness of \npcname~under certain conditions.
Furthermore, we analyze \rnpcname's benign performance on downstream tasks.
Theorem \ref{thm:compositional_error} demonstrates the compositional estimation error of \rnpcname, showing that its estimation error is upper bounded by a linear combination of errors from its individual modules.
Finally, Theorem \ref{thm:tradeoff} presents the distance between the optimal \rnpcname~and the ground-truth distribution, revealing a trade-off between adversarial robustness and benign performance.

Empirical evaluations on diverse image classification datasets, under three types of adversarial attacks with varying norm bounds, demonstrate that \rnpcname~achieves superior robustness compared with three baseline models while maintaining high accuracy on benign inputs.
Additionally, ablation studies on the impact of the number of attacked attributes show that \rnpcname~exhibits high robustness across varying numbers of attacked attributes.
Besides, we observe that the spurious correlations captured by the attribute recognition model can induce attack propagation, which may compromise the robustness of \rnpcname.

Overall, with the proposed class-wise integration inference approach, \rnpcname~achieves high robustness, capable of making correct class predictions even if the predicted attribute distributions are perturbed by adversarial attacks.

\section{Broader impact} \label{app:broad_impact}
This paper aims to understand and improve the robustness of neural probabilistic circuits against adversarial perturbations. Specifically, we propose a class-wise integration inference method and demonstrate that the resulting model is more robust under certain assumptions.
Improving adversarial robustness enhances the trustworthiness of machine learning models, making them more reliable for deployment in real-world scenarios, especially in high-stakes applications.
We expect that our method will inspire further research into enhancing the robustness of existing interpretable architectures or building an interpretable and robust model from scratch.
\section{Computational complexity comparison}
\label{app:complexity}

This section presents a comparison of the computational complexity between \npcname~and \rnpcname~in both the training and inference phases.

\paragraph{Training.} Since \npcname~and \rnpcname~share the same trained attribute recognition model and the same learned probabilistic circuit, their training complexities are identical. The practical training time across various datasets is reported in Table~\ref{tab:time}, with all experiments conducted using eight NVIDIA RTX A6000 GPUs.

\paragraph{Inference.}
Let $|f_k|$ denote the size of the $k$-th neural network in the attribute recognition model, and let $|S|$ denote the size of the probabilistic circuit (\ie, the number of edges).
During inference, given an input sample, a forward pass through all $K$ neural networks in the attribute recognition model incurs a computational cost of $O(\sum_{k=1}^K |f_k|)$.
According to the node-wise integration defined in Equation (\ref{eq:npc}), \npcname~requires performing conditional inference over the probabilistic circuit $\prod_{k=1}^K |\mathcal{A}_k|$ times, resulting in the overall inference complexity:
{
\small
\begin{equation*}
    O\left( \sum_{k=1}^K |f_k| + |S|\cdot \prod_{k=1}^K |\mathcal{A}_k| \right).
\end{equation*}
}

In contrast, according to the class-wise integration defined in Equation (\ref{eq:rnpc_unml}), \rnpcname~only requires performing conditional inference over the probabilistic circuit $|V|$ times, where $V:=\bigcup_{\tilde{y}\in\mathcal{Y}} V_{\tilde{y}}$, resulting in the overall inference complexity:
{
\small
\begin{equation*}
    O\left( \sum_{k=1}^K |f_k| + |S|\cdot |V| \right).
\end{equation*}
}

By construction, $V \subseteq \Omega:=\{a_{1:K}\}$, and therefore, $|V| \leqslant \prod_{k=1}^K |\mathcal{A}_k|$. It follows that $O(\sum_{k=1}^K |f_k| + |S|\cdot |V|) \leqslant O(\sum_{k=1}^K |f_k| + |S|\cdot \prod_{k=1}^K |\mathcal{A}_k|)$. Thus, \rnpcname~is more efficient than \npcname~in terms of the inference complexity.

As expected, the practical inference time of \rnpcname, shown in Table \ref{tab:time}, is consistently faster than \npcname. All inference was performed on a single NVIDIA RTX A6000 GPU.

In summary, \rnpcname~shares the same training complexity as \npcname~but offers better efficiency during inference.

\begin{table}[tb]
\centering
\caption{\small Comparison of training and inference time for \npcname~and \rnpcname~across the MNIST-Add3, MNIST-Add5, and CelebA-Syn datasets.}\label{tab:time}
\begin{tabular}{llccc}\toprule
Dataset &Phase &NPC &RNPC \\\cmidrule{1-4}
\multirow{2}{*}{MNIST-Add3} &Training &\multicolumn{2}{c}{78m} \\\cmidrule{2-4}
&Inference &6.78s &4.78s \\\cmidrule{1-4}
\multirow{2}{*}{MNIST-Add5} &Training &\multicolumn{2}{c}{136m} \\\cmidrule{2-4}
&Inference &16.52s &8.57s \\\cmidrule{1-4}
\multirow{2}{*}{CelebA-Syn} &Training &\multicolumn{2}{c}{381m} \\\cmidrule{2-4}
&Inference &14.19s &8.58s \\
\bottomrule
\end{tabular}
\end{table}
\section{More details on experimental settings} \label{app:exp_setting}
\subsection{Dataset construction} \label{app:exp_setting_dataset}
In Section \ref{sec:exp_setting}, we describe the main properties (\eg, attributes, downstream tasks) of various datasets.
Here, we provide additional details about the construction process of each dataset.

\paragraph{MNIST-Add3 dataset.}
In this dataset, each image concatenates three digit images from the MNIST dataset~\citep{mnist} under the CC BY-SA 3.0 license. These digit images are applied with different transformations to introduce the domain shifts between them, which include rotations and color modifications.
The three digits serve as the attributes for this dataset.
The downstream task is to predict the sum of these attributes.
To construct the dataset, we first identify an attribute set within the three-dimensional attribute space, ensuring a minimum inter-class distance of 3. Correspondingly, the radius of this attribute set is 1.
The randomly selected attribute set $V$ is
\{[6 3 7],
 [9 6 8],
 [0 2 4],
 [3 0 5],
 [5 5 1],
 [7 4 3],
 [2 7 6],
 [4 1 2],
 [1 9 0],
 [8 8 9]\}.
Each attribute node in $V$ results in a unique attribute sum, leading to a total of 10 classes for the downstream task.
Next, we generate images for each class by concatenating images of specific digits corresponding to the attributes.
For example, to generate an image belonging to class 16, we concatenate a digit-6 image, a digit-3 image, and a digit-7 image in sequence.
To introduce variability, we incorporate 1\% labeling noise into the dataset. Specifically, this involves randomizing either the class labels or the attribute labels.
In total, we generate 63,130 images and split them into training, validation, and testing sets by a ratio of 8:1:1.
Example testing images are illustrated in Figure \ref{fig:mnist-add3-add5} (a).


\paragraph{MNIST-Add5 dataset.}
In this dataset, each image concatenates five digit images from the MNIST dataset~\citep{mnist} under the CC BY-SA 3.0 license. These digit images are applied with different transformations to introduce the domain shifts between them, which include rotations, color modifications, and blurring.
The five digits serve as the attributes for this dataset.
The downstream task is to predict the sum of these attributes.
To construct the dataset, we first identify an attribute set within the five-dimensional attribute space, ensuring a minimum inter-class distance of 5. Correspondingly, the radius of this attribute set is 2.
The randomly selected attribute set $V$ is
\{[6 3 7 4 6],
 [0 9 5 3 1],
 [5 0 9 2 3],
 [2 7 8 5 4],
 [8 2 4 9 8],
 [7 5 0 6 0],
 [3 4 6 1 2],
 [4 1 3 0 5],
 [1 6 2 8 9],
 [9 8 1 7 7]\}.
These attribute nodes result in 7 different attribute sums, leading to a total of 7 classes.
Note that a class (\eg, class 26) may correspond to multiple attribute nodes in $V$.
Next, we generate images for each class by concatenating images of specific digits corresponding to the attributes.
For example, to generate an image belonging to class 32, we concatenate a digit-9 image, a digit-8 image, a digit-1 image, and two digit-7 images in sequence.
To introduce variability, we incorporate 1\% labeling noise into the dataset. Specifically, this involves randomizing either the class labels or the attribute labels.
In total, we generate 63,130 images and split them into training, validation, and testing sets by a ratio of 8:1:1.
Example testing images are illustrated in Figure \ref{fig:mnist-add3-add5} (b).

\begin{figure}[htbp]
    \centering
    \includegraphics[width=0.65\linewidth]{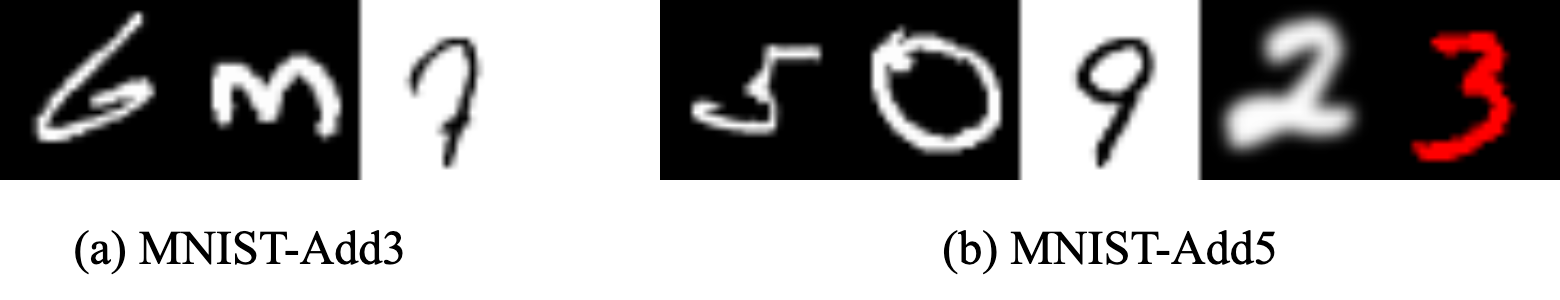}
    \caption{
    \small
    Examples of testing images from the MNIST-Add3 and MNIST-Add5 datasets.
    \textbf{(a)} A testing image from the MNIST-Add3 dataset, corresponding to the attribute node [6, 3, 7].
    \textbf{(b)} A testing image from the MNIST-Add5 dataset, corresponding to the attribute node [5, 0, 9, 2, 3].
    }
    \label{fig:mnist-add3-add5}
\end{figure}

\paragraph{CelebA-Syn dataset.}
This dataset is constructed based on the CelebA dataset~\citep{celeba} under its non-commercial research license, which includes annotations for forty facial attributes. In particular, we select eight of them that are visually easy to distinguish, which are Color\_Hair, Double\_Chin, Eyeglasses, Heavy\_Makeup, Mustache, Pale\_Skin, Smiling, Young.
Following a similar construction process as the above datasets, we first identify an attribute set within the eight-dimensional attribute space, ensuring a minimum inter-class distance of 4. Correspondingly, the radius of this attribute set is 1.
The randomly selected attribute set $V$ is
\{[2 1 0 0 0 1 0 0],
 [1 0 0 0 0 1 1 1],
 [1 1 1 1 0 1 1 0],
 [1 1 0 0 1 0 0 1],
 [2 0 1 1 1 1 0 1],
 [0 1 1 0 1 1 1 1],
 [0 0 0 1 0 0 0 1],
 [3 0 1 0 1 0 1 0],
 [0 0 0 1 1 1 1 0],
 [0 1 1 1 1 0 0 0]\}\footnote{For the attribute Color\_Hair, values 0-3 correspond to Black\_Hair, Blond\_Hair, Brown\_Hair, and Gray\_Hair, respectively. For other attributes, \eg, Double\_Chin, 0 indicates the absence of the feature, while 1 indicates its presence.}.
Following~\citet{cem}, each attribute node corresponds to a unique class, resulting in 10 classes in total.
Next, we train a StarGAN~\citep{stargan} model and use it to synthesize images for each class.
StarGAN is a powerful tool for transferring the attributes of input images to designated values. 
Specifically, taking a face image and a set of attribute values as input, the trained StarGAN generates a face image with attributes transferred to specified values.
In total, we generate 50,000 training images, 10,000 validation images, and 9,990 testing images.

\paragraph{GTSRB-Sub dataset.}
This dataset is derived from the GTSRB dataset~\citep{gtsrb} in accordance with its research-purpose terms, which contains images of German traffic signs along with class labels indicating the sign types.
\citet{npc} annotate each image in GTSRB with four attributes: color, shape, symbol, and text. 
Each class corresponds to a distinct instantiation of attributes, \ie, a unique attribute node.
The attribute set corresponding to GTSRB has a minimum inter-class distance of 1, resulting in a radius of 0.
To increase the minimum inter-class distance, we construct a subset of GTSRB by selecting and grouping specific classes. 
The final classes in this subset are:
\textbf{1) Direction:} This class consists of the original classes—`regulatory--maximum-speed-limit-20', `regulatory--maximum-speed-limit-30', `regulatory--maximum-speed-limit-50', `regulatory--maximum-speed-limit-60', `regulatory--maximum-speed-limit-70', `regulatory--maximum-speed-limit-80', `regulatory--maximum-speed-limit-100', `regulatory--maximum-speed-limit-120'.
\textbf{2) Priority:} This class consists of the original class—`regulatory--priority-road'.
\textbf{3) Speed:} This class consists of the original classes—`warning--other-danger', `warning--double-curve-first-left', `warning--uneven-road', `warning--slippery-road-surface', `warning--road-narrows-right', `warning--roadworks', `warning--traffic-signals', `warning--pedestrians-crossing', `warning--children', `warning--bicycles-crossing', `warning--ice-or-snow', `warning--wild-animals'.
\textbf{4) Warning:} This class consists of the original classes—`regulatory--turn-right-ahead', `regulatory--turn-left-ahead', `regulatory--go-straight', `regulatory--go-straight-or-turn-right', `regulatory--go-straight-or-turn-left', `regulatory--keep-right', `regulatory--keep-left', `regulatory--roundabout'.
Note that each class may correspond to multiple attribute nodes.
This subset achieves a minimum inter-class distance of 3, with a corresponding radius of 1.
In total, GTSRB-Sub contains 22,079 training images, 2,759 validation images, and 2,761 testing images.

\subsection{Implementation details} \label{app:exp_setting_implementation}
\paragraph{\npcname~\citep{npc}.}
To strive for simplicity in experiments, we implement the attribute recognition model of \npcname~using a set of independent two-layer MLPs. 
Specifically, each MLP is used to identify one particular attribute.
The attribute recognition model is trained using the sum of cross-entropy losses over all attributes.
The training process is conducted with a batch size of 256 for 100 epochs, using the SGD optimizer.
The probabilistic circuit of \npcname~is learned with the LearnSPN algorithm~\citep{learnspn}, with its parameters optimized via the CCCP algorithm~\citep{cccp}. 
Given the trained attribute recognition model and the learned probabilistic circuit, \npcname~integrates their outputs through node-wise integration to produce predictions for downstream tasks.

\paragraph{\rnpcname.}
\rnpcname~uses the same trained attribute recognition model and learned probabilistic circuit as \npcname.
But different from \npcname, \rnpcname~adopts class-wise integration to produce predictions for downstream tasks.

\paragraph{CBM~\citep{cbm}.}
To ensure a fair comparison among the baselines, we implement the recognition module of CBM using a two-layer MLP.
Following~\citet{cbm}, the predictor module of CBM is implemented with a linear layer.
CBM is trained using the weighted sum of the cross-entropy loss over concepts and the cross-entropy loss over classes.
The training process is conducted with a batch size of 256 for 100 epochs, using the SGD optimizer.

\paragraph{DCR~\citep{dcr}.}
The recognition module of DCR is implemented with two layers: a linear layer followed by ReLU activation and an embedding layer defined in~\citet{cem}.
The predictor module is implemented using the deep concept reasoner proposed in~\citet{dcr}.
DCR is trained using the weighted sum of the cross-entropy loss over concepts and the cross-entropy loss over classes.
The training process is conducted with a batch size of 256 for 100 epochs, using the SGD optimizer.

\subsection{Attack configurations} \label{app:exp_setting_attack}
The adversarial attacks used in this paper are implemented using the adversarial-attacks-pytorch library\footnote{The library is available at \href{https://github.com/Harry24k/adversarial-attacks-pytorch}{https://github.com/Harry24k/adversarial-attacks-pytorch}.}~\citep{kim2020torchattacks}.

\paragraph{$\infty$-norm-bounded PGD attack~\citep{pgd_attack}.}
For the MNIST-Add3, MNIST-Add5, and GTSRB-Sub datasets, the $\infty$-norm bounds are set to 0.03, 0.05, 0.07, 0.09, and 0.11.
For the CelebA-Syn dataset, where the model demonstrates greater vulnerability to adversarial attacks, the $\infty$-norm bounds are set to 0.01, 0.02, 0.03, 0.04, and 0.05.
Across all datasets, we use a step size of 2/255 and perform 50 steps for the attack.

\paragraph{2-norm-bounded PGD attack~\citep{pgd_attack}.}
For the MNIST-Add3 and MNIST-Add5 datasets, the 2-norm bounds are set to 3, 5, 7, 9, and 11.
For the CelebA-Syn dataset, where the model demonstrates greater vulnerability to adversarial attacks, the 2-norm bounds are set to 1, 2, 3, 4, and 5.
Across all datasets, we use a step size of 0.1*norm bound and perform 50 steps for the attack.

\paragraph{2-norm-bounded CW attack~\citep{cw_attack}.}
We employ the 2-norm-bounded CW attack with binary search. Across all datasets, we use a step size of 0.01 and perform 10 steps for the attack. The strength of the attack is varied by adjusting the number of binary search steps. Specifically, for the MNIST-Add3 and MNIST-Add5 datasets, the binary search steps are set to 3, 5, 7, 9, and 11. 
For the CelebA-Syn dataset, the steps are set to 1, 2, 3, 4, and 5. After the attack, we measure the 2-norm between the benign inputs and the perturbed inputs to quantify the magnitude of the perturbations.
\section{More experimental results}

\subsection{Performance for the attacked attribute} \label{app:attr_pred}
\begin{figure}[tb]
    \centering
    \includegraphics[width=0.60\linewidth]{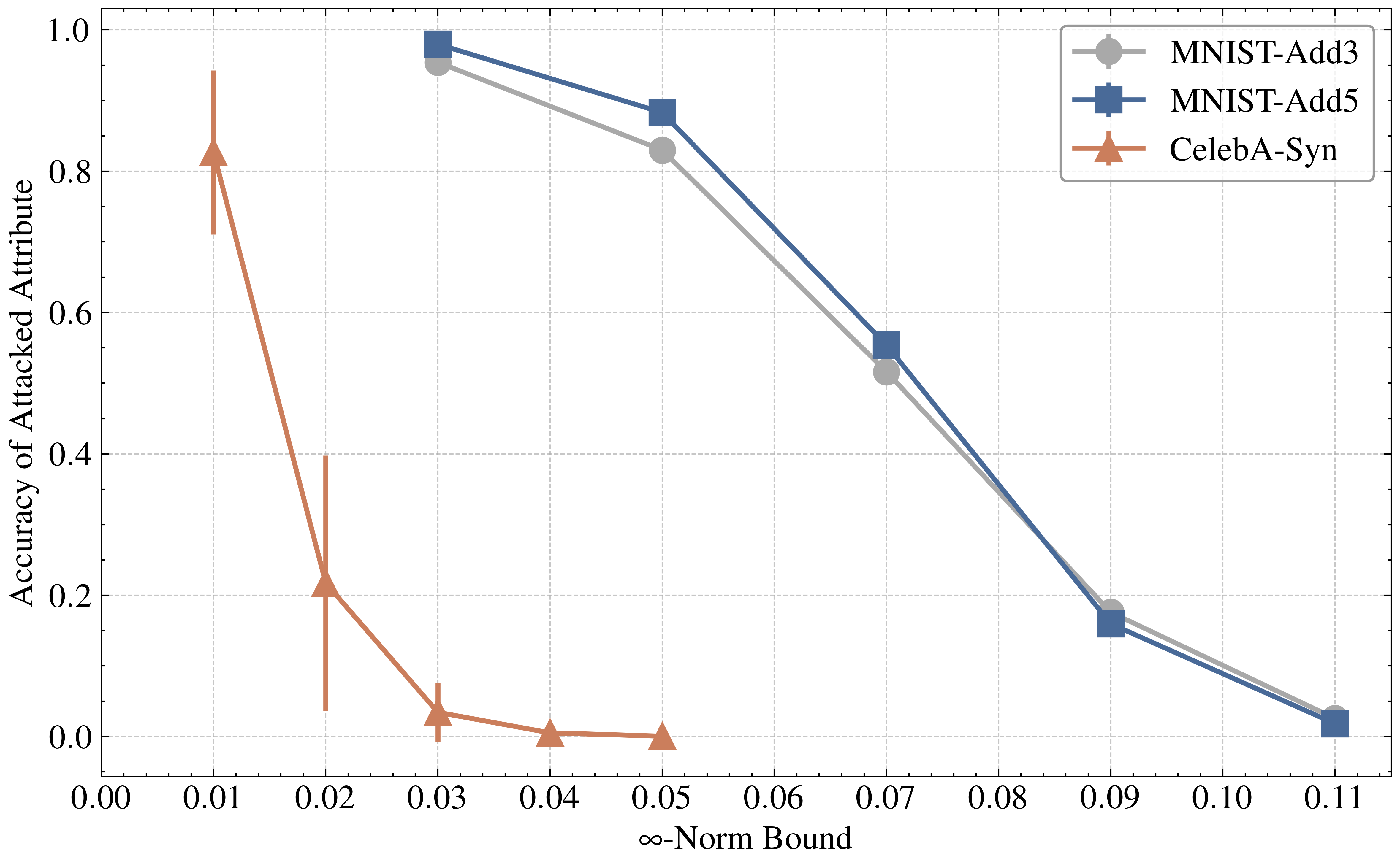}
    \caption{\small 
    Accuracy of the attribute recognition model in predicting the attacked attribute under the $\infty$-norm-bounded PGD attack with varying norm bounds on the MNIST-Add3, MNIST-Add5, and CelebA-Syn datasets.
    The attacker attacks \textit{a single attribute} at a time, generating an adversarial perturbation to distort the prediction result for that attribute. 
    The curves show the mean accuracy of the attacked attribute across varying norm bounds, with error bars indicating the standard deviation computed over all attacked attributes.}
    \label{fig:attr_pred}
\end{figure}

Consider a setting where the attacker attacks a single attribute at a time. Figure \ref{fig:attr_pred} illustrates the accuracy of the attribute recognition model in predicting the attacked attribute under the $\infty$-norm-bounded PGD attack with varying norm bounds on the MNIST-Add3, MNIST-Add5, and CelebA-Syn datasets. The curves show the mean accuracy of the attacked attribute across varying norm bounds, with error bars indicating the standard deviation computed over all attacked attributes.

As the norm bound increases, the accuracy for the attacked attribute consistently decreases, demonstrating that stronger adversarial perturbations more severely degrade the model's predictions. Notably, on all three datasets, the accuracy drops to nearly 0\% at large norm bounds (\eg, a bound of 0.11 for MNIST-Add3 and MNIST-Add5). This strong adversarial effect provides a compelling testbed for evaluating the robustness of \rnpcname.

\subsection{Performance against more adversarial attacks} \label{app:more_attack}
\paragraph{Performance against the 2-norm-bounded PGD attack.}
Figure \ref{fig:pgdl2_attack} illustrates the adversarial accuracy of \rnpcname~and the baseline models under the 2-norm-bounded PGD attack with varying norm bounds.
According to these results, we can conduct a similar analysis and reach similar conclusions to those in Section \ref{sec:exp_main}.
Specifically, we observe that on the MNIST-Add3, MNIST-Add5, and CelebA-Syn datasets, the adversarial accuracy of \npcname~and \rnpcname~is consistently higher than that of CBM and DCR under attacks with any 2-norm bound. This finding indicates that including the probabilistic circuit within a model's architecture potentially strengthens a model's robustness. In contrast, the task predictors used in CBM and DCR might compromise the model's robustness.

Furthermore, on the MNIST-Add3 and MNIST-Add5 datasets, \rnpcname~outperforms \npcname~by a large margin, especially under attacks with larger norm bounds. For instance, on MNIST-Add5, when the 2-norm bound reaches 11, \npcname's adversarial accuracy drops below 60\% whereas \rnpcname~maintains adversarial accuracy above 80\%.
These results demonstrate that \rnpcname~provides superior robustness compared to \npcname~on these datasets, highlighting the effectiveness of the proposed class-wise integration approach.
On the CelebA-Syn dataset, \rnpcname~performs similarly to \npcname, with both showing high robustness even under attacks with large norm bounds.

\begin{figure*}[tb]
    \centering
    \begin{minipage}{0.32\textwidth}
        \centering
        \includegraphics[width=\textwidth]{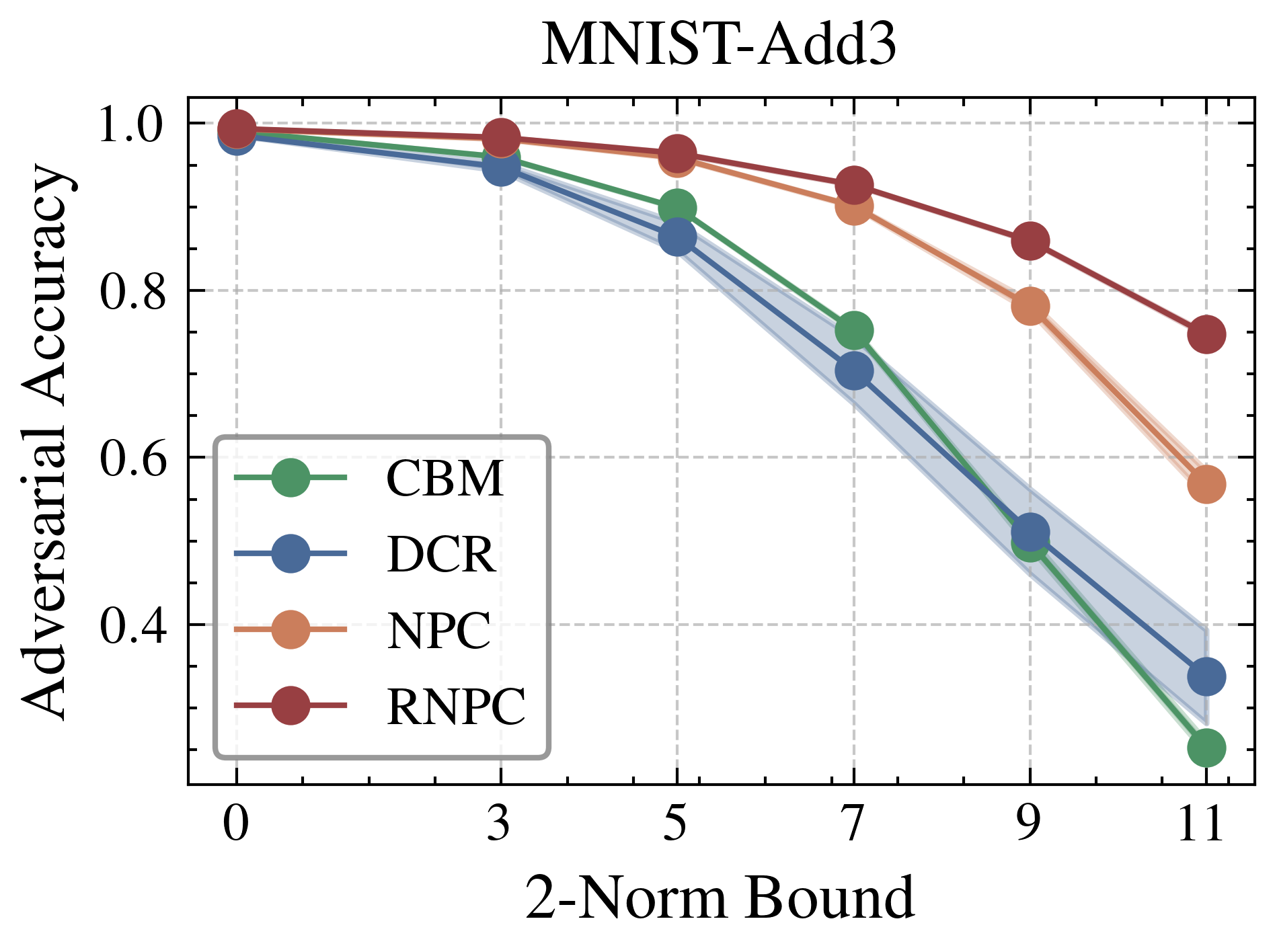}
    \end{minipage}
    \begin{minipage}{0.32\textwidth}
        \centering
        \includegraphics[width=\textwidth]{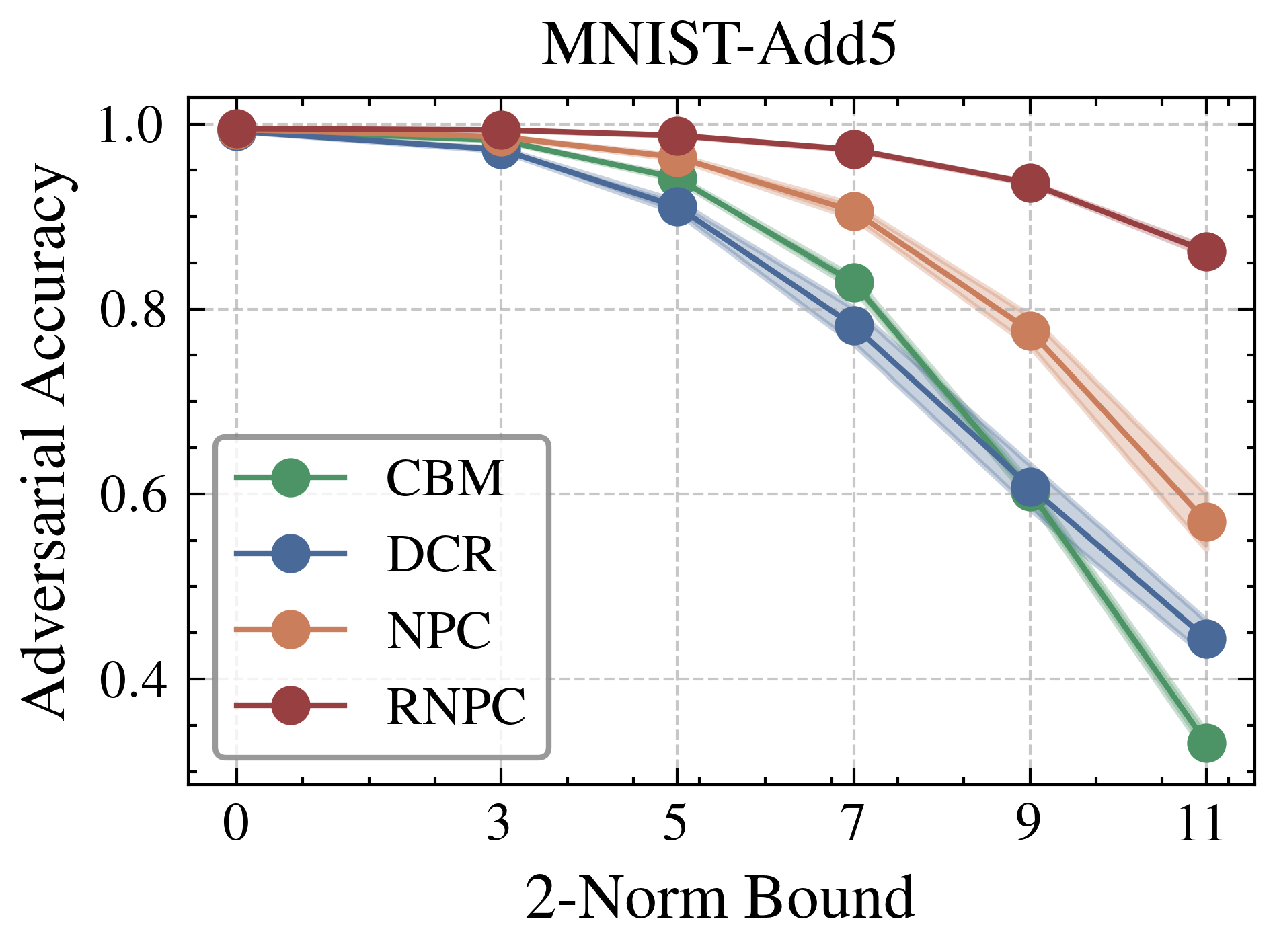}
    \end{minipage}
    \begin{minipage}{0.32\textwidth}
        \centering
        \includegraphics[width=\textwidth]{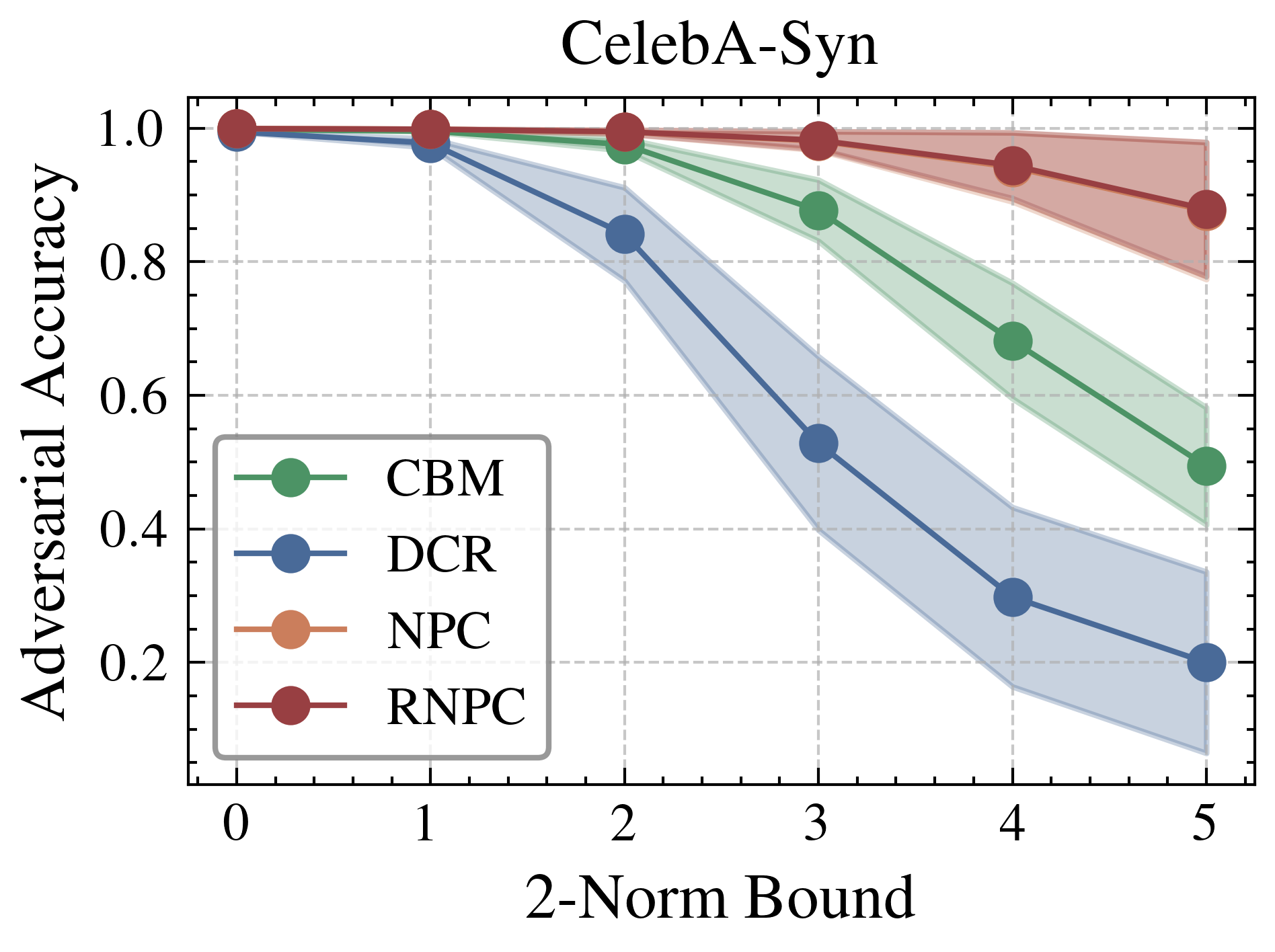}
    \end{minipage}
    \caption{
    \small
    Adversarial accuracy of CBM, DCR, \npcname, and \rnpcname~under the 2-norm-bounded PGD attack with varying norm bounds on the MNIST-Add3, MNIST-Add5, and CelebA-Syn datasets.}
    \label{fig:pgdl2_attack}
\end{figure*}

\paragraph{Performance against the 2-norm-bounded CW attack.}
Figure \ref{fig:cw_attack} illustrates the adversarial accuracy of \rnpcname~and the baseline models under the 2-norm-bounded CW attack with varying norm bounds.
We observe that \npcname~and \rnpcname~perform similarly and robustly on the MNIST-Add3 and CelebA-Syn datasets, both reaching adversarial accuracy close to 100\% under attacks with any 2-norm bound.
In contrast, the adversarial accuracy of CBM and DCR decreases as the norm bound increases.
This comparison demonstrates that \npcname~and \rnpcname~are robust against the 2-norm-bounded CW attack, indicating the robustness enhancement enabled by the probabilistic circuit.

On the MNIST-Add5 dataset, however, \npcname's adversarial accuracy also declines as the norm bound increases, while \rnpcname~maintains adversarial accuracy close to 100\%. These results demonstrate that \rnpcname~is more robust than \npcname, highlighting the robustness improvement achieved by the class-wise integration approach.

\begin{figure*}[tb]
    \centering
    \begin{minipage}{0.32\textwidth}
        \centering
        \includegraphics[width=\textwidth]{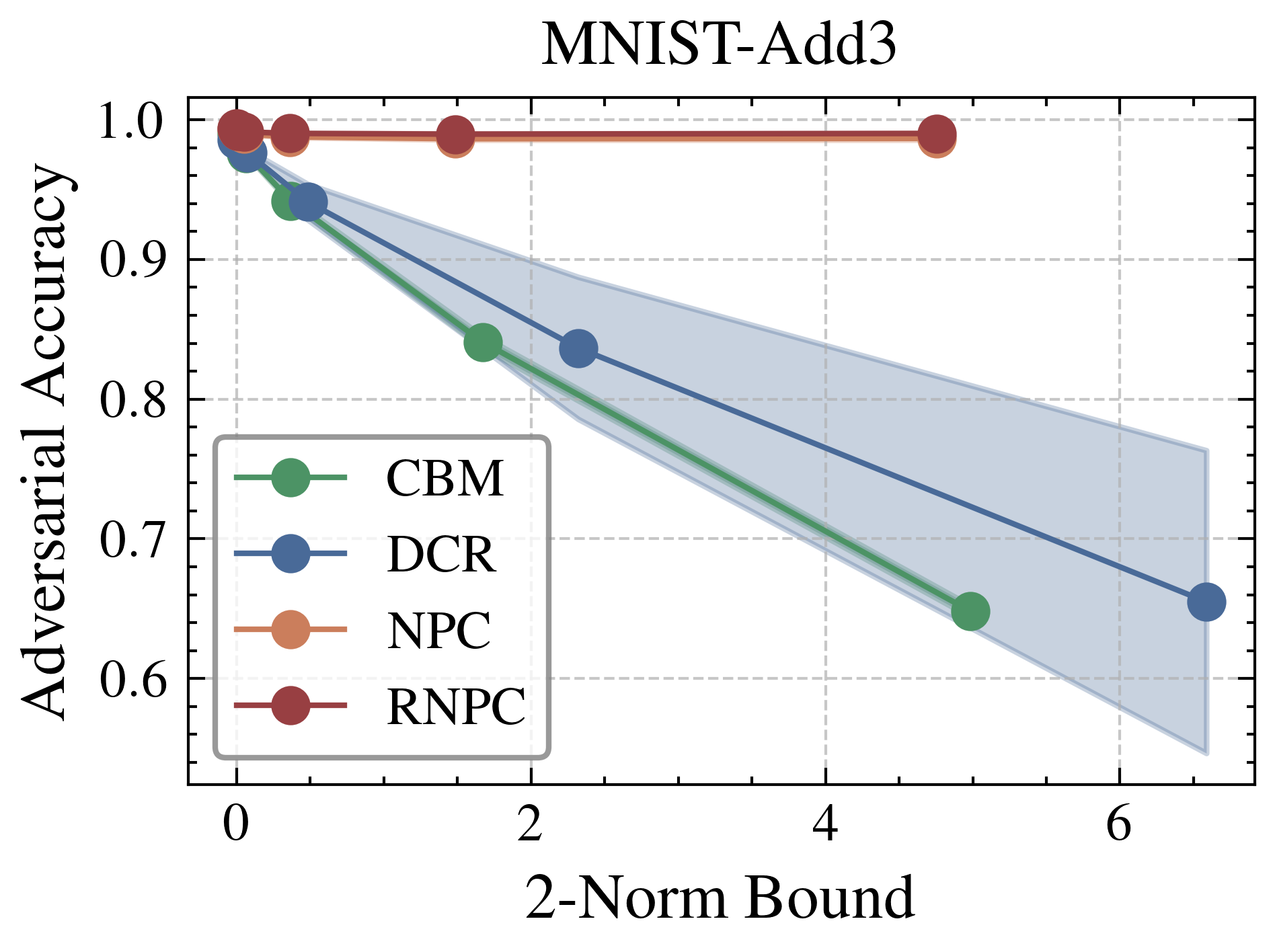}
    \end{minipage}
    \begin{minipage}{0.32\textwidth}
        \centering
        \includegraphics[width=\textwidth]{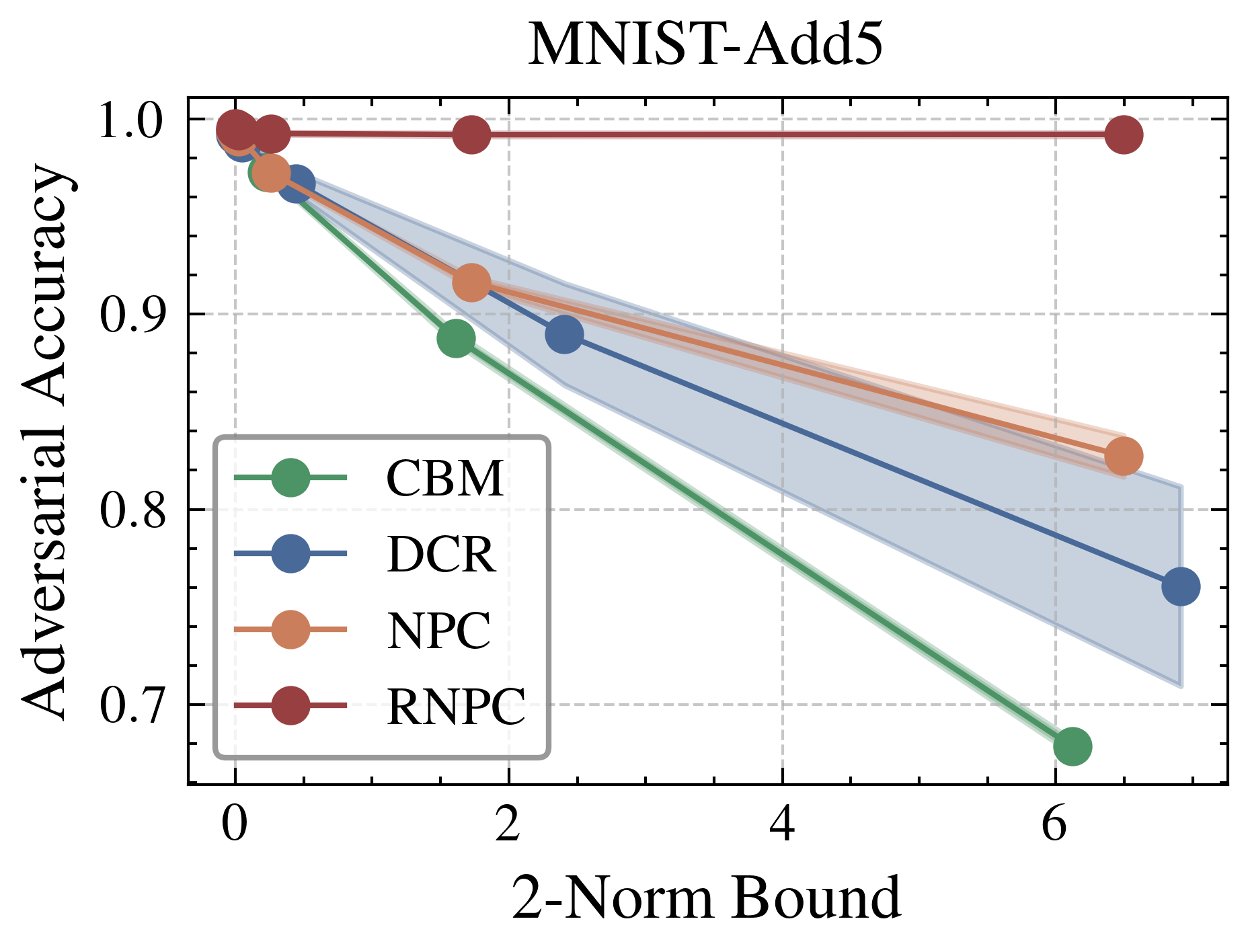}
    \end{minipage}
    \begin{minipage}{0.32\textwidth}
        \centering
        \includegraphics[width=\textwidth]{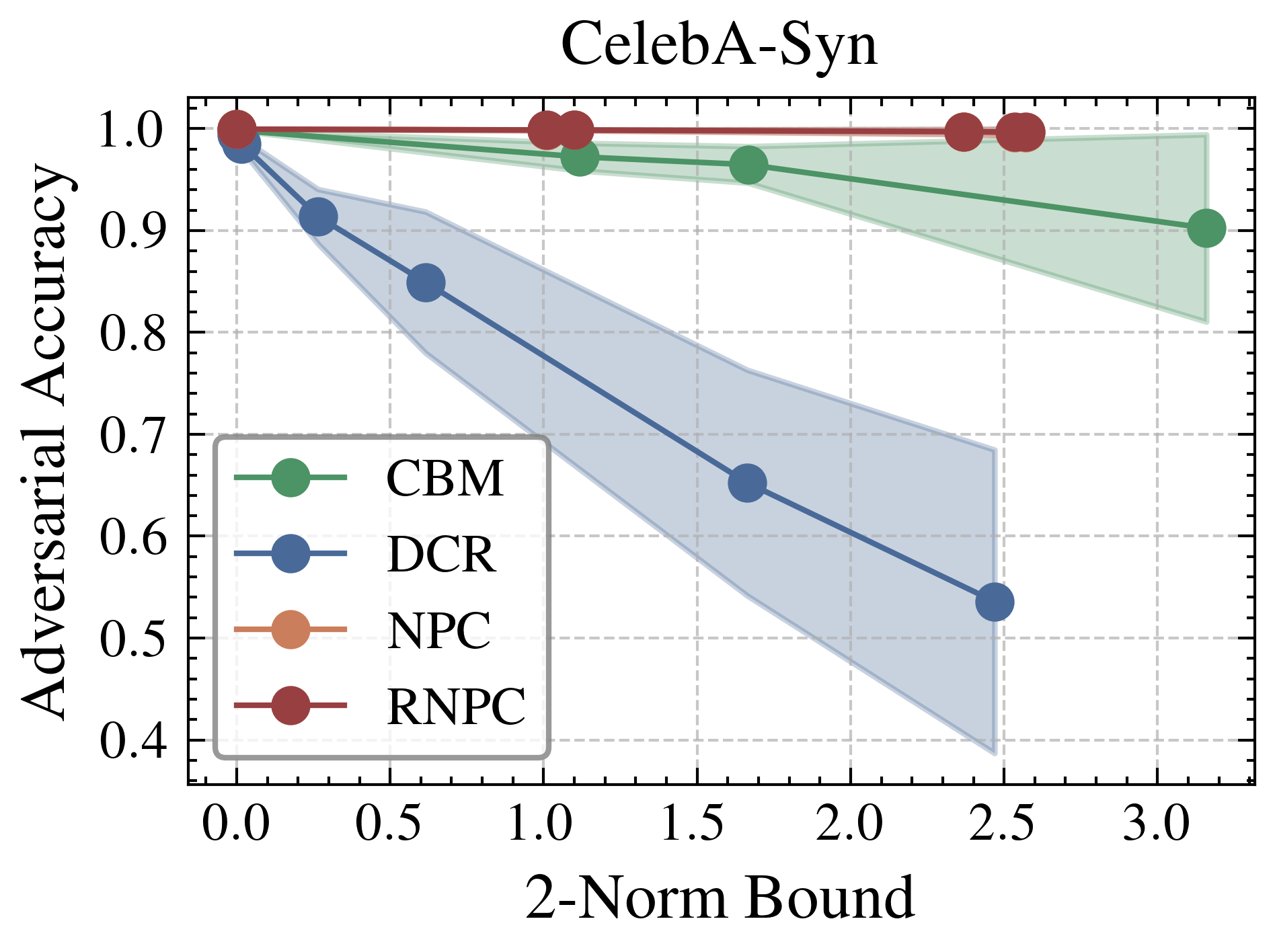}
    \end{minipage}
    \caption{
    \small
    Adversarial accuracy of CBM, DCR, \npcname, and \rnpcname~under the 2-norm-bounded CW attack with varying norm bounds on the MNIST-Add3, MNIST-Add5, and CelebA-Syn datasets.}
    \label{fig:cw_attack}
\end{figure*}

\subsection{More ablation studies} \label{app:ablation}
\paragraph{Impact of Differential Privacy (DP).}
Theorem \ref{thm:comparison} indicates that the robustness of \rnpcname~against a $p$-norm-bounded adversarial attack with a budget of $\ell$ is higher than that of \npcname~when the attribute recognition model satisfies $\epsilon$-DP with respect to the $p$-norm.
Here, we empirically validate whether this implication holds in practice.

To evaluate the robustness of \npcname~and \rnpcname, we conduct the 2-norm-bounded PGD attack (targeting a single attribute) with norm bounds of 3, 3, and 1 for the MNIST-Add3, MNIST-Add5, and CelebA-Syn datasets, respectively.
Following \citet{dp_robustness}, DP within the attribute recognition model can be implemented by injecting noise after various layers. For simplicity, we directly add noise to the input images. 
Specifically, the noise is sampled from a Gaussian distribution with zero mean and standard deviation $\sigma=\sqrt{2 \ln \left(\frac{1.25}{\delta}\right) \ell / \epsilon}$, where $\ell$ corresponds to the attack norm bound, $\epsilon$ is set to 0.5, and $\delta$ is chose as a small value (\eg, 0.01). 
This ensures that the attribute recognition model approximately satisfies $(\epsilon,0)$-DP.
Additionally, we estimate $\mathbb{E}[f_{\theta_k}(X)_{a_k}]$ by computing the mean of $f_{\theta_k}(X)_{a_k}$ over 20 noise draws.
The performance of \npcname~and \rnpcname~under these conditions is illustrated in Figure \ref{fig:dp}.

Compared to models without DP (\ie, vanilla models), the models satisfying DP generally exhibit lower adversarial accuracy across the three datasets due to the noise added to the inputs.
Despite this, we observe that \rnpcname~consistently outperforms \npcname~on the three datasets.
Notably, on the MNIST-Add3 and MNIST-Add5 datasets, the adversarial accuracy of \rnpcname~is higher than that of \npcname~by a large margin.
These results demonstrate that while integrating DP into the attribute recognition model might compromise the overall performance on downstream tasks, it highlights the robustness enhancement achieved by \rnpcname, thereby validating the implication of Theorem \ref{thm:comparison}.

\begin{figure*}[tb]
    \centering
    \begin{minipage}{0.32\textwidth}
        \centering
        \includegraphics[width=\textwidth]{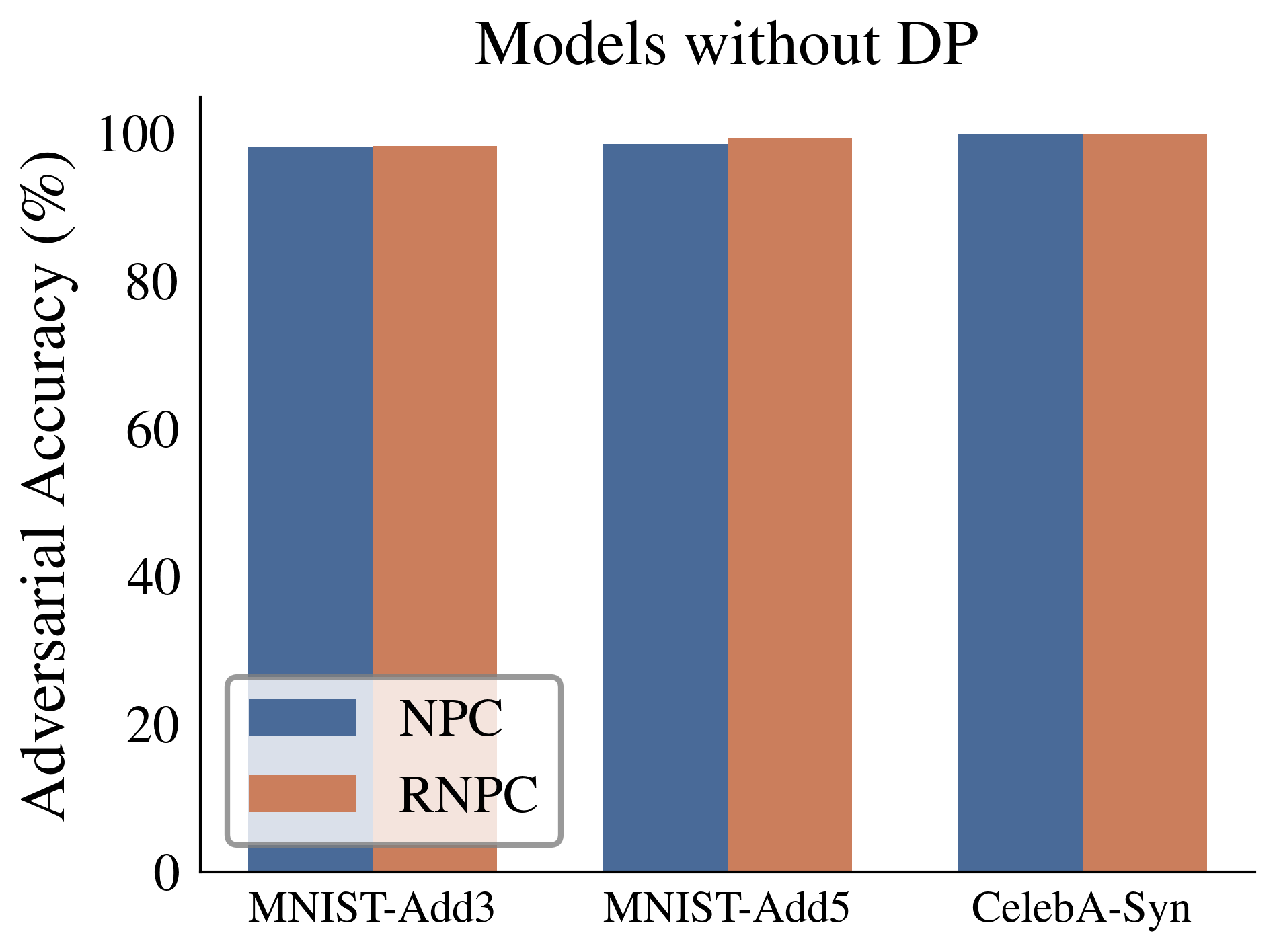}
    \end{minipage}
    \begin{minipage}{0.32\textwidth}
        \centering
        \includegraphics[width=\textwidth]{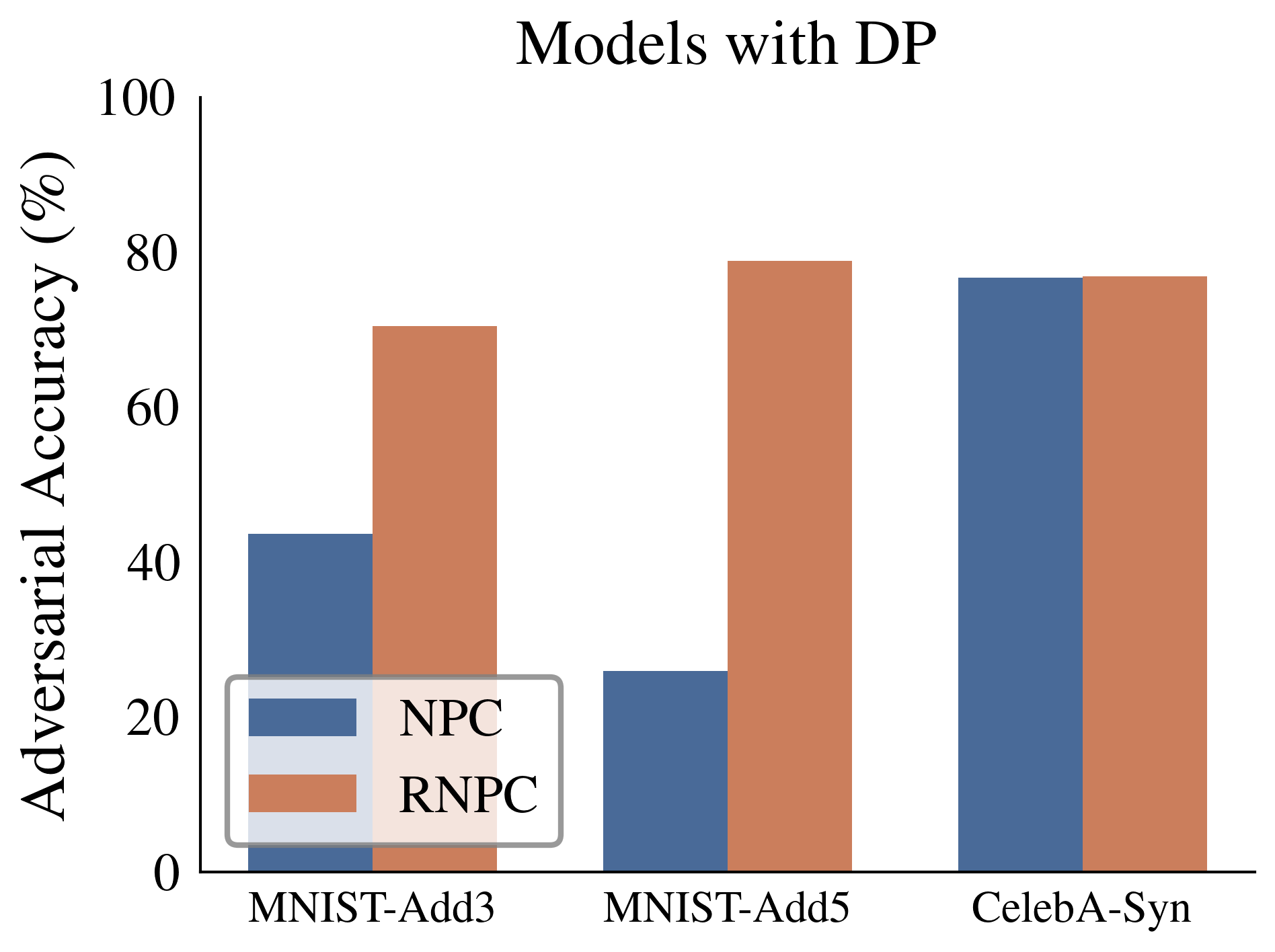}
    \end{minipage}
    \caption{
    \small
    Adversarial accuracy of \npcname~and \rnpcname~under the 2-norm-bounded PGD attack (targeting a single attribute) with norm bounds of 3, 3, and 1 for the MNIST-Add3, MNIST-Add5, and CelebA-Syn datasets, respectively.
    The left figure illustrates the performance of models without DP, \ie, vanilla models, while the right figure illustrates the performance of models with their attribute recognition models satisfying DP.
    }
    \label{fig:dp}
\end{figure*}
\section{Theoretical results with omitted proofs} \label{app:proof}

In this section, we provide more theoretical results and elaborate the proofs omitted in the main paper.

\subsection{Adversarial robustness of \npcname s}
\begin{theorem}[Adversarial robustness of \npcname s \textbf{(Restatement of Theorem \ref{thm:npc_perturbation})}]
Under Assumption \ref{assump:complete}, the prediction perturbation of \npcname~is bounded by the worst-case TV distance between the overall attribute distributions conditioned on the vanilla and perturbed inputs, which is further bounded by the sum of the worst-case TV distances for each attribute, \ie,
\resizebox{\textwidth}{!}{%
    \begin{minipage}{\textwidth}
    \begin{align*}
        \Delta_{\theta, w}^{\npcname}
        &\leqslant
        \mathbb{E}_X\left[\max_{\tilde{X}\in \mathbb{B}_p(X, \ell)}~d_{\mathrm{TV}}\left(\mathbb{P}_{\theta}\left(A_{1:K} \mid X\right), \mathbb{P}_{\theta}\left(A_{1:K} \mid \tilde{X}\right)\right)\right] 
        \leqslant \sum_{k=1}^{K} \mathbb{E}_X\left[\max_{\tilde{X}\in \mathbb{B}_p(X, \ell)}~d_{\mathrm{TV}}\left(\mathbb{P}_{\theta_k}\left(A_k \mid X\right), \mathbb{P}_{\theta_k}\left(A_k \mid \tilde{X}\right)\right)\right].
    \end{align*}
    \end{minipage}
}
\end{theorem}

\begin{proof}
Under Assumption \ref{assump:complete}, $ \mathbb{P}_{\theta}\left(A_{1:K} \mid X\right) = \prod_{k=1}^K\mathbb{P}_{\theta_k}(A_k\mid X)$. Therefore,
{
\scriptsize
\begin{align*}
&d_{\mathrm{TV}}\left(\mathbb{P}_{\theta, w}(Y \mid X), \mathbb{P}_{\theta, w}(Y \mid \tilde{X})\right) \\
&= \frac{1}{2} \sum_y \left|\mathbb{P}_{\theta, w}(Y=y \mid X)-\mathbb{P}_{\theta, w}(Y=y \mid \tilde{X})\right| \\
&\leqslant \frac{1}{2} \sum_y \sum_{a_{1:K}} \mathbb{P}_w(Y=y\mid A_{1:K}=a_{1:K}) \cdot \left| \prod_{k=1}^K\mathbb{P}_{\theta_k}(A_k=a_k\mid X) - \prod_{k=1}^K\mathbb{P}_{\theta_k}(A_k=a_k\mid \tilde{X})\right| \\
&= \frac{1}{2} \sum_{a_{1:K}} \left| \prod_{k=1}^K\mathbb{P}_{\theta_k}(A_k=a_k\mid X) - \prod_{k=1}^K\mathbb{P}_{\theta_k}(A_k=a_k\mid \tilde{X})\right| = d_{\mathrm{TV}}\left(\mathbb{P}_{\theta}\left(A_{1:K} \mid X\right), \mathbb{P}_{\theta}\left(A_{1:K} \mid \tilde{X}\right)\right) \\
&\leqslant \frac{1}{2} \sum_{k=1}^K \sum_{a_k} \ \left| \mathbb{P}_{\theta_k}(A_k=a_k\mid X)-\mathbb{P}_{\theta_k}(A_k=a_k\mid \tilde{X}) \right| = \sum_{k=1}^K d_{\mathrm{TV}}\left(\mathbb{P}_{\theta_k}(A_k\mid X), \mathbb{P}_{\theta_k}(A_k\mid \tilde{X})\right).
\end{align*}
}
By successively applying the max and expectation operators to both sides, we complete the proof. 
\end{proof}

\subsection{Adversarial robustness of \rnpcname s}
\begin{lemma}[Adversarial robustness of \rnpcname s \textbf{(Restatement of Lemma~\ref{thm:rnpc_perturbation})}]
    The prediction perturbation of \rnpcname~is bounded by the worst-case change in probabilities within a neighborhood caused by the attack, \ie,
    {
    \small
    \begin{align*}
    \Delta_{\theta, w}^{\rnpcname}
    \leqslant 
    \mathbb{E}_X\left[\max_{\tilde{X}\in \mathbb{B}_p(X, \ell)}~ \left\{ \max_{\tilde{y}\in\mathcal{Y}} \left| 1-\frac{\mathbb{P}_{\theta}(A_{1:K}\in \mathcal{N}(\tilde{y}, r)\mid \tilde{X})}{\mathbb{P}_{\theta}(A_{1:K}\in \mathcal{N}(\tilde{y}, r)\mid X)} \right| \right\} \right] .
    \end{align*}
    }
\end{lemma}

\begin{proof}
    $\Delta_{\theta, w}^{\rnpcname}$ can be bounded as follows,

    \resizebox{\textwidth}{!}{%
    \begin{minipage}{\textwidth}
    \begin{align}
    &d_{\mathrm{TV}}\left(\hat{\Phi}_{\theta, w}(Y \mid X), \hat{\Phi}_{\theta, w}(Y \mid \tilde{X})\right) \notag\\
    &= \frac{1}{2}\frac{1}{Z_{\theta}(X)\cdot Z_{\theta}(\tilde{X})} \sum_Y \left|Z_{\theta}(\tilde{X})\cdot \Phi_{\theta, w}(Y\mid X) - Z_{\theta}(X)\cdot \Phi_{\theta, w}(Y\mid \tilde{X}) \right| \notag\\
    &= \frac{1}{2}\frac{1}{Z_{\theta}(X)\cdot Z_{\theta}(\tilde{X})} \sum_Y \left|
    Z_{\theta}(\tilde{X})\cdot \Phi_{\theta, w}(Y\mid X) 
    - Z_{\theta}(\tilde{X})\cdot \Phi_{\theta, w}(Y\mid \tilde{X}) 
    + Z_{\theta}(\tilde{X})\cdot \Phi_{\theta, w}(Y\mid \tilde{X}) 
    - Z_{\theta}(X)\cdot \Phi_{\theta, w}(Y\mid \tilde{X}) \right| \notag\\
    &\leqslant \frac{1}{2}\frac{1}{Z_{\theta}(X)\cdot Z_{\theta}(\tilde{X})} \sum_Y \left[Z_{\theta}(\tilde{X})\cdot\left|\Phi_{\theta, w}(Y\mid X) - \Phi_{\theta, w}(Y\mid \tilde{X})\right| 
    + \Phi_{\theta, w}(Y\mid \tilde{X})\cdot\left|Z_{\theta}(\tilde{X}) - Z_{\theta}(X)\right|\right] \notag\\
    &= \frac{1}{2}\frac{1}{Z_\theta(X)} \sum_Y \left|\Phi_{\theta, w}(Y\mid X) - \Phi_{\theta, w}(Y\mid \tilde{X})\right| + \frac{1}{2}\frac{1}{Z_{\theta}(X)\cdot Z_{\theta}(\tilde{X})} \cdot  Z_\theta(\tilde{X}) \cdot \left|Z_{\theta}(\tilde{X}) - Z_{\theta}(X)\right| \notag\\
    &= \frac{1}{2}\frac{1}{Z_\theta(X)} \left[ \sum_Y \left|\Phi_{\theta, w}(Y\mid X) - \Phi_{\theta, w}(Y\mid \tilde{X})\right| + \left|Z_{\theta}(\tilde{X}) - Z_{\theta}(X)\right| \right], \label{eq:rnpc_tv}
    \end{align}
    \end{minipage}
    }
    where the penultimate equation is derived using $\sum_Y \Phi_{\theta, w}(Y\mid \tilde{X}) = Z_\theta(\tilde{X})$.

    For the first interior term in Equation (\ref{eq:rnpc_tv}),
    
    \resizebox{\textwidth}{!}{%
    \begin{minipage}{\textwidth}
    \begin{align*}
    \left|\Phi_{\theta, w}(Y\mid X) - \Phi_{\theta, w}(Y\mid \tilde{X})\right| &\leqslant \sum_{\tilde{y}} \left| \mathbb{P}_{\theta}(A_{1:K}\in \mathcal{N}(\tilde{y}, r)\mid X) - \mathbb{P}_{\theta}(A_{1:K}\in \mathcal{N}(\tilde{y}, r)\mid \tilde{X}) \right| \cdot \sum_{a_{1:K}\in V_{\tilde{y}}} \mathbb{P}_w(Y\mid A_{1:K}=a_{1:K}), \\
    \sum_Y \left|\Phi_{\theta, w}(Y\mid X) - \Phi_{\theta, w}(Y\mid \tilde{X})\right| &\leqslant \sum_{\tilde{y}} \left| \mathbb{P}_{\theta}(A_{1:K}\in \mathcal{N}(\tilde{y}, r)\mid X) - \mathbb{P}_{\theta}(A_{1:K}\in \mathcal{N}(\tilde{y}, r)\mid \tilde{X}) \right| \cdot |V_{\tilde{y}}|.
    \end{align*}
    \end{minipage}
    }

    For the second interior term in Equation (\ref{eq:rnpc_tv}),
    {
    \scriptsize
    \begin{align*}
    \left|Z_{\theta}(X) - Z_{\theta}(\tilde{X})\right| \leqslant \sum_{\tilde{y}} \left| \mathbb{P}_\theta(A_{1:K}\in \mathcal{N}(\tilde{y}, r)\mid X) - \mathbb{P}_{\theta}(A_{1:K}\in \mathcal{N}(\tilde{y}, r)\mid \tilde{X}) \right| \cdot \left| V_{\tilde{y}} \right|.
    \end{align*}
    }
    

    By combining these two inequalities, Equation (\ref{eq:rnpc_tv}) is bounded by,
    {
    \scriptsize
    \begin{align*}
    \text{Equation} (\ref{eq:rnpc_tv}) &\leqslant \frac{1}{2}\frac{1}{Z_{\theta}(X)} \cdot 2\sum_{\tilde{y}} \left| \mathbb{P}_\theta(A_{1:K}\in \mathcal{N}(\tilde{y}, r)\mid X) - \mathbb{P}_{\theta}(A_{1:K}\in \mathcal{N}(\tilde{y}, r)\mid \tilde{X}) \right| \cdot \left| V_{\tilde{y}} \right| \\
    &= \frac{\sum_{\tilde{y}} \left| \mathbb{P}_\theta(A_{1:K}\in \mathcal{N}(\tilde{y}, r)\mid X) - \mathbb{P}_{\theta}(A_{1:K}\in \mathcal{N}(\tilde{y}, r)\mid \tilde{X}) \right| \cdot \left| V_{\tilde{y}} \right|}{\sum_{\tilde{y}} \mathbb{P}_\theta(A_{1:K}\in \mathcal{N}(\tilde{y}, r)\mid X) \cdot \left| V_{\tilde{y}} \right|} \\
    &\leqslant \max_{\tilde{y}} \frac{\left| \mathbb{P}_\theta(A_{1:K}\in \mathcal{N}(\tilde{y}, r)\mid X) - \mathbb{P}_{\theta}(A_{1:K}\in \mathcal{N}(\tilde{y}, r)\mid \tilde{X}) \right|}{\mathbb{P}_\theta(A_{1:K}\in \mathcal{N}(\tilde{y}, r)\mid X)} \\
    &= \max_{\tilde{y}} \left| 1-\frac{\mathbb{P}_{\theta}(A_{1:K}\in \mathcal{N}(\tilde{y}, r)\mid \tilde{X})}{\mathbb{P}_\theta(A_{1:K}\in \mathcal{N}(\tilde{y}, r)\mid X)} \right|. 
    \end{align*}
    }
    

    Consequently, $d_{\mathrm{TV}}\left(\hat{\Phi}_{\theta, w}(Y \mid X), \hat{\Phi}_{\theta, w}(Y \mid \tilde{X})\right) \leqslant \max_{\tilde{y}} \left| 1-\frac{\mathbb{P}_{\theta}(A_{1:K}\in \mathcal{N}(\tilde{y}, r)\mid \tilde{X})}{\mathbb{P}_\theta(A_{1:K}\in \mathcal{N}(\tilde{y}, r)\mid X)} \right|$.
    By successively applying the max and expectation operators to both sides, we complete the proof of Lemma \ref{thm:rnpc_perturbation}. 
\end{proof}

\subsection{Comparison in adversarial robustness} \label{app:comparison}
\begin{theorem}[Comparison in adversarial robustness \textbf{(Restatement of Theorem \ref{thm:comparison})}]
    Consider a $p$-norm-bounded adversarial attack with a budget of $\ell$.
    Assume the attribute recognition model $f_\theta$ is randomized and satisfies $\epsilon$-Differential Privacy (DP) with respect to the $p$-norm.
    Let the probability of an attribute taking a specific value correspond to the expected model output, \ie, $\mathbb{P}_{\theta_k}(A_k=a_k \mid X) = \mathbb{E}[f_{\theta_k}(X)_{a_k}]$, where the expectation is taken over the randomness within the model.
    Under Assumption \ref{assump:complete}, the following holds:
    $\Lambda_{\npcname} \leqslant \frac{|\mathcal{A}_1|\ldots|\mathcal{A}_K|}{2}\alpha_\epsilon$ and $\Lambda_{\rnpcname} \leqslant \alpha_\epsilon$, where $\alpha_\epsilon := \max\{1-e^{-K\epsilon}, e^{K\epsilon}-1\}$.
    Moreover, there exist instances where both inequalities hold as equalities.
\end{theorem}

\begin{proof}
    \textbf{Firstly}, we aim to prove that, under the given conditions, the following two statements hold for any $\tilde{X}\in\mathbb{B}_p(X, \ell)$, any $y\in\mathcal{Y}$, and any $a_{1:K}\in\mathcal{A}_1\times\ldots\times \mathcal{A}_K$:
    {
    \scriptsize
    \begin{align}
    \left| \mathbb{P}_{\theta}(A_{1:K}=a_{1:K}\mid \tilde{X})-\mathbb{P}_{\theta}(A_{1:K}=a_{1:K}\mid X) \right| &\leqslant \left| \frac{\mathbb{P}_{\theta}(A_{1:K}=a_{1:K}\mid \tilde{X})}{\mathbb{P}_{\theta}(A_{1:K}=a_{1:K}\mid X)}-1 \right| \leqslant \alpha_\epsilon, \label{eq:statement1}\\
    \left| \mathbb{P}_{\theta}(A_{1:K}\in \mathcal{N}(\tilde{y}, r)\mid \tilde{X})-\mathbb{P}_{\theta}(A_{1:K}\in \mathcal{N}(\tilde{y}, r)\mid X) \right| &\leqslant \left| \frac{\mathbb{P}_{\theta}(A_{1:K}\in \mathcal{N}(\tilde{y}, r)\mid \tilde{X})}{\mathbb{P}_{\theta}(A_{1:K}\in \mathcal{N}(\tilde{y}, r)\mid X)}-1 \right| \leqslant \alpha_\epsilon. \label{eq:statement2}
    \end{align}
    }


    Given that the attribute recognition model satisfies $\epsilon$-DP and using the expected output stability property of DP \cite{dp_robustness},
    {
    \scriptsize
    \begin{align*}
    &\mathbb{P}_{\theta_k}(A_k=a_k \mid X)\leqslant e^\epsilon \mathbb{P}_{\theta_k}(A_k=a_k \mid \tilde{X}).
    \end{align*}
    }


    Building on this, and under Assumption \ref{assump:complete}, the joint probability over all attributes $A_{1:K}$ is bounded by,
    {
    \scriptsize
    \begin{align*}
    &\mathbb{P}_{\theta}(A_{1:K}=a_{1:K} \mid X) = \prod_k \mathbb{P}_{\theta_k}(A_k=a_k \mid X) \leqslant (e^\epsilon)^K \prod_k \mathbb{P}_{\theta_k}(A_k=a_k \mid \tilde{X}) = e^{K\epsilon}\mathbb{P}_{\theta}(A_{1:K}=a_{1:K} \mid \tilde{X}).
    \end{align*}
    }


    Consequently, the probabilities within the neighborhood of any $\tilde{y}\in\mathcal{Y}$ are bounded by,

    \resizebox{\textwidth}{!}{%
    \begin{minipage}{\textwidth}
    \begin{align*}
    &\mathbb{P}_{\theta}(A_{1:K}\in\mathcal{N}(\tilde{y}, r) \mid X) = \sum_{a_{1:K}\in \mathcal{N} (\tilde{y}, r)} \mathbb{P}_{\theta}(A_{1:K}=a_{1:K} \mid X) \leqslant \sum_{a_{1:K}\in \mathcal{N} (\tilde{y}, r)} e^{K\epsilon} \mathbb{P}_{\theta}(A_{1:K}=a_{1:K} \mid \tilde{X}) =e^{K\epsilon} \mathbb{P}_{\theta}(A_{1:K}\in\mathcal{N}(\tilde{y}, r) \mid \tilde{X}).
    \end{align*}
    \end{minipage}
    }

    Therefore, $\frac{\mathbb{P}_{\theta}(A_{1:K}\in\mathcal{N}(\tilde{y}, r) \mid \tilde{X})}{\mathbb{P}_{\theta}(A_{1:K}\in\mathcal{N}(\tilde{y}, r) \mid X)} - 1 \geqslant e^{-K\epsilon}-1$, and similarly, $\frac{\mathbb{P}_{\theta}(A_{1:K}\in\mathcal{N}(\tilde{y}, r) \mid \tilde{X})}{\mathbb{P}_{\theta}(A_{1:K}\in\mathcal{N}(\tilde{y}, r) \mid X)} - 1 \leqslant e^{K\epsilon}-1$.
    
    By combining these two inequalities, we obtain,
    {
    \scriptsize
    \begin{align*}
    &\left| \frac{\mathbb{P}_{\theta}(A_{1:K}\in\mathcal{N}(\tilde{y}, r) \mid \tilde{X})}{\mathbb{P}_{\theta}(A_{1:K}\in\mathcal{N}(\tilde{y}, r) \mid X)} - 1 \right| \leqslant \alpha_\epsilon,
    \end{align*}
    }
    where $\alpha_\epsilon := \max\{1-e^{-K\epsilon}, e^{K\epsilon}-1\}$. 
    
    On the other hand, the following holds,

    \resizebox{\textwidth}{!}{%
    \begin{minipage}{\textwidth}
    \begin{align*}
    &\left| \frac{\mathbb{P}_{\theta}(A_{1:K}\in\mathcal{N}(\tilde{y}, r) \mid \tilde{X})}{\mathbb{P}_{\theta}(A_{1:K}\in\mathcal{N}(\tilde{y}, r) \mid X)} - 1 \right| = \frac{\left| \mathbb{P}_{\theta}(A_{1:K}\in\mathcal{N}(\tilde{y}, r) \mid \tilde{X})-\mathbb{P}_{\theta}(A_{1:K}\in\mathcal{N}(\tilde{y}, r) \mid X) \right|}{\mathbb{P}_{\theta}(A_{1:K}\in\mathcal{N}(\tilde{y}, r) \mid X)} \geqslant \left| \mathbb{P}_{\theta}(A_{1:K}\in\mathcal{N}(\tilde{y}, r) \mid \tilde{X})-\mathbb{P}_{\theta}(A_{1:K}\in\mathcal{N}(\tilde{y}, r) \mid X) \right|.
    \end{align*}
    \end{minipage}
    }

    Hence, Equation (\ref{eq:statement2}) is satisfied. 
    Following a similar approach, Equation (\ref{eq:statement1}) can also be proven to hold.    

    \textbf{Secondly}, we aim to prove that, when Equation (\ref{eq:statement1}) and Equation (\ref{eq:statement2}) are satisfied, $\Lambda_{\npcname} \leqslant \frac{|\mathcal{A}_1|\ldots|\mathcal{A}_K|}{2}\alpha_\epsilon$ and $\Lambda_{\rnpcname} \leqslant \alpha_\epsilon$. In particular, the bound for $\Lambda_{\rnpcname}$ is apparent based on Equation (\ref{eq:statement2}).
    Besides, for $\Lambda_{\npcname}$, the following holds,
    {
    \scriptsize
    \begin{align*}
    \Lambda_{\npcname} &= \mathbb{E}_X\left[\max_{\tilde{X}\in \mathbb{B}_p(X, \ell)}~\left\{\frac{1}{2}\sum_{a_{1:K}} \left| \mathbb{P}_{\theta}(A_{1:K}=a_{1:K}\mid X) - \mathbb{P}_{\theta}(A_{1:K}=a_{1:K}\mid \tilde{X})\right| \right\}\right] \leqslant \frac{|\mathcal{A}_1|\ldots|\mathcal{A}_K|}{2}\alpha_\epsilon.
    \end{align*}
    }


    Therefore, we have proven that $\Lambda_{\npcname} \leqslant \frac{|\mathcal{A}_1|\ldots|\mathcal{A}_K|}{2}\alpha_\epsilon$ and $\Lambda_{\rnpcname} \leqslant \alpha_\epsilon$ hold under the given conditions.

    \textbf{Finally}, we aim to provide an instance showing that the bounds for $\Lambda_{\npcname}$ and $\Lambda_{\rnpcname}$ can be simultaneously achieved.
    
    Suppose a case where $\forall X\in\mathcal{X}$, there exists $\tilde{y}\in\mathcal{Y}$ such that $\mathbb{P}_\theta(A_{1:K}\in\mathcal{N}(\tilde{y},r)\mid X) = 1$. There are $2n+1$ nodes in $\mathcal{N}(\tilde{y},r)$ and the probabilities of $n$ of them are increased after attack, in particular, $\mathbb{P}_\theta(A_{1:K}=a_{1:K}\mid \tilde{X}) = \mathbb{P}_\theta(A_{1:K}=a_{1:K}\mid X) + \alpha_\epsilon$.
    In contrast, the remaining $n+1$ of them are decreased, in particular, $\mathbb{P}_\theta(A_{1:K}=a_{1:K}\mid \tilde{X}) = \mathbb{P}_\theta(A_{1:K}=a_{1:K}\mid X) - \alpha_\epsilon$. Overall, $\mathbb{P}_\theta(A_{1:K}\in\mathcal{N}(\tilde{y},r)\mid \tilde{X}) = \mathbb{P}_\theta(A_{1:K}\in\mathcal{N}(\tilde{y},r)\mid X) - \alpha_\epsilon = 1-\alpha_\epsilon$.
    
    On the other hand, suppose there are $2m+1$ nodes in the complement set $\Omega\backslash\mathcal{N}(\tilde{y},r)$. The probabilities of $m$ of them are decreased after attack, in particular, $\mathbb{P}_\theta(A_{1:K}=a_{1:K}\mid \tilde{X}) = \mathbb{P}_\theta(A_{1:K}=a_{1:K}\mid X) - \alpha_\epsilon$.
    In contrast, the remaining $m+1$ of them are increased, in particular, $\mathbb{P}_\theta(A_{1:K}=a_{1:K}\mid \tilde{X}) = \mathbb{P}_\theta(A_{1:K}=a_{1:K}\mid X) + \alpha_\epsilon$. Overall, $\mathbb{P}_\theta(A_{1:K}\in\Omega\backslash\mathcal{N}(\tilde{y},r)\mid \tilde{X}) = \mathbb{P}_\theta(A_{1:K}\in\Omega\backslash\mathcal{N}(\tilde{y},r)\mid X) + \alpha_\epsilon = \alpha_\epsilon$.
    
    In the above case, it is easy to show that $\Lambda_{\rnpcname}=\alpha_\epsilon$. In addition, we notice that $\forall a_{1:K},~|\mathbb{P}_\theta(A_{1:K}=a_{1:K}\mid X) - \mathbb{P}_\theta(A_{1:K}=a_{1:K}\mid \tilde{X})| = \alpha_\epsilon$. Thus, $\Lambda_{\npcname} = \frac{|\mathcal{A}_1|\ldots|\mathcal{A}_K|}{2}\alpha_\epsilon$.
    Therefore, in the case constructed above, the bounds for $\Lambda_{\npcname}$ and $\Lambda_{\rnpcname}$ are simultaneously achieved.
\end{proof}

\begin{theorem} [Direct comparison in adversarial robustness]
    Assume that there exists \(c \in (0,1)\) such that for all \(X \in \mathcal{X}\) and \(\tilde{y} \in \mathcal{Y}\), \(\mathbb{P}_\theta(A_{1:K} \in \mathcal{N}(\tilde{y}, r) \mid X) \geqslant c\). Then, the following inequality holds: $\Lambda_{\rnpcname} \leqslant \frac{1}{c}\Lambda_{\npcname}$.

    \label{thm:direct_comparison}
\end{theorem}

\begin{proof}
    By the definition of $\Lambda_{\rnpcname}$ and the given conditions, the following holds,
    {
    \scriptsize
    \begin{align*}
    \Lambda_{\rnpcname} &= \mathbb{E}_X\left[ \max_{\tilde{y}} \frac{1}{\mathbb{P}_{\theta}(A_{1:K}\in \mathcal{N}(\tilde{y}, r)\mid X)}\left| \mathbb{P}_{\theta}(A_{1:K}\in \mathcal{N}(\tilde{y}, r)\mid X)-\mathbb{P}_{\theta}(A_{1:K}\in \mathcal{N}(\tilde{y}, r)\mid \tilde{X}) \right| \right] \\
    &\leqslant \frac{1}{c}\mathbb{E}_X\left[ \max_{\tilde{y}} \left| \mathbb{P}_{\theta}(A_{1:K}\in \mathcal{N}(\tilde{y}, r)\mid X)-\mathbb{P}_{\theta}(A_{1:K}\in \mathcal{N}(\tilde{y}, r)\mid \tilde{X}) \right| \right] \\
    &\leqslant \frac{1}{c}\mathbb{E}_X\left[ d_{\mathrm{TV}}\left( \mathbb{P}_{\theta}(A_{1:K}\mid X),\mathbb{P}_{\theta}(A_{1:K}\mid \tilde{X}) \right) \right] = \frac{1}{c}\Lambda_{\npcname}.\qedhere
    \end{align*}
    }
    
\end{proof}

Compared to Theorem \ref{thm:comparison}, which compares the upper bounds for \(\Lambda_{\npcname}\) and \(\Lambda_{\rnpcname}\), Theorem \ref{thm:direct_comparison} provides a more direct relationship between them. Specifically, Theorem \ref{thm:direct_comparison} demonstrates that \(\Lambda_{\rnpcname}\) cannot exceed a fixed multiple of \(\Lambda_{\npcname}\), with the multiplier inversely proportional to the lower bound \(c\) of the neighborhood probabilities.

\subsection{Benign task performance of \rnpcname s} \label{app:benign}
\begin{definition}
    The \textit{prediction error} of \rnpcname~is defined as the expected TV distance between the predicted distribution and the ground-truth distribution, \ie, $\varepsilon_{\theta, w}^{\rnpcname} := \mathbb{E}_X\left[ d_{\mathrm{TV}}\left( \hat{\Phi}_{\theta,w}(Y\mid X), \mathbb{P}^*(Y\mid X) \right) \right].$
\end{definition}
Note that the definition of \textit{prediction error} is different from that of \textit{estimation error}. The latter is defined as the expected TV distance between the predicted distribution and the \textbf{optimal distribution}.

\begin{theorem}[Prediction error of \rnpcname]
    The prediction error of \rnpcname~is bounded as follows,
    {
    \small
    \begin{align*}
        \varepsilon_{\theta, w}^{\rnpcname} \leqslant 
        &\mathbb{E}_X\left[ \min \left\{ \max_{\tilde{y}} \left| \frac{\mathbb{P}_\theta(A_{1:K}\in \mathcal{N}(\tilde{y}, r) \mid X)}{\mathbb{P}^*(A_{1:K}\in \mathcal{N}(\tilde{y}, r) \mid X)}-1 \right|, 
        \max_{\tilde{y}} \left| \frac{\mathbb{P}^*(A_{1:K}\in \mathcal{N}(\tilde{y}, r) \mid X)}{\mathbb{P}_\theta(A_{1:K}\in \mathcal{N}(\tilde{y}, r) \mid X)}-1 \right| \right\} \right] \\
        &+ \frac{2}{\gamma} d_{\mathrm{TV}}\left(\mathbb{P}_w(Y, A_{1:K}), \mathbb{P}^*(Y, A_{1:K})\right) 
        + \mathbb{E}_X\left[ \max_{\tilde{y}} d_{\mathrm{TV}}\left( \bar{\mathbb{P}}^*(Y\mid A_{1:K}\in V_{\tilde{y}}), \mathbb{P}^*(Y\mid X) \right) \right],
    \end{align*}
    }
    where $\bar{\mathbb{P}}^*(Y\mid A_{1:K}\in V_{\tilde{y}}) := \frac{1}{|V_{\tilde{y}}|}\sum_{a_{1:K}\in V_{\tilde{y}}}\mathbb{P}^*(Y\mid A_{1:K}=a_{1:K})$ represents the average ground-truth conditional distribution of $Y$ given $A_{1:K}\in V_{\tilde{y}}$.
    \label{thm:general_error}
\end{theorem}

\begin{proof}
    Define $\Phi^*(Y\mid X) := \sum_{\tilde{y}}\mathbb{P}^*(A_{1:K}\in \mathcal{N}(\tilde{y},r)\mid X)\cdot \sum_{a_{1:K}\in V_{\tilde{y}}}\mathbb{P}^*(Y\mid A_{1:K}=a_{1:K})$ and $Z^*(X) := \sum_{\tilde{y}} \mathbb{P}^*(A_{1:K}\in\mathcal{N}(\tilde{y},r)\mid X)\cdot |V_{\tilde{y}}|$.

    By applying the triangle inequality, the following holds,
    {
    \scriptsize
    \begin{equation}
    \varepsilon_{\theta, w}^{\rnpcname} \leqslant 
    \mathbb{E}_X\left[ d_{\mathrm{TV}}\left( \frac{\Phi_{\theta,w}(Y\mid X)}{Z_\theta(X)}, \frac{\Phi^*(Y\mid X)}{Z^*(X)} \right) \right] + \mathbb{E}_X\left[ d_{\mathrm{TV}}\left( \frac{\Phi^*(Y\mid X)}{Z^*(X)}, \mathbb{P}^*(Y\mid X) \right) \right]. \label{general_error}
    \end{equation}
    }
    

    \textbf{For the first term in Equation (\ref{general_error}):}
    {
    \scriptsize
    \begin{align*}
    &d_{\mathrm{TV}}\left( \frac{\Phi_{\theta,w}(Y\mid X)}{Z_\theta(X)}, \frac{\Phi^*(Y\mid X)}{Z^*(X)} \right) = \mathbb{E}_X\left[\frac{1}{2} \sum_y\left|\frac{\Phi_{\theta, w}(Y=y \mid X)}{Z_\theta(X)}-\frac{\Phi^*(Y=y \mid X)}{Z^*(X)}\right|\right] \\
    &= \mathbb{E}_X\left[\frac{1}{2} \cdot \frac{1}{Z_\theta(X) \cdot Z^*(X)} \sum_y \left| Z^*(X)\cdot \Phi_{\theta, w}(Y=y \mid X) - Z_\theta(X)\cdot \Phi^*(Y=y \mid X) \right|\right] .
    \end{align*}
    }
    

    In particular, for the term $\left| Z^*(X)\cdot \Phi_{\theta, w}(Y=y \mid X) - Z_\theta(X)\cdot \Phi^*(Y=y \mid X) \right|$, we have,
    
    \resizebox{\textwidth}{!}{%
    \begin{minipage}{\textwidth}
    \begin{align*}
    &\left| Z^*(X)\cdot \Phi_{\theta, w}(Y=y \mid X) - Z_\theta(X)\cdot \Phi^*(Y=y \mid X) \right| \\
    &= \left| Z^*(X)\cdot \Phi_{\theta, w}(Y=y \mid X) - Z_\theta(X)\cdot \Phi_{\theta, w}(Y=y \mid X) + Z_\theta(X)\cdot \Phi_{\theta, w}(Y=y \mid X) - Z_\theta(X)\cdot \Phi^*(Y=y \mid X) \right| \\
    &\leqslant \Phi_{\theta, w}(Y=y \mid X)\cdot\left| Z_\theta(X)-Z^*(X) \right| + Z_\theta(X)\cdot\left| \Phi_{\theta, w}(Y=y \mid X)-\Phi^*(Y=y \mid X) \right|.
    \end{align*}
    \end{minipage}
    }

    Consequently, the following holds,
    {
    \scriptsize
    \begin{align*}
    &d_{\mathrm{TV}}\left( \frac{\Phi_{\theta,w}(Y\mid X)}{Z_\theta(X)}, \frac{\Phi^*(Y\mid X)}{Z^*(X)} \right) \\
    &\leqslant \mathbb{E}_X\left[\frac{1}{2} \cdot \frac{1}{Z_\theta(X) \cdot Z^*(X)} \sum_y \Phi_{\theta, w}(Y=y \mid X)\cdot\left| Z_\theta(X)-Z^*(X) \right| + Z_\theta(X)\cdot\left| \Phi_{\theta, w}(Y=y \mid X)-\Phi^*(Y=y \mid X) \right| \right] \\
    &= \mathbb{E}_X\left[ \frac{1}{2}\cdot\frac{1}{Z^*(X)} |Z_\theta(X)-Z^*(X)| \right] + \mathbb{E}_X\left[ \frac{1}{2}\cdot\frac{1}{Z^*(X)}\sum_y|\Phi_{\theta, w}(Y=y \mid X)-\Phi^*(Y=y \mid X)| \right].
    \end{align*}
    }

    
    Moreover, we can bound the first term as follows,
    
    \resizebox{\textwidth}{!}{%
    \begin{minipage}{\textwidth}
    \begin{align*}
    \mathbb{E}_X\left[ \frac{1}{2}\cdot\frac{1}{Z^*(X)} |Z_\theta(X)-Z^*(X)| \right] &\leqslant \mathbb{E}_X\left[ \frac{1}{2}\cdot\frac{1}{Z^*(X)} \sum_{\tilde{y}}|\mathbb{P}_\theta(A_{1:K}\in \mathcal{N}(\tilde{y}, r)\mid X)-\mathbb{P}^*(A_{1:K}\in \mathcal{N}(\tilde{y}, r) \mid X)|\cdot|V_{\tilde{y}}| \right]\\
    &\leqslant \mathbb{E}_X\left[ \frac{1}{2}\max_{\tilde{y}} \left| \frac{\mathbb{P}_\theta(A_{1:K}\in \mathcal{N}(\tilde{y}, r) \mid X)}{\mathbb{P}^*(A_{1:K}\in \mathcal{N}(\tilde{y}, r) \mid X)}-1 \right| \right],
    \end{align*}
    \end{minipage}
    }

    and bound the second term as follows,

    \resizebox{\textwidth}{!}{%
    \begin{minipage}{\textwidth}
    \begin{align}
    &|\Phi_{\theta, w}(Y \mid X)-\Phi^*(Y \mid X)| \notag \\
    &= |\sum_{\tilde{y}} [\mathbb{P}_\theta(A_{1:K}\in \mathcal{N}(\tilde{y}, r)\mid X)\cdot\sum_{\tilde{a}_{1:K}\in V_{\tilde{y}}}\mathbb{P}_w(Y\mid A_{1:K}=\tilde{a}_{1:K}) - \mathbb{P}^*(A_{1:K}\in \mathcal{N}(\tilde{y}, r)\mid X)\cdot\sum_{\tilde{a}_{1:K}\in V_{\tilde{y}}}\mathbb{P}_w(Y\mid A_{1:K}=\tilde{a}_{1:K}) \notag \\
    &\quad + \mathbb{P}^*(A_{1:K}\in \mathcal{N}(\tilde{y}, r)\mid X)\cdot\sum_{\tilde{a}_{1:K}\in V_{\tilde{y}}}\mathbb{P}_w(Y\mid A_{1:K}=\tilde{a}_{1:K}) - \mathbb{P}^*(A_{1:K}\in \mathcal{N}(\tilde{y}, r)\mid X)\cdot\sum_{\tilde{a}_{1:K}\in V_{\tilde{y}}}\mathbb{P}^*(Y\mid A_{1:K}=\tilde{a}_{1:K})
    ]| \notag \\
    &\leqslant  \sum_{\tilde{y}} |\mathbb{P}_\theta(A_{1:K}\in \mathcal{N}(\tilde{y}, r)\mid X)-\mathbb{P}^*(A_{1:K}\in \mathcal{N}(\tilde{y}, r)\mid X)|\cdot\sum_{\tilde{a}_{1:K}\in V_{\tilde{y}}}\mathbb{P}_w(Y\mid A_{1:K}=\tilde{a}_{1:K}) \label{eq1_1}\\
    &\quad + \sum_{\tilde{y}} \mathbb{P}^*(A_{1:K}\in \mathcal{N}(\tilde{y}, r)\mid X)\cdot|\sum_{\tilde{a}_{1:K}\in V_{\tilde{y}}}\mathbb{P}_w(Y\mid A_{1:K}=\tilde{a}_{1:K})-\sum_{\tilde{a}_{1:K}\in V_{\tilde{y}}}\mathbb{P}^*(Y\mid A_{1:K}=\tilde{a}_{1:K})|. \label{eq1_2}
    \end{align}
    \end{minipage}
    }

    In particular, the following holds:
    {
    \scriptsize
    \begin{align*}
    \mathbb{E}_X\left[ \frac{1}{2}\cdot\frac{1}{Z^*(X)}\sum_y (Eq. (\ref{eq1_1})) \right] &= \mathbb{E}_X\left[ \frac{1}{2}\cdot\frac{1}{Z^*(X)}\sum_{\tilde{y}} |\mathbb{P}_\theta(A_{1:K}\in \mathcal{N}(\tilde{y}, r)\mid X)-\mathbb{P}^*(A_{1:K}\in \mathcal{N}(\tilde{y}, r)\mid X)|\cdot|V_{\tilde{y}}| \right] \\
    &\leqslant \mathbb{E}_X\left[ \frac{1}{2}\max_{\tilde{y}} \left| \frac{\mathbb{P}_\theta(A_{1:K}\in \mathcal{N}(\tilde{y}, r) \mid X)}{\mathbb{P}^*(A_{1:K}\in \mathcal{N}(\tilde{y}, r) \mid X)}-1 \right| \right],
    \end{align*}
    }
    

    and

    \resizebox{\textwidth}{!}{%
    \begin{minipage}{\textwidth}
    \begin{align*}
    &\mathbb{E}_X\left[ \frac{1}{2}\cdot\frac{1}{Z^*(X)}\sum_y (Eq. (\ref{eq1_2})) \right] \\
    &\leqslant \mathbb{E}_X\left[ \frac{1}{Z^*(X)}\sum_{\tilde{y}} \mathbb{P}^*(A_{1:K}\in \mathcal{N}(\tilde{y}, r) \mid X)\cdot \frac{1}{2}\sum_y\sum_{\tilde{a}_{1:K}\in V_{\tilde{y}}} |\mathbb{P}_w(Y=y\mid A_{1:K}=\tilde{a}_{1:K})-\mathbb{P}^*(Y=y\mid A_{1:K}=\tilde{a}_{1:K})| \right] \\
    &= \mathbb{E}_X\left[ \frac{1}{Z^*(X)}\sum_{\tilde{y}} \mathbb{P}^*(A_{1:K}\in \mathcal{N}(\tilde{y}, r) \mid X)\cdot \sum_{\tilde{a}_{1:K}\in V_{\tilde{y}}} d_{\mathrm{TV}}\left(\mathbb{P}_w(Y\mid A_{1:K}=\tilde{a}_{1:K}),\mathbb{P}^*(Y\mid A_{1:K}=\tilde{a}_{1:K})  \right)\right] \\
    &\leqslant \mathbb{E}_X\left[ \frac{1}{Z^*(X)}\sum_{\tilde{y}} \mathbb{P}^*(A_{1:K}\in \mathcal{N}(\tilde{y}, r) \mid X)\cdot \sum_{\tilde{a}_{1:K}\in V_{\tilde{y}}} \frac{1}{\mathbb{P}^*(A_{1:K}=\tilde{a}_{1:K})}\cdot\sum_y \left|\mathbb{P}_w(Y=y, A_{1:K}=\tilde{a}_{1:K})-\mathbb{P}^*(Y=y, A_{1:K}=\tilde{a}_{1:K})\right|\right] \\
    &\leqslant \frac{1}{\gamma}\max_{\tilde{y}} \frac{1}{|V_{\tilde{y}}|} \cdot \sum_{\tilde{a}_{1:K}\in V_{\tilde{y}}}\sum_y \left|\mathbb{P}_w(Y=y, A_{1:K}=\tilde{a}_{1:K})-\mathbb{P}^*(Y=y, A_{1:K}=\tilde{a}_{1:K})\right| \\
    &\leqslant \frac{2}{\gamma} d_{\mathrm{TV}}\left(\mathbb{P}_w(Y, A_{1:K}), \mathbb{P}^*(Y, A_{1:K})\right).
    \end{align*}
    \end{minipage}
    }

    Within the above derivation, we utilize two facts. The first one is,
    
    \resizebox{\textwidth}{!}{%
    \begin{minipage}{\textwidth}
    \begin{align*}
    &d_{\mathrm{TV}}\left(\mathbb{P}_w(Y\mid A_{1:K}=a_{1:K}),\mathbb{P}^*(Y\mid A_{1:K}=a_{1:K}) \right) \\
    &= \frac{1}{2}\sum_y\left| \mathbb{P}_w(Y=y\mid A_{1:K}=a_{1:K})-\mathbb{P}^*(Y=y\mid A_{1:K}=a_{1:K}) \right| \\
    &= \frac{1}{2}\cdot\frac{1}{\mathbb{P}^*(A_{1:K}=a_{1:K})}\sum_y |\mathbb{P}_w(Y=y\mid A_{1:K}=a_{1:K})\cdot\mathbb{P}^*(A_{1:K}=a_{1:K}) - \mathbb{P}_w(Y=y\mid A_{1:K}=a_{1:K})\cdot\mathbb{P}_w(A_{1:K}=a_{1:K}) \\
    &\quad + \mathbb{P}_w(Y=y\mid A_{1:K}=a_{1:K})\cdot\mathbb{P}_w(A_{1:K}=a_{1:K}) - \mathbb{P}^*(Y=y\mid A_{1:K}=a_{1:K})\cdot\mathbb{P}^*(A_{1:K}=a_{1:K}) | \\
    &\leqslant \frac{1}{2}\cdot\frac{1}{\mathbb{P}^*(A_{1:K}=a_{1:K})}\sum_y \mathbb{P}_w(Y=y\mid A_{1:K}=a_{1:K})\cdot|\mathbb{P}^*(A_{1:K}=a_{1:K}) - \mathbb{P}_w(A_{1:K}=a_{1:K})| \\
    &\quad + \frac{1}{2}\cdot\frac{1}{\mathbb{P}^*(A_{1:K}=a_{1:K})}\sum_y |\mathbb{P}_w(Y=y, A_{1:K}=a_{1:K}) - \mathbb{P}^*(Y=y, A_{1:K}=a_{1:K}) |\\
    &\leqslant \frac{1}{\mathbb{P}^*(A_{1:K}=a_{1:K})}\sum_y |\mathbb{P}_w(Y=y, A_{1:K}=a_{1:K}) - \mathbb{P}^*(Y=y, A_{1:K}=a_{1:K}) |.
    \end{align*}
    \end{minipage}
    }

    The second one is,
    {
    \scriptsize
    \begin{align*}
    \forall \tilde{a}_{1:K}\in V_{\tilde{y}}, \mathbb{P}^*(A_{1:K}=\tilde{a}_{1:K}) \geqslant \gamma.
    \end{align*}
    }
    
    
    Combining the above, we have,
    {
    \scriptsize
    \begin{align*}
    d_{\mathrm{TV}}\left( \frac{\Phi_{\theta,w}(Y\mid X)}{Z_\theta(X)}, \frac{\Phi^*(Y\mid X)}{Z^*(X)} \right)
    \leqslant &\mathbb{E}_X\left[ \max_{\tilde{y}} \left| \frac{\mathbb{P}_\theta(A_{1:K}\in \mathcal{N}(\tilde{y}, r) \mid X)}{\mathbb{P}^*(A_{1:K}\in \mathcal{N}(\tilde{y}, r) \mid X)}-1 \right| \right] + \frac{2}{\gamma} d_{\mathrm{TV}}\left(\mathbb{P}_w(Y, A_{1:K}), \mathbb{P}^*(Y, A_{1:K})\right).
    \end{align*}
    }


    Similarly, we can derive,
    {
    \scriptsize
    \begin{align*}
    d_{\mathrm{TV}}\left( \frac{\Phi_{\theta,w}(Y\mid X)}{Z_\theta(X)}, \frac{\Phi^*(Y\mid X)}{Z^*(X)} \right) \leqslant \mathbb{E}_X\left[ \max_{\tilde{y}} \left| \frac{\mathbb{P}^*(A_{1:K}\in \mathcal{N}(\tilde{y}, r) \mid X)}{\mathbb{P}_\theta(A_{1:K}\in \mathcal{N}(\tilde{y}, r) \mid X)}-1 \right| \right] + \frac{2}{\gamma} d_{\mathrm{TV}}\left(\mathbb{P}_w(Y, A_{1:K}), \mathbb{P}^*(Y, A_{1:K})\right).
    \end{align*}
    }
    
    
    Therefore, the following holds,

    \resizebox{\textwidth}{!}{%
    \begin{minipage}{\textwidth}
    \begin{align*}
    d_{\mathrm{TV}}\left( \frac{\Phi_{\theta,w}(Y\mid X)}{Z_\theta(X)}, \frac{\Phi^*(Y\mid X)}{Z^*(X)} \right) &\leqslant \mathbb{E}_X\left[ \min\left\{\max_{\tilde{y}} \left| \frac{\mathbb{P}^*(A_{1:K}\in \mathcal{N}(\tilde{y}, r) \mid X)}{\mathbb{P}_\theta(A_{1:K}\in \mathcal{N}(\tilde{y}, r) \mid X)}-1 \right|, \max_{\tilde{y}} \left| \frac{\mathbb{P}_\theta(A_{1:K}\in \mathcal{N}(\tilde{y}, r) \mid X)}{\mathbb{P}^*(A_{1:K}\in \mathcal{N}(\tilde{y}, r) \mid X)}-1 \right| \right\}\right] \\
    &\quad + \frac{2}{\gamma} d_{\mathrm{TV}}\left(\mathbb{P}_w(Y, A_{1:K}), \mathbb{P}^*(Y, A_{1:K})\right).
    \end{align*}
    \end{minipage}
    }

    \textbf{For the second term in Equation (\ref{general_error}):} The following holds,
    {
    \scriptsize
    \begin{align*}
    \mathbb{E}_X\left[ d_{\mathrm{TV}}\left( \frac{\Phi^*(Y\mid X)}{Z^*(X)}, \mathbb{P}^*(Y\mid X) \right) \right] &\leqslant \mathbb{E}_X\left[ \max_{\tilde{y}} d_{\mathrm{TV}}\left( \bar{\mathbb{P}}^*(Y\mid A_{1:K}\in V_{\tilde{y}}), \mathbb{P}^*(Y\mid X) \right) \right],
    \end{align*}
    }


    because:
    {
    \scriptsize
    \begin{align}
    d_{\mathrm{TV}}\left( \frac{\Phi^*(Y\mid X)}{Z^*(X)}, \mathbb{P}^*(Y\mid X) \right) &= \frac{1}{2}\sum_Y\left| \frac{\sum_{\tilde{y}}\mathbb{P}^*(A_{1:K}\in \mathcal{N}(\tilde{y},r)\mid X)\cdot \sum_{a_{1:K}\in V_{\tilde{y}}}\mathbb{P}^*(Y\mid A_{1:K}=a_{1:K})}{\sum_{\tilde{y}} \mathbb{P}^*(A_{1:K}\in\mathcal{N}(\tilde{y},r)\mid X)\cdot |V_{\tilde{y}}|} - \mathbb{P}^*(Y\mid X)  \right| \nonumber\\
    &= \frac{1}{2}\sum_Y\left| \frac{\sum_{\tilde{y}}\mathbb{P}^*(A_{1:K}\in \mathcal{N}(\tilde{y},r)\mid X)\cdot |V_{\tilde{y}}| \cdot \bar{\mathbb{P}}^*(Y\mid A_{1:K}\in V_{\tilde{y}})}{\sum_{\tilde{y}} \mathbb{P}^*(A_{1:K}\in\mathcal{N}(\tilde{y},r)\mid X)\cdot |V_{\tilde{y}}|} - \mathbb{P}^*(Y\mid X)  \right|. \label{han}
    \end{align}
    }
    

    Define $W_{\tilde{y}} := \mathbb{P}^*(A_{1:K}\in \mathcal{N}(\tilde{y},r)\mid X)\cdot |V_{\tilde{y}}|$ and $\alpha_{\tilde{y}} := \frac{W_{\tilde{y}}}{\sum_{y^{\prime}}W_{y^{\prime}}}$. We have $\sum_{\tilde{y}}\alpha_{\tilde{y}}=1$.

    Then, for Equation (\ref{han}),
    {
    \scriptsize
    \begin{align*}
    \text{Equation} (\ref{han}) = \frac{1}{2}\sum_Y \left| \sum_{\tilde{y}}\alpha_{\tilde{y}}\bar{\mathbb{P}}^*(Y\mid A_{1:K}\in V_{\tilde{y}}) - \mathbb{P}^*(Y\mid X)  \right| &\leqslant \frac{1}{2}\sum_Y \sum_{\tilde{y}}\alpha_{\tilde{y}}\left| \bar{\mathbb{P}}^*(Y\mid A_{1:K}\in V_{\tilde{y}}) - \mathbb{P}^*(Y\mid X)  \right| \\
    &= \sum_{\tilde{y}} \alpha_{\tilde{y}}~d_{TV}\left( \bar{\mathbb{P}}^*(Y\mid A_{1:K}\in V_{\tilde{y}}), \mathbb{P}^*(Y\mid X) \right) \\
    &\leqslant \max_{\tilde{y}} d_{TV}\left( \bar{\mathbb{P}}^*(Y\mid A_{1:K}\in V_{\tilde{y}}), \mathbb{P}^*(Y\mid X) \right).
    \end{align*}
    }


    By combining the bounds for the first and second terms in Equation (\ref{general_error}), we establish the bound for $\varepsilon_{\theta,w}^{\rnpcname}$.
\end{proof}

\begin{proposition}[Optimal \rnpcname s \textbf{(Restatement of Proposition \ref{prop:optimal})}]
    The optimal \rnpcname~\wrt the prediction error $\varepsilon_{\theta,w}^{\rnpcname}$ is $\hat{\Phi}^*(Y\mid X) := \frac{\Phi^*(Y\mid X)}{Z^*(X)}$, where
    {
    \small
    \begin{align*}
    \Phi^*(Y\mid X) &:= \sum_{\tilde{y}}\mathbb{P}^*(A_{1:K}\in \mathcal{N}(\tilde{y},r)\mid X)\cdot \sum_{a_{1:K}\in V_{\tilde{y}}}\mathbb{P}^*(Y\mid A_{1:K}=a_{1:K}), \\
    Z^*(X) &:= \sum_{\tilde{y}} \mathbb{P}^*(A_{1:K}\in\mathcal{N}(\tilde{y},r)\mid X)\cdot |V_{\tilde{y}}|.
    \end{align*}
    }


\end{proposition}

\begin{proof}
    According to Theoremn \ref{thm:general_error}, the minimum of the bound for $\varepsilon_{\theta,w}^{\rnpcname}$ is achieved when $\mathbb{P}_\theta(A_{1:K}\mid X)=\mathbb{P}^*(A_{1:K}\mid X)$ and $\mathbb{P}_w(Y, A_{1:K})=\mathbb{P}^*(Y, A_{1:K})$.
    In this case, $\Phi_{\theta,w}(Y\mid X)=\Phi^*(Y\mid X)$ and $Z_\theta(X)=Z^*(X)$.

\end{proof}

\begin{theorem}[Trade-off of \rnpcname s \textbf{(Restatement of Theoremn \ref{thm:tradeoff})}]
    The trade-off of \rnpcname s in benign performance, defined as the expected TV distance between the optimal \rnpcname~\(\hat{\Phi}^*(Y \mid X)\) and the ground-truth distribution \(\mathbb{P}^*(Y \mid X)\), is bounded as follows,    
    {
    \small
    \begin{align*}
        &\mathbb{E}_X \left[ d_{\mathrm{TV}}\left( \hat{\Phi}^*(Y\mid X), \mathbb{P}^*(Y\mid X) \right) \right] 
        \leqslant \mathbb{E}_X\left[ \max_{\tilde{y}}~d_{\mathrm{TV}}\left( \bar{\mathbb{P}}^*(Y\mid A_{1:K}\in V_{\tilde{y}}), \mathbb{P}^*(Y\mid X) \right) \right],
    \end{align*}
    }
    where $\bar{\mathbb{P}}^*(Y\mid A_{1:K}\in V_{\tilde{y}}) := \frac{1}{|V_{\tilde{y}}|}\sum_{a_{1:K}\in V_{\tilde{y}}}\mathbb{P}^*(Y\mid A_{1:K}=a_{1:K})$ represents the average ground-truth conditional distribution of $Y$ given $A_{1:K}\in V_{\tilde{y}}$.
\end{theorem}

\begin{proof}
    See proof for Theorem \ref{general_error}.
\end{proof}

\begin{theorem}[Compositional estimation error \textbf{(Extension of Theorem \ref{thm:compositional_error})}]
    The estimation error of \rnpcname~is bounded by a linear combination of the errors from the attribute recognition model and the probabilistic circuit, \ie,
    {
    \small
    \begin{align*}
        \hat{\varepsilon}_{\theta, w}^{\rnpcname} 
        &\leqslant 
        \mathbb{E}_X\left[ \min \left\{ \max_{\tilde{y}} \left| \frac{\mathbb{P}_\theta(A_{1:K}\in \mathcal{N}(\tilde{y}, r) \mid X)}{\mathbb{P}^*(A_{1:K}\in \mathcal{N}(\tilde{y}, r) \mid X)}-1 \right|, 
        \max_{\tilde{y}} \left| \frac{\mathbb{P}^*(A_{1:K}\in \mathcal{N}(\tilde{y}, r) \mid X)}{\mathbb{P}_\theta(A_{1:K}\in \mathcal{N}(\tilde{y}, r) \mid X)}-1 \right| \right\} \right] \\
        &\quad + \frac{2}{\gamma} d_{\mathrm{TV}}\left(\mathbb{P}_w(Y, A_{1:K}), \mathbb{P}^*(Y, A_{1:K})\right), \\
    \end{align*}
    }
    where $P^*$ represents the ground-truth distribution.
\end{theorem}

\begin{proof}
    See proof for Theorem \ref{general_error}.
\end{proof}

\end{document}